\documentclass[a4paper,usenames,dvipsnames,twoside,11pt]{article}
\usepackage[margin=3.5cm]{geometry}
\usepackage{pdfpages}
\usepackage[round]{natbib}
\usepackage{epsfig}
\usepackage{graphicx}
\usepackage{url}

\usepackage{myAlgorithm}
\usepackage{amsmath,amssymb}
\usepackage{bm}
\usepackage{relsize}
\usepackage{color}
\usepackage{blindtext}
\usepackage{todonotes}
\usepackage{subfig}
\usetikzlibrary{bayesnet,calc,angles,quotes}
\usepackage[utf8]{inputenc}
\usepackage{footmisc}
\usepackage{tikz}
\usepackage{bbold}

\newtheorem{problem}{Problem}
\DeclareMathOperator{\Tr}{Tr}
\DeclareMathOperator{\Sp}{\mathrm{Sp}}
\DeclareMathOperator{\Diag}{\mathrm{Diag}}
\DeclareMathOperator{\rank}{\mathrm{rk}}
\DeclareMathOperator{\Det}{Det}
\DeclareMathOperator{\Vol}{Vol}
\DeclareMathOperator{\Span}{\mathrm{Span}}
\DeclareMathOperator{\Fr}{\mathrm{Fr}}
\DeclareMathOperator{\DPP}{\mathrm{DPP}}

\DeclareMathOperator{\VS}{\mathrm{VS}}
\DeclareMathOperator{\eff}{\mathrm{eff}}

\DeclareMathOperator{\Kerspace}{\mathrm{Ker}}

\DeclareMathOperator{\Sinmatrix}{\mathcal{S}}
\DeclareMathOperator{\Cosmatrix}{\mathcal{C}}
\DeclareMathOperator{\Tanmatrix}{\mathcal{T}}

\DeclareMathOperator{\Tran}{\intercal}
\def\rk{\text{rk}}

\DeclareMathOperator{\EX}{\mathbb{E}}
\DeclareMathOperator{\Prb}{\mathbb{P}}

\newtheorem{example}{Example} 
\newtheorem{theorem}{Theorem}
\newtheorem{lemma}[theorem]{Lemma} 
\newtheorem{proposition}[theorem]{Proposition} 

\newtheorem{corollary}[theorem]{Corollary}
\newtheorem{definition}[theorem]{Definition}

\newcommand{\BlackBox}{\rule{1.5ex}{1.5ex}}  
\newenvironment{proof}{\par\noindent{\bf Proof\ }}{\hfill\BlackBox\\[2mm]}

\newcommand{\rb}[1]{}

\newcommand{\ar}[1]{\textcolor{magenta}{~\algoremark{#1}}}
\usepackage[colorlinks=true,linkcolor = blue,linktoc=page,citecolor=blue,allbordercolors={1 1 1}]{hyperref}
\usepackage{amsmath}

\DeclareMathOperator*{\argmin}{arg\,min}

\usepackage{placeins}

\newcommand*\newprec{\vcenter{\hbox{\includegraphics[width=0.5cm]{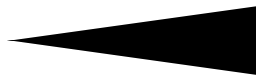}}}}

\newcommand*\newtiltedprec{\vcenter{\hbox{\rotatebox{-30}{\includegraphics[width=0.5cm]{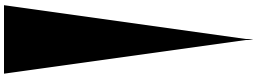}}}}}

\title{A determinantal point process for column subset selection}

\author{
Ayoub Belhadji$^{1}$\footnote{Corresponding author: \href{mailto:ayoub.belhadji@centralelille.fr}{ayoub.belhadji@centralelille.fr}}, R\'emi Bardenet$^1$, Pierre Chainais$^1$\\
\small $^1$ Univ. Lille, CNRS, Centrale Lille, UMR 9189 - CRIStAL, 59651 Villeneuve d’Ascq, France \\
}

\date{}

\begin{document}

\maketitle

\begin{abstract}
Dimensionality reduction is a first step of many machine learning pipelines. Two popular approaches are principal component analysis, which projects onto a small number of well chosen but non-interpretable directions, and feature selection, which selects a small number of the original features. Feature selection can be abstracted as a numerical linear algebra problem called the column subset selection problem (CSSP). CSSP corresponds to selecting the best subset of columns of a matrix $\bm{X} \in \mathbb{R}^{N \times d}$, where \emph{best} is often meant in the sense of minimizing the approximation error, i.e., the norm of the residual after projection of $\bm{X}$ onto the space spanned by the selected columns. Such an optimization over subsets of $\{1,\dots,d\}$ is usually impractical. One workaround that has been vastly explored is to resort to polynomial-cost, random subset selection algorithms that favor small values of this approximation error. We propose such a randomized algorithm, based on sampling from a projection determinantal point process (DPP), a repulsive distribution over a fixed number $k$ of indices $\{1,\dots,d\}$ that favors diversity among the selected columns. We give bounds on the ratio of the expected approximation error for this DPP over the optimal error of PCA. These bounds improve over the state-of-the-art bounds of \emph{volume sampling} when some realistic structural assumptions are satisfied for $\bm{X}$. Numerical experiments suggest that our bounds are tight, and that our algorithms have comparable performance with the \emph{double phase} algorithm, often considered to be the practical state-of-the-art. Column subset selection with DPPs thus inherits the best of both worlds: good empirical performance and tight error bounds.
\end{abstract}


\tableofcontents

\section{Introduction}
\label{s:introduction}


Datasets come in always larger dimensions, and dimension reduction is thus often one the first steps in any machine learning pipeline. Two of the most widespread strategies are principal component analysis (PCA) and feature selection. PCA projects the data in directions of large variance, called principal components. While the initial features (the canonical coordinates) generally have a direct interpretation, principal components are linear combinations of these original variables, which makes them hard to interpret. On the contrary, using a selection of original features will preserve interpretability when it is desirable. Once the data are gathered in an $N\times d$ matrix, of which each row is an observation encoded by $d$ features, feature selection boils down to selecting columns of $\bm{X}$. Independently of what comes after feature selection in the machine learning pipeline, a common performance criterion for feature selection is the approximation error in some norm, that is, the norm of the residual after projecting $\bm{X}$ onto the subspace spanned by the selected columns. Optimizing such a criterion over subsets of $\{1,\dots,d\}$ requires exhaustive enumeration of all possible subsets, which is prohibitive in high dimension. One alternative is to use a polynomial-cost, random subset selection strategy that favors small values of the criterion.

This rationale corresponds to a rich literature on randomized algorithms for column subset selection \citep{DeVe06,DrMaMu07,BoDrMI11}. A prototypal example corresponds to sampling $s$ columns of $\bm{X}$ i.i.d. from a multinomial distribution of parameter $\bm{p} \in \mathbb{R}^{d}$. This parameter $\bm{p}$ can be the squared norms of each column \citep{DFKVV04}, for instance, or the more subtle $k$-leverage scores \citep{DrMaMu07}. While the former only takes $\mathcal{O}(d N^{2})$ time to evaluate, it comes with loose guarantees; see Section~\ref{subsec:length_square_sampling}. The $k$-leverage scores are more expensive to evaluate, since they call for a truncated SVD of order $k$, but they come with tight bounds on the ratio of their expected approximation error over that of PCA.

To minimize approximation error, the subspace spanned by the selected columns should be as large as possible. Simultaneously, the number of selected columns should be as small as possible, so that intuitively, diversity among the selected columns is desirable. The column subset selection problem (CSSP) then becomes a question of designing a discrete point process over the column indices $\{1,\dots,d\}$ that favors diversity in terms of directions covered by the corresponding columns of $\bm{X}$.
Beyond the problem of designing such a point process, guarantees on the resulting approximation error are desirable. Since, given a target dimension $k\leq d$ after projection, PCA provides the best approximation in Frobenius or spectral norm, it is often used a reference: a good CSS algorithm preserves interpretability of the $c$ selected features while guaranteeing an approximation error not much worse than that of rank-$k$ PCA, all of this with $c$ not much larger than $k$.

In this paper, we introduce and analyze a new randomized algorithm for selecting $k$ diverse columns. Diversity is ensured using a determinantal point process (DPP). DPPs can be viewed as the kernel machine of point processes; they were introduced by \cite{Mac75} in quantum optics, and their use widely spread after the 2000s in random matrix theory \citep{Joh05}, machine learning \citep{KuTa12}, spatial statistics \citep{LaMoRu15}, and Monte Carlo methods \citep{BaHa16}, among others. In a sense, the DPP we propose is a nonindependent generalization of the multinomial sampling with $k$-leverage scores of \citep{BoMaDr09}. It further naturally connects to volume sampling, the CSS algorithm that has the best error bounds \citep{DRVW06}. We give error bounds for DPP sampling that exploit sparsity and decay properties of the $k$-leverage scores, and outperform volume sampling when these properties hold. Our claim is backed up by experiments on toy and real datasets.

The paper is organized as follows. Section~\ref{s:notation} introduces our notation. Section~\ref{s:relatedwork} is a survey of column subset selection, up to the state of the art to which we later compare. In Section~\ref{s:dppsection}, we discuss determinantal point processes and their connection to volume sampling. Section~\ref{sec:main_results} contains our main results, in the form of both classical bounds on the approximation error and risk bounds when CSS is a prelude to linear regression. In Section~\ref{s:numexpesection}, we numerically compare CSS algorithms, using in particular a routine that samples random matrices with prescibed $k$-leverage scores.

\section{Notation}
\label{s:notation}
We use $[n]$ to denote the set $\{1,\dots,n\}$, and $[n:m]$ for $\{n,\dots,m\}$. We use bold capitals $\bm{A},\bm{X},\dots$ to denote matrices
. For a matrix $\bm{A} \in \mathbb{R}^{m \times n}$ and subsets of indices $I\subset[m]$ and $J\subset[n]$, we denote by $\bm{A}_{I,J}$ the submatrix of $\bm{A}$ obtained by keeping only the rows indexed by $I$ and the columns indexed by $J$. When we mean to take all rows or $\bm{A}$, we write $\bm A_{:,J}$, and similarly for all columns. We write $\rk (\bm{A})$ for the rank of $\bm{A}$, and $\sigma_i(\bm A)$, $i=1,\dots,\rk(\bm{A})$ for its singular values, ordered decreasingly. Sometimes, we will need the vectors $\Sigma(\bm{A})$ and $\Sigma(\bm{A})^2$ the vectors of $\mathbb{R}^d$ with respective entries $\sigma_i(\bm{A})$ and $\sigma_i^2(\bm{A})$, $i=1,\dots,\rk(\bm{A})$. Similarly, when $\bm{A}$ can be diagonalized, $\Lambda(\bm{A})$ (and $\Lambda(\bm{A})^2$) are vectors with the decreasing eigenvalues (squared eigenvalues) of $\bm{A}$ as entries.

The spectral norm of $\bm{A}$ is $\|\bm{A}\|_{2} = \sigma_{1}(\bm A)$, while the Frobenius norm of $\bm{A}$ is defined by
$$\Vert\bm{A}\Vert_{\Fr} = \sqrt{\sum_{i=1}^{\rk(\bm{A})} \sigma_{i}(\bm{A})^{2}}.$$
For $\ell \in \mathbb{N}$, we need to introduce the $\ell$-th elementary symmetric polynomial on $L \in \mathbb{N}$ variables, that is
\begin{equation}
e_{\ell}(X_{1}, \dots, X_{L}) = \sum\limits_{\substack{T \subset [L]\\|T| = \ell}}~ \prod\limits_{j \in T} X_{j}.
\end{equation}
Finally, we follow \cite{Ben92} and denote spanned volumes by
$$\Vol_{q}(\bm{A}) = \sqrt{e_{q}\left(\sigma_{1}(\bm{A})^{2},\dots,\sigma_{\rk(A)}(\bm{A})^2\right)}, \quad q=1,\dots,\rk(\bm{A}).$$

Throughout the paper, $\bm{X}$ will always denote an $N\times d$ matrix that we think of as the original data matrix, of which we want to select $k\leq d$ columns. Unless otherwise specified, $r$ is the rank of $\bm{X}$, and matrices $\bm{U},\bm{\Sigma}$,$\bm{V}$ are reserved for the SVD of $\bm{X}$, that is,
\begin{align}
   \bm{X} &= \bm{U}\bm{\Sigma}\bm{V}^{T}\\
   &=
\left[
\begin{array}{c|c}
\bm{U}_{k} & \bm{U}_{r-k}
\end{array}
\right]
\left[
\begin{array}{c|c}
\bm{\Sigma}_{k} & \bm{0} \\
\hline
\bm{0} & \bm{\Sigma}_{r-k}
\end{array}
\right]
\left[
\begin{array}{c}
\bm{V}_{k}^{T} \\
\hline
\bm{V}_{r-k}^{T}
\end{array}
\right],
\label{e:blockSVD}
\end{align}
where $\bm{U} \in \mathbb{R}^{N\times r}$ and $\bm{V} \in \mathbb{R}^{d \times r}$ are orthogonal, and $\bm{\Sigma} \in \mathbb{R}^{r \times r}$ is diagonal. The diagonal entries of $\Sigma$ are denoted by $\sigma_i=\sigma_i(\bm{X})$, $i=1,\dots,r$, and we assume they are in decreasing order. We will also need the blocks given in \eqref{e:blockSVD}, where we separate blocks of size $k$ corresponding to the largest $k$ singular values. To simplify notation, we abusively write $\bm{U}_{k}$ for $\bm{U}_{:,[k]}$ and $\bm{V}_{k}$ for $\bm{V}_{:,[k]}$ in \eqref{e:blockSVD}, among others. Though they will be introduced and discussed at length in Section~\ref{subsec:k-lvs_sampling}, we also recall here that we note $\ell_{i}^{k} = \|\bm{V}_{[k],i}\|_{2}^{2}$ the so-called \emph{ $k$-leverage score} of the $i$-th column of $\bm{X}$.

We need some notation for the selection of columns. Let $S \subset [d]$ be such that $\vert S\vert =k$, and let $\bm{S} \in \{0,1\}^{d \times k}$ be the corresponding sampling matrix: $\bm{S}$ is defined by $\forall \bm{M} \in \mathbb{R}^{N\times d}, \bm{M}\bm{S} = \bm{M}_{:,S}$. In the context of column selection, it is often referred to $\bm{X}\bm{S} = \bm{X}_{:,S}$ as $\bm{C}$. We set for convenience $\bm{Y}_{:,S}^{\Tran} = (\bm{Y}_{:,S})^{\Tran}$.

The result of column subset selection will usually be compared to the result of PCA. We denote by $\Pi_k\bm{X}$ the best rank-$k$ approximation to $\bm{X}$. The sense of \emph{best} can be understood either in Frobenius or spectral norm, as both give the same result. On the other side, for a given subset $S \subset [d]$ of size $\vert S\vert=s$ and $\nu\in\{2,\Fr\}$, let
$$\Pi_{S,k}^{\nu}\bm{X} = \arg\min_{A} \Vert \bm{X}- A\Vert_{\nu}$$
where the minimum is taken over all matrices $\bm{A} = \bm{X}_{:,S}\bm{B}$ such that $\bm{B}\in\mathbb{R}^{s\times d}$ and $\rank \bm{B}\leq k$; in words, the minimum is taken over matrices of rank at most $k$ that lie in the column space of $\bm{C}=\bm{X}_{:,S}$. When $|S| = k$, we simply write $\Pi_{S}^{\nu}\bm{X} = \Pi_{S,k}^{\nu}\bm{X}$. In practice, the Frobenius projection can be computed as $\Pi_{S}^{\Fr}\bm{X} = \bm{C}\bm{C}^{+}\bm{X}$, yet there is no simple expression for $\Pi_{S}^{2}\bm{X}$. However, $\Pi_{S}^{\Fr}\bm{X}$ can be used as an approximation of $\Pi_{S}^{2}\bm{X}$ since
\begin{equation}\label{eq:frob_as_estimator_of_spe}
\|\bm{X} - \Pi_{S}^{2}\bm{X}\|_{2} \leq \|\bm{X} - \Pi_{S}^{\Fr}\bm{X}\|_{2}  \leq \sqrt{2} \|\bm{X} - \Pi_{S}^{2}\bm{X}\|_{2},
\end{equation}
see \citep[Lemma 2.3]{BoDrMI11}.

\section{Related Work}
\label{s:relatedwork}
In this section, we review the main results about column subset selection.
\subsection{Rank revealing QR decompositions}\label{sec:rrqr}
The first $k$-CSSP algorithm can be traced back to the article of \cite{Golu65} on pivoted QR factorization. This work introduced the concept of Rank Revealing QR factorization (RRQR). The original motivation was to calculate a well-conditioned QR factorization of a matrix $\bm{X}$ that reveals its numerical rank.
\begin{definition}
  \label{d:rrqr}
Let $\bm{X} \in \mathbb{R}^{N \times d}$  and $k \in \mathbb{N}$ $ (k \leq d)$. A RRQR factorization of $\bm{X}$ is a 3-tuple $(\bm{\Pi},\bm{Q},\bm{R})$ with $\bm{\Pi} \in \mathbb{R}^{d \times d}$ a permutation matrix, $\bm{Q} \in \mathbb{R}^{N \times d}$ an orthogonal matrix, and $\bm{R} \in \mathbb{R}^{d \times d}$ a triangular matrix, such that $\bm{X}\bm{\Pi} = \bm{Q} \bm{R}$,
\begin{equation}
    \frac{\sigma_{k}(\bm{X})}{p_{1}(k,d)} \leq \sigma_{min}(\bm{R}_{[k],[k]}) \leq \sigma_{k}(\bm{X}) \:,
\end{equation}
and
\begin{equation}
\label{eq:rrqr_singular_values_inequality}
    \sigma_{k+1}(\bm{X}) \leq \sigma_{max}(\bm{R}_{[k+1:d],[k+1:d]}) \leq p_{2}(k,d)\sigma_{k+1}(\bm{X}),
\end{equation}
where
$p_{1}(k,d)$ and $p_{2}(k,d)$ are controlled.
\end{definition}

In practice, a RRQR factorization algorithm interchanges pairs of columns and updates or builds a QR decomposition on the fly.
The link between RRQR factorization and k-CSSP was first discussed by \cite*{BoMaDr09}. The structure of a RRQR factorization indeed gives a deterministic selection of a subset of $k$ columns of $\bm{X}$. More precisely, if we take $\bm{C}$ to be the first $k$ columns of $\bm{X}\bm{\Pi}$, $\bm{C}$ is a subset of columns of $\bm{X}$ and $\| \bm{X} - \Pi_{S}^{\Fr}\bm{X}\|_{2} = \| \bm{R}_{[k+1:r],[k+1:r]}\|_{2}$. By \eqref{eq:rrqr_singular_values_inequality}, any RRQR algorithm thus provides provable guarantees in spectral norm for $k$-CSSP.

Following \citep{Golu65}, many papers gave algorithms that improved on $p_{1}(k,d)$ and $p_{2}(k,d)$ in Definition~\ref{d:rrqr}. Table~\ref{table:rrqr_examples} sums up the guarantees of the original algorithm of \citep{Golu65} and the state-of-the-art algorithms of \cite{GuEi96}. Note the dependency of $p_2(k,d)$ on the dimension $d$ through the term $\sqrt{d-k}$; this term is common for guarantees in spectral norm for $k$-CSSP. We refer to \citep{BoMaDr09} for an exhaustive survey on RRQR factorization.
\FloatBarrier
\begin{table}
\centering
 \begin{tabular}{| c| c| c| c|}
 \hline
  Algorithm & $p_2(k,d)$ & Complexity &  References\\
 \hline
 Pivoted QR & $2^{k}\sqrt{d-k}$ & $\mathcal{O}(dNk)$ & \citep{GoVa96}\\
 \hline
 Strong RRQR (Alg. 3) & $\sqrt{(d-k) k +1}$ & not polynomial & \citep{GuEi96} \\
 \hline
 Strong RRQR (Alg. 4) & $\sqrt{f^{2}(d-k) k +1}$ & $\mathcal{O}(dNk \log_{f}(d))$ & \citep{GuEi96}\\
 \hline
\end{tabular}
\caption{Examples of some RRQR algorithms and their theoretical performances.\label{table:rrqr_examples}}
\end{table}

\subsection{Length square importance sampling and additive bounds}\label{subsec:length_square_sampling}

\cite*{DFKVV04} proposed a randomized CSS algorithm based on i.i.d. sampling $s$ indices $S=\{i_1,\dots,i_s\}$ from a multinomial distribution of parameter $\bm{p}$, where
\begin{equation}
 p_{j} = \frac{\|\bm{X}_{:,j}\|_{2}^{2}}{\|\bm{X}\|_{\Fr}^{2}} \, , j \in [d].
\end{equation}
Let $\bm{C} = \bm{X}_{:,S}$ be the corresponding submatrix. First, we note that some columns of $\bm{X}$ may appear more than once in $\bm{C}$. Second, \cite[Theorem 3]{DFKVV04} states that
\begin{equation}\label{eq:length_square_sampling_result}
    \Prb \left( \| \bm{X} - \Pi_{S,k}^{\Fr}\bm{X}\|_{\Fr}^{2} \leq \|\bm{X} - \Pi_{k}\bm{X}\|_{\Fr}^{2} + 2(1+\sqrt{8\log(\frac{2}{\delta}})) \sqrt{\frac{k}{s}}\| \bm{X}\|_{\Fr}^{2} \right) \geq 1-\delta.
\end{equation}
Equation \eqref{eq:length_square_sampling_result} is a high-probability additive upper bound for $\| \bm{X} - \Pi_{S}^{\Fr}\bm{X}\|_{\Fr}^{2}$. The drawback of such bounds is that they can be very loose if the first $k$ singular values of $\bm{X}$ are large compared to $\sigma_{k+1}$. For this reason, multiplicative approximation bounds have been considered.

\subsection{$k$-leverage scores sampling and multiplicative bounds}
\label{subsec:k-lvs_sampling}
\citet*{DrMaMu07} proposed an algorithm with provable multiplicative upper bound using multinomial sampling, but this time according to $k$-leverage scores.
\begin{definition}[$k$-leverage scores]
Let $\bm{X} = \bm{U}\bm{\Sigma}\bm{V}^{\Tran} \in\mathbb{R}^{N\times d}$ be the SVD of $\bm{X}$. We note $\bm{V}_{k} = \bm{V_{:,[k]}}$ the first $k$ columns of $\bm{V}$. For $j \in [d]$, the $k$-leverage score of the $j$-th column of  $\bm{X}$ is defined by
\begin{equation}
 \ell^{k}_{j} = \sum\limits_{i = 1}^{k} V_{i,j}^{2} .
\end{equation}
\end{definition}
In particular, it holds
\begin{equation}
\sum\limits_{j \in [d]} \ell^{k}_{j} = \sum\limits_{j \in [d]} \| (\bm{V}_{k}^{\Tran})_{:,j}\|_{2}^{2}= \Tr(\bm{V}^{}_{k}\bm{V}^{\Tran}_{k}) = k,
\end{equation}
since $\bm{V}^{}_{k}$ is an orthogonal matrix. Therefore, one can consider the multinomial distribution on $[d]$ with parameters
\begin{equation}
p_{j} = \frac{\ell^{k}_{j}}{k} \:\: , j \in [d].
\end{equation}
This multinomial is called the \emph{$k$-leverage scores distribution}.

\begin{theorem}[\citealp{DrMaMu07}, Theorem 3]
If the number $s$ of sampled columns satisfies
\begin{equation}
s \geq \frac{4000 k^{2}}{\epsilon^{2}}\log\left(\frac{1}{\delta}\right),
\end{equation}
then, under the $k$-leverage scores distribution,
\begin{equation}\label{eq:relative_error_bound}
    \Prb \bigg( \| \bm{X} - \Pi_{S,k}^{\Fr}\bm{X}\|_{\Fr}^{2} \leq (1 + \epsilon) \| \bm{X} - \Pi_{k}\bm{X}\|_{\Fr}^{2} \bigg) \geq 1-\delta .
\end{equation}
\end{theorem}
\citet{DrMaMu07} also considered replacing multinomial with Bernoulli sampling, still using the $k$-leverage scores. The expected number of columns needed for \eqref{eq:relative_error_bound} to hold is then lowered to $\mathcal{O}(\frac{k \log k}{\epsilon^{2}})$. A natural question is then to understand how low the number of columns can be, while still guaranteeing a multiplicative bound like \eqref{eq:relative_error_bound}. A partial answer has been given by \cite{DeVe06}.
\begin{proposition}[\citealp{DeVe06}, Proposition 4]\label{prop:lower_bound_cssp}
Given $\epsilon > 0$, $k,d \in \mathbb{N}$ such that $d \epsilon \geq 2k$, there exists a matrix $\bm{X}^{\epsilon} \in \mathbb{R}^{kd \times k(d+1)}$ such that for any $S \subset [d]$,
\begin{equation}
\| \bm{X}^{\epsilon} - \Pi_{S,k}^{\Fr} \bm{X}^{\epsilon} \|_{\Fr}^{2} \geq (1+\epsilon) \| \bm{X}^{\epsilon} - \bm{X}^{\epsilon}_{k} \|_{\Fr}^{2}.
\end{equation}
\end{proposition}
This suggests that the lower bound for the number of columns is $k/\epsilon$, at least in the worst case sense of Proposition~\ref{prop:lower_bound_cssp}. Interestingly, the $k$-leverage scores distribution of the matrix $\bm{X}^{\epsilon}$ in the proof of Proposition~\ref{prop:lower_bound_cssp} is uniform, so that $k$-leverage score sampling boils down to simple uniform sampling.

To match the lower bound of \cite{DeVe06}, \citet*{BoDrMI11} proposed a greedy algorithm to select columns. This algorithm is inspired by the sparsification of orthogonal matrices proposed in \citep{BaSpSr09}. The full description of this family of algorithms is beyond the scope of this article. We only recall one of the results of the article.
\begin{theorem}[\citealp{BoDrMI11}, Theorem 1.5]
There exists a randomized algorithm $\mathcal{A}$ that select at most $c = \frac{2 k}{\epsilon}(1 + o(1))$ columns of $\bm{X}$ such that
\begin{equation}
\EX_{\mathcal{A}} \| \bm{X}-\Pi_{S,k}^{\Fr}\bm{X}\|_{\Fr}^{2} \leq (1+ \epsilon) \| \bm{X} - \Pi_{k}\bm{X} \|_{\Fr}^{2}.
\end{equation}
\end{theorem}

Finally, a deterministic algorithm based on $k$-leverage score sampling was proposed by \citet*{PaKyBo14}. The algorithm selects the $c(\theta)$ columns of $\bm{X}$ with the largest $k$-leverage scores, where
\begin{equation}
    c(\theta) \in \argmin\limits_{u}\left(\sum_{i = 1}^{u}\ell_{i}^{k} > \theta\right),
\end{equation}
and $\theta$ is a free parameter that controls the approximation error. To guarantee that there exists a matrix of rank $k$ in the subspace spanned by the selected columns, \cite{PaKyBo14} assume that
\begin{equation}
0 \leq k-\theta < 1.
\end{equation}
Loosely speaking, this condition is satisfied for a low value of $c(\theta)$ if the $k$-leverage scores (after ordering) are decreasing rapidly enough. The authors give empirical evidence that this condition is satisfied by a large proportion of real datasets.

\begin{theorem}[\citealp{PaKyBo14}, Theorem 2]
Let $\epsilon = k-\theta \in [0,1)$, letting $S$ index the columns with the $c(\theta)$ largest $k$-leverage scores,
\begin{equation}
\| \bm{X} - \Pi_{S,k}^{\nu}\bm{X}\|_{\nu} \leq \frac{1}{1-\epsilon}\|\bm{X} - \Pi_{k}\bm{X}\|_{\nu},\quad \nu\in\{2,\Fr\}.
\end{equation}
In particular, if $\epsilon \in [0,\frac{1}{2}]$,
\begin{equation}
\| \bm{X} - \Pi_{S,k}^{\nu}\bm{X}\|_{\nu} \leq (1+2\epsilon)\|\bm{X} - \Pi_{k}\bm{X}\|_{\nu},\quad \nu\in\{2,\Fr\}.
\end{equation}
\end{theorem}

Furthermore, they proved that if the $k$-leverage scores decay like a power law, the number of columns needed to obtain a multiplicative bound can actually be smaller than $\frac{k}{\epsilon}$.
\begin{theorem}[\citealp{PaKyBo14}, Theorem 3]
Assume, for $\eta > 0$,
\begin{equation}
\ell_{i}^{k} = \frac{\ell_{1}^{k}}{i^{\eta +1}}.
\end{equation}
Let $\epsilon = k-\theta\in [0,1)$, then
\begin{equation}
c(\theta) = \max \bigg\{ \left(\frac{4k}{\epsilon}\right)^{\frac{1}{\eta + 1}} - 1, \left(\frac{4k}{\eta \epsilon}\right)^{\frac{1}{\eta}}, k \bigg\}.
\end{equation}
\end{theorem}
This complements the fact that the worst case example in Proposition~\ref{prop:lower_bound_cssp} had uniform $k$-leverage scores. Loosely speaking, matrices with fast decaying $k$-leverage scores can be efficiently subsampled.

\subsection{The geometric interpretation of the $k$-leverage scores}
The $k$-leverage scores can be given a geometric interpretation, the generalization of which serves as a first motivation for our work.

For $i \in [d]$, let $\bm{e}_{i}$ be the $i$-th canonical basis vector of $\mathbb{R}^{d}$. Let further $\theta_{i}$ be the angle between $\bm{e}_{i}$ and the subspace $\mathcal{P}_{k} = \Span(\bm{V}_{k})$, and denote by $\Pi_{\mathcal{P}_{k}}\bm{e}_{i}$ the orthogonal projection of $\bm{e}_{i}$ onto the subspace $\mathcal{P}_{k}$. Then
\begin{equation}
\cos^{2}(\theta_{i}) := \frac{(\bm{e}_{i},\Pi_{\mathcal{P}_{k}}\bm{e}_{i})^2}{\Vert \Pi_{\mathcal{P}_{k}}\bm{e}_{i}\Vert^2} = (\bm{e}_{i},\Pi_{\mathcal{P}_{k}}(\bm{e}_{i})) = (\bm{e}_{i}, \sum\limits_{j =1}^{k} V_{i,j}\bm{V}_{:,j}) = \sum\limits_{j =1}^{k} V_{i,j}^{2} = \ell^{k}_{i}.
\end{equation}
A large $k$-leverage score $\ell_i^k$ thus indicates that $\bm{e}_i$ is almost aligned with $\mathcal{P}_{k}$. Selecting columns with large $k$-leverage scores as in \citep{DrMaMu07} can thus be interpreted as replacing the principal eigenspace $\mathcal{P}_k$ by a subspace that must contain $k$ of the original coordinate axes. Intuitively, a closer subspace to the original $\mathcal{P}_k$ would be obtained by selecting columns \emph{jointly} rather than independently, considering the angle with $\mathcal{P}_k$ of the subspace spanned by these columns. More precisely, consider $S \subset [d], |S| = k$, and denote $\mathcal{P}_{S} = \Span(\bm{e}_{j}, j \in S)$. A natural definition of the cosine between $\mathcal{P}_{k}$ and $\mathcal{P}_{S}$ is in terms of the so-called \emph{principal angles} \cite[Section 6.4.4]{GoVa96}; see Appendix~\ref{app:principal_angles}. In particular, Proposition~\ref{principal_angles_theorem_1} in Appendix~\ref{app:principal_angles} yields
\begin{equation}
\cos^{2}(\mathcal{P}_{k},\mathcal{P}_{S}) = \Det(\bm{V}_{S,[k]})^{2}.
\label{e:cosines}
\end{equation}
This paper is about sampling $k$ columns proportionally to \eqref{e:cosines}.

In Appendix~\ref{app:statisticalInterpretationOfLVSs}, we contribute a different interpretation of $k$-leverage scores and volumes, which relates them to the length-square distribution of Section~\ref{subsec:length_square_sampling}.

 \subsection{Negative correlation: volume sampling and the double phase algorithm}\label{subsec:volume_sampling}
 In this section, we survey algorithms that randomly sample exactly $k$ columns from $\bm{X}$, unlike the multinomial sampling schemes of Sections~\ref{subsec:length_square_sampling} and \ref{subsec:k-lvs_sampling}, which typically require more than $k$ columns.

 \citet*{DRVW06} obtained a multiplicative bound on the expected approximation error, with only $k$ columns, using so-called \emph{volume sampling}.
\begin{theorem}[\citealp{DRVW06}]
  \label{thrm:volume_sampling_theorem}
Let $S$ be a random subset of $[d]$, chosen with probability
\begin{equation}
\Prb_{\VS}(S) = Z \Det(\bm{X}_{:,S}^{\Tran}\bm{X}_{:,S}^{}) \mathbb{1}_{\{|S| = k \}},
\label{e:vs}
\end{equation}
where $Z = \sum\limits_{|S| = k} \Det(\bm{X}_{:,S}^{\Tran}\bm{X}_{:,S}^{})$.
Then
\begin{equation}
\EX_{\VS} \| \bm{X} - \Pi_{S}^{\Fr}\bm{X} \|_{\Fr}^{2} \leq (k+1)\| \bm{X} - \Pi_{k}\bm{X} \|_{\Fr}^{2}
\label{e:vsBoundFr}
\end{equation}
and
\begin{equation}
\EX_{\VS} \| \bm{X} - \Pi_{S}^{2}\bm{X} \|_{2}^{2} \leq (d-k)(k+1)\| \bm{X} - \Pi_{k}\bm{X} \|_{\Fr}^{2} .
\label{e:vsBound2}
\end{equation}
\end{theorem}
Note that the bound for the spectral norm was proven in \citep{DRVW06} for the Frobenius projection, that is, they bound $\| \bm{X} - \Pi_{S}^{\Fr}\bm{X} \|_{2}$. The bound \eqref{e:vsBound2} easily follows from \eqref{eq:frob_as_estimator_of_spe}. Later, sampling according to \eqref{e:vs} was shown to be doable in polynomial time \citep{DeRa10}. Using a worst case example, \cite{DRVW06} proved that the $k+1$ factor in \eqref{e:vsBoundFr} cannot be improved.
\begin{proposition}[\citealp{DRVW06}]\label{prop:volume_sampling_lower_bound}
Let $\epsilon > 0$. There exists a $(k+1) \times (k+1)$ matrix $\bm{X}^{\epsilon}$ such that for every subset $S$ of $k$ columns of $\bm{X}^{\epsilon}$,
\begin{equation}
\| \bm{X}^{\epsilon} - \Pi_{S}^{\Fr}\bm{X}^{\epsilon} \|_{\Fr}^{2} >(1-\epsilon) (k+1)\| \bm{X}^{\epsilon} - \Pi_{k}\bm{X}^{\epsilon} \|_{\Fr}^{2}.
\end{equation}
\end{proposition}
We note that there has been recent interest in a similar but different distribution called \emph{dual volume sampling} \citep{AvBo13,LiJeSr17,DeWa18}, sometimes also termed volume sampling. The main application of dual VS is row subset selection of a matrix $\bm{X}$ for linear regression on label budget constraints.

\citep{BoMaDr09} proposed a $k$-CSSP algorithm, called \emph{double phase}, that combines ideas from multinomial sampling and RRQR factorization. The motivating idea is that the theoretical performance of RRQR factorizations depends on the dimension through a factor $\sqrt{d-k}$; see Table~\ref{table:rrqr_examples}. To improve on that, the authors propose to first reduce the dimension $d$ to $c$ by preselecting a large number of columns $c > k$ using multinomial sampling from the $k$-leverage scores distribution, as in Section~\ref{subsec:k-lvs_sampling}. Then only, they perform a RRQR factorization of the reduced matrix $\bm{V}_{k}^{\Tran}\bm{S}_{1}\bm{D}_{1} \in \mathbb{R}^{k \times c}$, where $\bm{S}_{1} \in \mathbb{R}^{d \times c}$ is the sampling matrix of the multinomial phase and $\bm{D}_{1} \in \mathbb{R}^{c \times c}$ is a scaling matrix.

\begin{theorem}[\citealp{BoMaDr09}]\label{thrm:double_phase_theorem}
Let $S$ be the output of the double phase algorithm with $c = \Theta(k \log k)$. Then
\begin{equation}\label{eq:frob_double_phase}
\Prb_{\emph{DPh}} \Bigg( \| \bm{X} - \Pi_{S}^{\Fr}\bm{X} \|_{\Fr} \leq \Theta(k \log^{\frac{1}{2}} k)\| \bm{X} - \Pi_{k}\bm{X} \|_{\Fr} \Bigg) \geq 0.8 \: .
\end{equation}
\begin{equation}\label{eq:spectral_double_phase}
\Prb_{\emph{DPh}} \Bigg( \| \bm{X} - \Pi_{S}^{2}\bm{X} \|_{2} \leq \Theta(k \log^{\frac{1}{2}} k)\| \bm{X} - \Pi_{k}\bm{X} \|_{2} + \Theta(k^{\frac{3}{4}} \log^{\frac{1}{4}} k)\|\bm{X} - \Pi_{k}\bm{X} \|_{\Fr} \Bigg) \geq 0.8 \: .
\end{equation}

\end{theorem}

Note that the spectral norm bound was proven for a slightly different distribution in the randomized phase. Furthermore this bound was proved in \citep{DRVW06} for $\| \bm{X} - \Pi_{S}^{\Fr}\bm{X} \|_{2}$ but using \eqref{eq:frob_as_estimator_of_spe} the bound \eqref{eq:spectral_double_phase} follows. The constants $\Theta(k \log^{\frac{1}{2}} k)$ and $\Theta(k^{\frac{3}{4}} \log^{\frac{1}{4}} k)$ in the bounds \eqref{eq:frob_double_phase} and \eqref{eq:spectral_double_phase} depends on $c$ the number of pre-selected columns in the randomized step. In practice, the choice of the parameter $c$ of the randomized pre-selection phase has an influence on the quality of the approximation.
We refer to \citep{BoMaDr09} for details.

\subsection{Excess risk in sketched linear regression}
\label{s:related_work_sparse_regression}
So far, we have focused on approximation bounds in spectral or Frobenius norm for $\bm{X}~-~\Pi_{S,k}^\nu\bm{X}$. This is a reasonable measure of error as long as it is not known what the practitioner wants to do with the submatrix $\bm{X}_{:,S}$. In this section, we assume that the ultimate goal is to perform linear regression of some $\bf{y}\in\mathbb{R}^N$ onto $\bm{X}$. Other measures of performance then become of interest, such as the excess risk incurred by regressing onto $\bm{X}_{:,S}$ rather than $\bm{X}$. We use here the framework of \cite{Sla18}, further assuming well-specification for simplicity.

For every $i \in [N]$, assume $y_i = \bm{X}_{i,:}\bm{w}^*+\xi_i$, where the noises $\xi_i$ are i.i.d. with mean $0$ and variance $v$. 
For a given estimator $\bm{w} = \bm{w}(\bm{X},\bm{y})$, its excess risk is defined as
\begin{equation}
\mathcal{E}(\bm{w}) = \EX_{\bm{\xi}} \left[\frac{\|\bm{X}\bm{w}^{*} - \bm{X}\bm{w}\|^{2}_2}{N}\right].
\end{equation}
In particular, it is easy to show that the ordinary least squares (OLS) estimator $\hat{\bm{w}}=\bm{X}^{+}\bm{y}$ has excess risk
\begin{equation}
\mathcal{E}(\hat{\bm{w}})=v\times\frac{\text{rk}(\bm{X})}{N}.
\label{e:riskOLS}
\end{equation}
Selecting $k$ columns indexed by $S$ in $\bm{X}$ prior to performing linear regression yields $\bm{w}_S = (\bm{X}\bm{S})^+\bm{y}\in\mathbb{R}^k$. We are interested in the excess risk of the corresponding sparse vector
$$\hat{\bm{w}}_S := \bm{S}\bm{w}_S = \bm{S}(\bm{X}\bm{S})^{+}\bm{y}\in\mathbb{R}^d$$
which has all coordinates zero, except those indexed by $S$.
\begin{proposition}[Theorem 9, \citealp{LiHa18}]\label{prop:sparse_regression_bound}
Let $S \subset [d]$, such that $|S| = k$. Let $(\theta_{i}(S))_{i \in [k]}$ be the principal angles between $\Span \bm{S}$ and $\Span \bm{V}_{k}$, see Appendix~\ref{app:principal_angles}. Then
\begin{equation}\label{eq:prediction_bound_CSS}
    \mathcal{E}(\hat{\bm{w}}_{S}) \leq \frac{1}{N}\left(1+ \max\limits_{i \in [k]} \tan^{2}\theta_{i}(S)\right) \|\bm{w}^{*}\|^{2} \sigma_{k+1}^{2} + \frac{v k}{N}.
\end{equation}
\end{proposition}
Compared to the excess risk \eqref{e:riskOLS} of the OLS estimator, the second term of the right-hand side of \eqref{eq:prediction_bound_CSS} replaces $\rk\bm{X}$ by $k$. But the price is the first term of the right-hand side of \eqref{eq:prediction_bound_CSS}, which we loosely term \emph{bias}. To interpret this bias term, we first look at the excess risk of the principal component regressor (PCR)

\begin{equation}
\bm{w}_{k}^{*} \in \argmin\limits_{\bm{w} \in \Span \bm{V}_{k}} \EX_{\xi} \left[\|\bm{y} - \bm{X}\bm{w}\|^{2}/N\right].
\end{equation}
\begin{proposition}[Corollary 11, \citealp{LiHa18}]\label{prop:pcr_bound}
\begin{equation}\label{eq:prediction_bound_PCR}
    \mathcal{E}(\bm{w}_{k}^{*}) \leq \frac{\|\bm{w}^{*}\|^{2}\sigma_{k+1}^{2}}{N} + \frac{v k}{N}.
\end{equation}
\end{proposition}
The right-hand side of \eqref{eq:prediction_bound_PCR} is almost that of \eqref{eq:prediction_bound_CSS}, except that the bias term in the CSS risk \eqref{eq:prediction_bound_CSS} is larger by a factor that measures how well the subspace spanned by $S$ is aligned with the principal eigenspace $\bm{V}_k$. This makes intuitive sense: the performance of CSS will match PCR if selecting columns yields almost the same eigenspace.

The excess risk \eqref{eq:prediction_bound_CSS} is yet another motivation to investigate DPPs for column subset selection. We shall see in Section~\ref{sec:bounds_for_regression_under_dpp} that the expectation of \eqref{eq:prediction_bound_CSS} under a well-chosen DPP for $S$ has a particularly simple bias term.

%


\section{Determinantal Point Processes}
\label{s:dppsection}



In this section, we introduce discrete determinantal point processes (DPPs) and the related k-DPPs, of which volume sampling is an example. DPPs were introduced by \cite{Mac75} as probabilistic models for beams of fermions in quantum optics. Since then, DPPs have been thoroughly studied in random matrix theory \citep{Joh05}, and have more recently been adopted in machine learning \citep{KuTa12}, spatial statistics \cite{LaMoRu15}, and Monte Carlo methods \citep{BaHa16}.

\subsection{Definitions}
For all the definitions in this section, we refer the reader to \citep{KuTa12}. Recall that $[d] = \{1,\dots,d\}$.
\begin{definition}[DPP]
Let $\bm{K} \in \mathbb{R}^{d\times d}$ be a positive semi-definite matrix.
A random subset $Y \subset [d]$ is drawn from a DPP of marginal kernel $\bm{K}$ if and only if
\begin{equation}\label{eq:def_dpp}
\forall S \subset [d],\quad \Prb(S \subset Y) = \Det(\bm{K}_{S}),
\end{equation}
where $\bm{K}_{S} = [\bm{K}_{i,j}]_{i,j \in S}$. We take as a convention $\Det(\bm{K}_{\emptyset}) = 1$.
\end{definition}
For a given matrix $\bm{K}$, it is not obvious that \eqref{eq:def_dpp} consistently defines a point process. One sufficient condition is that $\bm{K}$ is symmetric and its spectrum is in $[0,1]$; see \citep{Mac75} and \citep{Sos00}[Theorem 3]. In particular, when the spectrum of $\bm{K}$ is included in $\{0,1\}$, we call $\bm{K}$ a projection kernel and the corresponding DPP a \emph{projection} DPP\footnote{All projection DPPs in this paper have symmetric kernels}. Letting $r$ be the number of unit eigenvalues of its kernel, samples from a projection DPP have fixed cardinality $r$ with probability 1 \cite[Lemma 17]{BeKrPeVi05}.

For symmetric kernels $\bm{K}$, a DPP can be seen as a \emph{repulsive} distribution, in the sense that for all $i,j\in [d]$,
\begin{align}
  \Prb(\{i,j\} \subset Y) &= \bm{K}_{i,i} \bm{K}_{j,j} - \bm{K}^2_{i,j}\\
  &= \Prb(\{i\} \subset Y)\Prb(\{j\} \subset Y) - \bm{K}^2_{i,j}\\
  &\leq \Prb(\{i\} \subset Y)\Prb(\{j\} \subset Y).
\end{align}

Besides projection DPPs, there is another natural way of using a kernel matrix to define a random subset of $[d]$ with prespecified cardinality $k$.
\begin{definition}[$k$-DPP]
Let $\bm{L} \in \mathbb{R}^{d\times d}$ be a positive semi-definite matrix.
A random subset $Y \subset [d]$ is drawn from a $k$-DPP of kernel $\bm{L}$ if and only if
\begin{equation}\label{eq:def_kdpp}
\forall S \subset [d],\quad \Prb(Y = S) \propto \mathbb{1}_{\{\vert S\vert = k\}}\Det(\bm{L}_{S})
\end{equation}
where $\bm{L}_{S} = [\bm{L}_{i,j}]_{i,j \in S}$.
\end{definition}
DPPs and $k$-DPPs are closely related but different objects. For starters, $k$-DPPs are always well defined, provided $\bm{L}$ has a nonzero minor of size $k$.

\subsection{Sampling from a DPP and a $k$-DPP}
\label{subsec:sampling_from_a_dpp}
Let $\bm{K}\in\mathbb{R}^{d\times d}$ be a symmetric, positive semi-definite matrix, with eigenvalues in $[0,1]$, so that $\bm{K}$ is the marginal kernel of a DPP on $[d]$. Let us diagonalize it as $\bm{K} = \bm{V}\text{Diag}(\lambda_i)\bm{V}^{\Tran}$. \cite{BeKrPeVi05} established that sampling from the DPP with kernel $\bm K$ can be done by \emph{(i)} sampling independent Bernoullis $B_i, i=1,\dots,d$, with respective parameters $\lambda_i$, \emph{(ii)} forming the submatrix $\bm V_{:,B}$ of $\bm V$ corresponding to columns $i$ such that that $B_i=1$, and \emph{(iii)} sampling from the projection DPP with kernel
$$\bm{K}_\text{proj} = \bm{V}_{:,B}\bm{V}_{:,B}^{\Tran}.$$
The only nontrivial step is sampling from a projection DPP, for which we give pseudocode in Figure~\ref{alg:DPP_SAMPLER}; see \cite[Theorem 7]{BeKrPeVi05} or \cite[Theorem 2.3]{KuTa12} for a proof. For a survey of variants of the algorithm, we also refer to \citep{TrBaAm18} and the documentation of the DPPy toolbox\footnote{\url{http://github.com/guilgautier/DPPy}} \citep{GaBaVa18}. For our purposes, it is enough to remark that general DPPs are mixtures of projection DPPs of different ranks, and that the cardinality of a general DPP is a sum of independent Bernoulli random variables.

\begin{figure}[!ht]
\centerline{
\scalebox{0.9}{
\begin{algorithm}{$\Algo{ProjectionDPP}\big(\bm{K}_\text{proj}=\bm{V}\bm{V}^{\Tran})$}
\Aitem $Y \longleftarrow  \emptyset$
\Aitem $\bm{W} \longleftarrow  \bm{V}$
\Aitem \While $\text{rk}(\bm{W}) > 0 $
\Aitem \mt  Sample $i$ from $\Omega$ with probability $ \propto \|\bm{W}_{i,:}\|_{2}^{2}$ \ar{Chain rule}
\Aitem \mt  $Y \longleftarrow  Y \cup \{i\}$
\Aitem \mt $\bm{V} \longleftarrow \bm{V}_{\perp}$ an orthonormal basis of $\Span(\bm{V} \cap \bm{e}_{i}^{\perp})$
\Aitem \Return $Y$
\end{algorithm}
}
}
\caption{Pseudocode for sampling from a DPP of marginal kernel $\bm{K}$.}
\label{alg:DPP_SAMPLER}
\end{figure}

The next proposition establishes that $k$-DPPs also are mixtures of projection DPPs.
\begin{proposition}(\citet[Section 5.2.2]{KuTa12})
\label{prop:kdpp_mixture_proposition}
Let $Y$ be a random subset of $[d]$ sampled from a $k$-DPP with kernel $\bm{L}$. We further assume that $\bm L$ is symmetric, we denote its rank by $r$ and its diagonalization by $\bm L = \bm{V}\bm{\Lambda}\bm{V}^{\Tran}$. Finally, let $k\leq r$. It holds
\begin{equation}
\label{eq:volume_sampling_as_mixture_equation}
\Prb(Y = S) = \sum\limits_{\substack{T \subset [r]\\|T| = k}} \mu_{T} \left[\frac{1}{k!} \Det\left(\bm{V}_{T,S}\bm{V}^{\Tran}_{T,S}\right)\right]
\end{equation}
where
\begin{equation}
\mu_{T} = \frac{\prod_{i \in T}\lambda_{i}}{\sum\limits_{\substack{U \subset[r]\\ |U| = k}}\prod_{i \in U}\lambda_{i}}.
\label{e:kDPPWeights}
\end{equation}
\end{proposition}
Each mixture component in square brackets in \eqref{eq:volume_sampling_as_mixture_equation} is a projection DPP with cardinality $k$. Sampling a $k$-DPP can thus be done by \emph{(i)} sampling a multinomial distribution with parameters \eqref{e:kDPPWeights}, and \emph{(ii)} sampling from the corresponding projection DPP using the algorithm in Figure~\ref{alg:DPP_SAMPLER}. The main difference between $k$-DPPs and DPPs is that all mixture components in \eqref{eq:volume_sampling_as_mixture_equation} have the same cardinality $k$. In particular, projection DPPs are the only DPPs that are also $k$-DPPs.

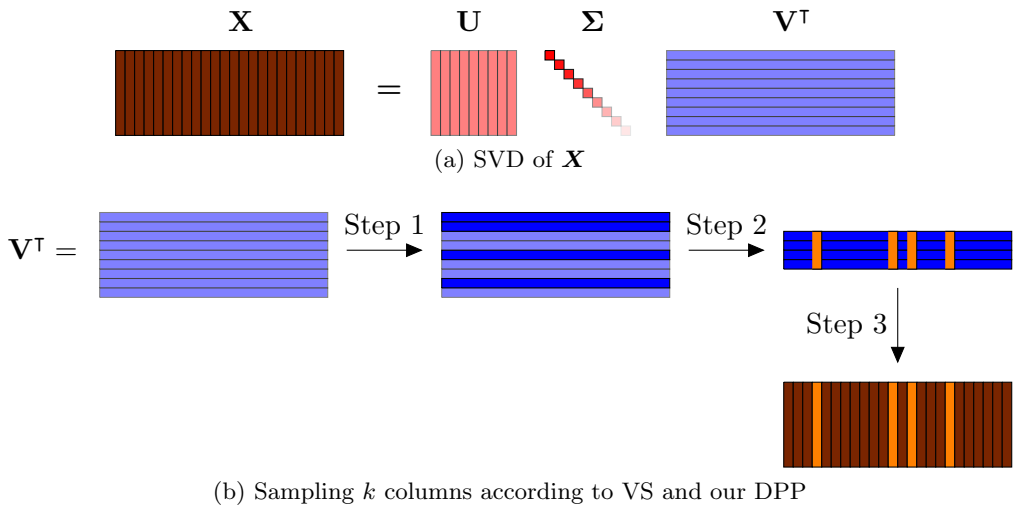
\begin{figure}[!ht]
  \centering
     \subfloat[SVD of $\bm{X}$]{\begin{tikzpicture}[scale = 0.5]


\begin{scope}[xshift=7mm]

\draw [fill=Brown, opacity=1] (-16,4) -- (-15.75,4) -- (-15.75,1.75) -- (-16,1.75) -- (-16,4);
\draw [fill=Brown, opacity=1] (-15.75,4) -- (-15.5,4) -- (-15.5,1.75) -- (-15.75,1.75) -- (-15.75,4);

\draw [fill=Brown, opacity=1] (-15.5,4) -- (-15.25,4) -- (-15.25,1.75) -- (-15.5,1.75) -- (-15.5,4);
\draw [fill=Brown, opacity=1] (-15.25,4) -- (-15,4) -- (-15,1.75) -- (-15.25,1.75) -- (-15.25,4);
\draw [fill=Brown, opacity=1] (-15,4) -- (-14.75,4) -- (-14.75,1.75) -- (-15,1.75) -- (-15,4);
\draw [fill=Brown, opacity=1] (-14.75,4) -- (-14.5,4) -- (-14.5,1.75) -- (-14.75,1.75) -- (-14.75,4);
\draw [fill=Brown, opacity=1] (-14.5,4) -- (-14.25,4) -- (-14.25,1.75) -- (-14.5,1.75) -- (-14.5,4);
\draw [fill=Brown, opacity=1] (-14.25,4) -- (-14,4) -- (-14,1.75) -- (-14.25,1.75) -- (-14.25,4);
\draw [fill=Brown, opacity=1] (-14,4) -- (-13.75,4) -- (-13.75,1.75) -- (-14,1.75) -- (-14,4);
\draw [fill=Brown, opacity=1] (-13.75,4) -- (-13.5,4) -- (-13.5,1.75) -- (-13.75,1.75) -- (-13.75,4);
\draw [fill=Brown, opacity=1] (-13.5,4) -- (-13.25,4) -- (-13.25,1.75) -- (-13.5,1.75) -- (-13.5,4);
\draw [fill=Brown, opacity=1] (-13.25,4) -- (-13,4) -- (-13,1.75) -- (-13.25,1.75) -- (-13.25,4);
\draw [fill=Brown, opacity=1] (-13,4) -- (-12.75,4) -- (-12.75,1.75) -- (-13,1.75) -- (-13,4);

\draw [fill=Brown, opacity=1] (-12.75,4) -- (-12.5,4) -- (-12.5,1.75) -- (-12.75,1.75) -- (-12.75,4);
\draw [fill=Brown, opacity=1] (-12.5,4) -- (-12.25,4) -- (-12.25,1.75) -- (-12.5,1.75) -- (-12.5,4);
\draw [fill=Brown, opacity=1] (-12.25,4) -- (-12,4) -- (-12,1.75) -- (-12.25,1.75) -- (-12.25,4);
\draw [fill=Brown, opacity=1] (-12,4) -- (-11.75,4) -- (-11.75,1.75) -- (-12,1.75) -- (-12,4);

\draw [fill=Brown, opacity=1] (-11.75,4) -- (-11.5,4) -- (-11.5,1.75) -- (-11.75,1.75) -- (-11.75,4);
\draw [fill=Brown, opacity=1] (-11.5,4) -- (-11.25,4) -- (-11.25,1.75) -- (-11.5,1.75) -- (-11.5,4);
\draw [fill=Brown, opacity=1] (-11.25,4) -- (-11,4) -- (-11,1.75) -- (-11.25,1.75) -- (-11.25,4);
\draw [fill=Brown, opacity=1] (-11,4) -- (-10.75,4) -- (-10.75,1.75) -- (-11,1.75) -- (-11,4);
\draw [fill=Brown, opacity=1] (-10.75,4) -- (-10.5,4) -- (-10.5,1.75) -- (-10.75,1.75) -- (-10.75,4);
\draw [fill=Brown, opacity=1] (-10.5,4) -- (-10.25,4) -- (-10.25,1.75) -- (-10.5,1.75) -- (-10.5,4);
\draw [fill=Brown, opacity=1] (-10.25,4) -- (-10,4) -- (-10,1.75) -- (-10.25,1.75) -- (-10.25,4);

\end{scope}

\begin{scope}[xshift=0mm]
 
\draw [fill=red, opacity=0.5] (-7,4) -- (-6.75,4) -- (-6.75,1.75) -- (-7,1.75) -- (-7,4);
\draw [fill=red, opacity=0.5] (-6.75,4) -- (-6.5,4) -- (-6.5,1.75) -- (-6.75,1.75) -- (-6.75,4);
\draw [fill=red, opacity=0.5] (-6.5,4) -- (-6.25,4) -- (-6.25,1.75) -- (-6.5,1.75) -- (-6.5,4);

\draw [fill=red, opacity=0.5] (-6.25,4) -- (-6,4) -- (-6,1.75) -- (-6.25,1.75) -- (-6.25,4);
\draw [fill=red, opacity=0.5] (-6,4) -- (-5.75,4) -- (-5.75,1.75) -- (-6,1.75) -- (-6,4);
\draw [fill=red, opacity=0.5] (-5.75,4) -- (-5.5,4) -- (-5.5,1.75) -- (-5.75,1.75) -- (-5.75,4);
\draw [fill=red, opacity=0.5] (-5.5,4) -- (-5.25,4) -- (-5.25,1.75) -- (-5.5,1.75) -- (-5.5,4);

\draw [fill=red, opacity=0.5] (-5.25,4) -- (-5,4) -- (-5,1.75) -- (-5.25,1.75) -- (-5.25,4);
\draw [fill=red, opacity=0.5] (-5,4) -- (-4.75,4) -- (-4.75,1.75) -- (-5,1.75) -- (-5,4);

\end{scope}

\begin{scope}[xshift=-5mm]

\draw [fill=red, opacity=1] (-3.5,4) -- (-3.25,4) -- (-3.25,3.75) -- (-3.5,3.75) -- (-3.5,4);
\draw [fill=red, opacity=0.95] (-3.25,3.75) -- (-3,3.75) -- (-3,3.5) -- (-3.25,3.5) -- (-3.25,3.75);
\draw [fill=red, opacity=0.9] (-3,3.5) -- (-2.75,3.5) -- (-2.75,3.25) -- (-3,3.25) -- (-3,3.5);
\draw [fill=red, opacity=0.8] (-2.75,3.25) -- (-2.5,3.25) -- (-2.5,3) -- (-2.75,3) -- (-2.75,3.25);

\draw [fill=red, opacity=0.65] (-2.5,3) -- (-2.25,3) -- (-2.25,2.75) -- (-2.5,2.75) -- (-2.5,3);
\draw [fill=red, opacity=0.45] (-2.25,2.75) -- (-2,2.75) -- (-2,2.5) -- (-2.25,2.5) -- (-2.25,2.75);
\draw [fill=red, opacity=0.3] (-2,2.5) -- (-1.75,2.5) -- (-1.75,2.25) -- (-2,2.25) -- (-2,2.5);
\draw [fill=red, opacity=0.2] (-1.75,2.25) -- (-1.5,2.25) -- (-1.5,2) -- (-1.75,2) -- (-1.75,2.25);

\draw [fill=red, opacity=0.1] (-1.5,2) -- (-1.25,2) -- (-1.25,1.75) -- (-1.5,1.75) -- (-1.5,2);
\end{scope}

\begin{scope}[xshift=-18mm]

\draw [fill=blue, opacity=0.5] (7,3.75) -- (1,3.75) -- (1,4) -- (7,4) -- (7,3.75);
\draw [fill=blue, opacity=0.5] (7,3.5) -- (1,3.5) -- (1,3.75) -- (7,3.75) -- (7,3.5);
\draw [fill=blue, opacity=0.5] (7,3.25) -- (1,3.25) -- (1,3.5) -- (7,3.5) -- (7,3.25);
\draw [fill=blue, opacity=0.5] (7,3) -- (1,3) -- (1,3.25) -- (7,3.25) -- (7,3);
\draw [fill=blue, opacity=0.5] (7,2.75) -- (1,2.75) -- (1,3) -- (7,3) -- (7,2.75); 
\draw [fill=blue, opacity=0.5] (7,2.5) -- (1,2.5) -- (1,2.75) -- (7,2.75) -- (7,2.5);
\draw [fill=blue, opacity=0.5] (7,2.25) -- (1,2.25) -- (1,2.5) -- (7,2.5) -- (7,2.25);
\draw [fill=blue, opacity=0.5] (7,2) -- (1,2) -- (1,2.25) -- (7,2.25) -- (7,2);
\draw [fill=blue, opacity=0.5] (7,1.75) -- (1,1.75) -- (1,2) -- (7,2) -- (7,1.75);

\end{scope}


\draw  (-12,4.25)  node [above] {$\bm{\mathrm{X}}$} ;
\draw  (-8.1,2.5)  node [above] {$\bm{=}$} ;

\draw  (-6,4.25)  node [above] {$\bm{\mathrm{U}}$} ;

\draw  (-2.75,4.25)  node [above] {$\bm{\mathrm{\Sigma}}$} ;

\draw  (2.5,4.25)  node [above] {$\bm{\mathrm{V}}^{\Tran}$} ;



\end{tikzpicture}}\\
     \subfloat[Sampling $k$ columns according to VS and our DPP]{\begin{tikzpicture}[scale = 0.5]


\draw [fill=blue, opacity=0.5] (-16,3.75) -- (-10,3.75) -- (-10,4) -- (-16,4) -- (-16,3.75);
\draw [fill=blue, opacity=0.5] (-16,3.5) -- (-10,3.5) -- (-10,3.75) -- (-16,3.75) -- (-16,3.5);
\draw [fill=blue, opacity=0.5] (-16,3.25) -- (-10,3.25) -- (-10,3.5) -- (-16,3.5) -- (-16,3.25);
\draw [fill=blue, opacity=0.5] (-16,3) -- (-10,3) -- (-10,3.25) -- (-16,3.25) -- (-16,3);
\draw [fill=blue, opacity=0.5] (-16,2.75) -- (-10,2.75) -- (-10,3) -- (-16,3) -- (-16,2.75);
\draw [fill=blue, opacity=0.5] (-16,2.5) -- (-10,2.5) -- (-10,2.75) -- (-16,2.75) -- (-16,2.5);
\draw [fill=blue, opacity=0.5] (-16,2.25) -- (-10,2.25) -- (-10,2.5) -- (-16,2.5) -- (-16,2.25);
\draw [fill=blue, opacity=0.5] (-16,2) -- (-10,2) -- (-10,2.25) -- (-16,2.25) -- (-16,2);
\draw [fill=blue, opacity=0.5] (-16,1.75) -- (-10,1.75) -- (-10,2) -- (-16,2) -- (-16,1.75);

\begin{scope}[xshift=-30mm]

\draw [fill=blue, opacity=1] (-4,3.75) -- (2,3.75) -- (2,4) -- (-4,4) -- (-4,3.75);
\draw [fill=blue, opacity=1] (-4,3.5) -- (2,3.5) -- (2,3.75) -- (-4,3.75) -- (-4,3.5);
\draw [fill=blue, opacity=0.5] (-4,3.25) -- (2,3.25) -- (2,3.5) -- (-4,3.5) -- (-4,3.25);
\draw [fill=blue, opacity=0.5] (-4,3) -- (2,3) -- (2,3.25) -- (-4,3.25) -- (-4,3);
\draw [fill=blue, opacity=1] (-4,2.75) -- (2,2.75) -- (2,3) -- (-4,3) -- (-4,2.75);
\draw [fill=blue, opacity=0.5] (-4,2.5) -- (2,2.5) -- (2,2.75) -- (-4,2.75) -- (-4,2.5);
\draw [fill=blue, opacity=0.5] (-4,2.25) -- (2,2.25) -- (2,2.5) -- (-4,2.5) -- (-4,2.25);
\draw [fill=blue, opacity=1] (-4,2) -- (2,2) -- (2,2.25) -- (-4,2.25) -- (-4,2);
\draw [fill=blue, opacity=0.5] (-4,1.75) -- (2,1.75) -- (2,2) -- (-4,2) -- (-4,1.75);

\end{scope}

\begin{scope}[xshift=-65mm]

\draw [fill=blue, opacity=1] (8.5,3.25) -- (14.5,3.25) -- (14.5,3.5) -- (8.5,3.5) -- (8.5,3.25);
\draw [fill=blue, opacity=1] (8.5,3) -- (14.5,3) -- (14.5,3.25) -- (8.5,3.25) -- (8.5,3);
\draw [fill=blue, opacity=1] (8.5,2.75) -- (14.5,2.75) -- (14.5,3) -- (8.5,3) -- (8.5,2.75);
\draw [fill=blue, opacity=1] (8.5,2.5) -- (14.5,2.5) -- (14.5,2.75) -- (8.5,2.75) -- (8.5,2.5);

\draw [fill=orange, opacity=1] (9.25,2.5) -- (9.25,2.5) -- (9.5,2.5) -- (9.5,3.5) -- (9.25,3.5);
\draw [fill=orange, opacity=1] (11.25,2.5) -- (11.25,2.5) -- (11.5,2.5) -- (11.5,3.5) -- (11.25,3.5);
\draw [fill=orange, opacity=1] (11.75,2.5) -- (11.75,2.5) -- (12,2.5) -- (12,3.5) -- (11.75,3.5);
\draw [fill=orange, opacity=1] (12.75,2.5) -- (12.75,2.5) -- (13,2.5) -- (13,3.5) -- (12.75,3.5);

\end{scope}

\begin{scope}[yshift=-45mm]
\begin{scope}[xshift=180mm]

\draw [fill=Brown, opacity=1] (-16,4) -- (-15.75,4) -- (-15.75,1.75) -- (-16,1.75) -- (-16,4);
\draw [fill=Brown, opacity=1] (-15.75,4) -- (-15.5,4) -- (-15.5,1.75) -- (-15.75,1.75) -- (-15.75,4);

\draw [fill=Brown, opacity=1] (-15.5,4) -- (-15.25,4) -- (-15.25,1.75) -- (-15.5,1.75) -- (-15.5,4);
\draw [fill=orange, opacity=1] (-15.25,4) -- (-15,4) -- (-15,1.75) -- (-15.25,1.75) -- (-15.25,4);
\draw [fill=Brown, opacity=1] (-15,4) -- (-14.75,4) -- (-14.75,1.75) -- (-15,1.75) -- (-15,4);
\draw [fill=Brown, opacity=1] (-14.75,4) -- (-14.5,4) -- (-14.5,1.75) -- (-14.75,1.75) -- (-14.75,4);
\draw [fill=Brown, opacity=1] (-14.5,4) -- (-14.25,4) -- (-14.25,1.75) -- (-14.5,1.75) -- (-14.5,4);
\draw [fill=Brown, opacity=1] (-14.25,4) -- (-14,4) -- (-14,1.75) -- (-14.25,1.75) -- (-14.25,4);
\draw [fill=Brown, opacity=1] (-14,4) -- (-13.75,4) -- (-13.75,1.75) -- (-14,1.75) -- (-14,4);
\draw [fill=Brown, opacity=1] (-13.75,4) -- (-13.5,4) -- (-13.5,1.75) -- (-13.75,1.75) -- (-13.75,4);
\draw [fill=Brown, opacity=1] (-13.5,4) -- (-13.25,4) -- (-13.25,1.75) -- (-13.5,1.75) -- (-13.5,4);
\draw [fill=orange, opacity=1] (-13.25,4) -- (-13,4) -- (-13,1.75) -- (-13.25,1.75) -- (-13.25,4);
\draw [fill=Brown, opacity=1] (-13,4) -- (-12.75,4) -- (-12.75,1.75) -- (-13,1.75) -- (-13,4);

\draw [fill=orange, opacity=1] (-12.75,4) -- (-12.5,4) -- (-12.5,1.75) -- (-12.75,1.75) -- (-12.75,4);
\draw [fill=Brown, opacity=1] (-12.5,4) -- (-12.25,4) -- (-12.25,1.75) -- (-12.5,1.75) -- (-12.5,4);
\draw [fill=Brown, opacity=1] (-12.25,4) -- (-12,4) -- (-12,1.75) -- (-12.25,1.75) -- (-12.25,4);
\draw [fill=Brown, opacity=1] (-12,4) -- (-11.75,4) -- (-11.75,1.75) -- (-12,1.75) -- (-12,4);

\draw [fill=orange, opacity=1] (-11.75,4) -- (-11.5,4) -- (-11.5,1.75) -- (-11.75,1.75) -- (-11.75,4);
\draw [fill=Brown, opacity=1] (-11.5,4) -- (-11.25,4) -- (-11.25,1.75) -- (-11.5,1.75) -- (-11.5,4);
\draw [fill=Brown, opacity=1] (-11.25,4) -- (-11,4) -- (-11,1.75) -- (-11.25,1.75) -- (-11.25,4);
\draw [fill=Brown, opacity=1] (-11,4) -- (-10.75,4) -- (-10.75,1.75) -- (-11,1.75) -- (-11,4);
\draw [fill=Brown, opacity=1] (-10.75,4) -- (-10.5,4) -- (-10.5,1.75) -- (-10.75,1.75) -- (-10.75,4);
\draw [fill=Brown, opacity=1] (-10.5,4) -- (-10.25,4) -- (-10.25,1.75) -- (-10.5,1.75) -- (-10.5,4);
\draw [fill=Brown, opacity=1] (-10.25,4) -- (-10,4) -- (-10,1.75) -- (-10.25,1.75) -- (-10.25,4);

\end{scope}
\end{scope}

\draw [->] (-9.5,3) --  (-8.5,3)  node [above] { Step 1 } --  (-7.5,3);

\draw [->] (-0.5,3) --  (0.5,3)  node [above] { Step 2 } --  (1.5,3);

\draw [->] (5,2) --  (5,1)  node [left] { Step 3 } --  (5,0);

\draw  (-17.5,2.5)  node [above ] {$\bm{\mathrm{V}}^{\Tran}=$} ;


\end{tikzpicture}}
    \caption{A graphical depiction of the sampling algorithms for volume sampling (VS) and the DPP with marginal kernel $\bm{V}^{}_{k}\bm{V}_{k}^{\Tran}$. (a) Both algorithms start with an SVD. (b) In Step 1, VS randomly selects $k$ rows of $\bm V^{\Tran}$, while the DPP always picks the first $k$ rows. Step 2 is the same for both algorithms: jointly sample $k$ columns of the subsampled $\bm V^{\Tran}$, proportionally to their squared volume. Step 3 is simply the extraction of the corresponding columns of $\bm X$.
    \label{f:sampling}
    }
\end{figure}


A fundamental example of $k$-DPP is volume sampling, as defined in Section \ref{subsec:volume_sampling}. Its kernel is the Gram matrix of the data $\bm{L} = \bm{X}^{\Tran}\bm{X}$. In general, $\bm{L}$ is not an orthogonal projection, so that volume sampling is not a DPP.

\subsection{Motivations for column subset selection using projection DPPs}
volume sampling has been successfully used for column subset selection, see Section~\ref{subsec:volume_sampling}. Our motivation to investigate projection DPPs instead of volume sampling is twofold.

Following \eqref{eq:volume_sampling_as_mixture_equation}, volume sampling can be seen as a mixture of projection DPPs indexed by $T\subset [d], \vert T\vert=k$, with marginal kernels $\bm{K}_{T} = \bm{V}^{}_{:,T}\bm{V}^{\Tran}_{:,T}$ and mixture weights $\mu_{T} \propto \prod_{i \in T} \sigma_{i}^{2}$. The component with the highest weight thus corresponds to the $k$ largest singular values, that is, the projection DPP with marginal kernel
$\bm K:=\bm{V}^{}_{k}\bm{V}_{k}^{\Tran}$. This paper is about column subset selection using precisely this DPP. Alternately, we could motivate the study of this DPP by remarking that its marginals $\Prb({i}\subset Y)$ are the $k$-leverage scores introduced in Section~\ref{subsec:k-lvs_sampling}. Since $\bm K$ is symmetric, this DPP can be seen as a repulsive generalization of leverage score sampling.

Finally, we recap the difference between volume sampling and the DPP with kernel $\bm K$ with a graphical depiction in Figure~\ref{f:sampling} of the two procedures to sample from them that we introduced in Section~\ref{subsec:sampling_from_a_dpp}. Figure~\ref{f:sampling} is another illustration of the decomposition of volume sampling as a mixture of projection DPPs.

\section{Main Results}
\label{sec:main_results}

In this section, we prove bounds for $\EX_{\DPP} \| \bm{X} - \Pi_{S}^{\nu}\bm{X} \|_{\nu}$ under the projection DPP of marginal kernel $\bm{K} = \bm{V}^{}_{k}\bm{V}^{\Tran}_{k}$ presented in Section~\ref{s:dppsection}. Throughout, we compare our bounds to the state-of-the-art bounds of volume sampling obtained by \cite{DRVW06}; see Theorem~\ref{thrm:volume_sampling_theorem} and Section~\ref{subsec:volume_sampling}. For clarity, we defer the proofs of our results from this section to Appendix~\ref{app:proofs}.

%

\subsection{Multiplicative bounds in spectral and Frobenius norm}
\label{sec:new_results_randomized}
Let $S$ be a random subset of $k$ columns of $\bm{X}$ chosen with probability:
\begin{equation}
	\Prb_{\DPP}(S) = \Det(\bm{V}_{S,[k]})^{2}.
\end{equation}
First, without any further assumption, we have the following result.
\begin{proposition}
    \label{projection_dpp_theorem}
    Under the projection DPP of marginal kernel $\bm{V}^{}_{k}\bm{V}^{\Tran}_{k}$, it holds
    \begin{equation}
    	\label{eq:projection_dpp_theorem}
    	\EX_{\DPP} \| \bm{X} - \Pi_{S}^{\nu}\bm{X} \|_{\nu}^{2} \leq k(d+1-k)\| \bm{X} - \Pi_{k}\bm{X} \|_{\nu}^{2}, \quad \nu\in\{2,\Fr\}.
    \end{equation}
\end{proposition}
For the spectral norm, the bound is practically the same as that of volume sampling \eqref{e:vsBound2}. However, our bound for the Frobenius norm is worse than \eqref{e:vsBoundFr} by a factor $(d-k)$. In the rest of this section, we sharpen our bounds by taking into account the sparsity level of the $k$-leverage scores and the decay of singular values.

In terms of sparsity, we first replace the dimension $d$ in \eqref{eq:projection_dpp_theorem} by the number $p\in[d]$ of nonzero $k$-leverage scores
\begin{equation}
  p = \left| \{i \in [d], \bm{V}_{i,[k]} \neq \bm{0}\}\right|.
  \label{e:defp}
\end{equation}
To quantify the decay of the singular values, we define the flatness parameter
\begin{equation}
  \beta = \bm{\sigma}_{k+1}^{2} \left(\frac{1}{d-k} \sum\limits_{j \geq k+1} \bm{\sigma}_{j}^{2}\right)^{-1}.
  \label{e:defbeta}
\end{equation}
In words, $\beta\in[1,d-k]$ measures the flatness of the spectrum of $\bm{X}$ above the cut-off at $k$. Indeed, \eqref{e:defbeta} is the ratio of the largest term in a sum to that sum. The closer $\beta$ is to $1$, the more similar the terms in the sum in the denominator of \eqref{e:defbeta}. At the extreme, $\beta=d-k$ when $\sigma^2_{k+1}>0$ while $\sigma_j^2=0,$ $\forall j\geq k+2$.

\begin{proposition}\label{hypo_one_two_proposition}
Under the projection DPP of marginal kernel $\bm{V}^{}_{k}\bm{V}^{\Tran}_{k}$, it holds
\begin{equation}\label{eq:spectral_bound_dpp}
    \EX_{\DPP} \| \bm{X} - \Pi_{S}^{2}\bm{X} \|_{2}^{2} \leq k(p-k) \| \bm{X}- \Pi_{k}\bm{X}\|_{2}^{2}
\end{equation}
and
\begin{equation}\label{eq:frobenius_bound_dpp}
    \EX_{\DPP} \| \bm{X} - \Pi_{S}^{\Fr}\bm{X} \|_{\Fr}^{2} \leq  \left(1 + \beta\frac{p-k}{d-k}k \right) \| \bm{X}- \Pi_{k}\bm{X}\|_{\Fr}^{2}.
\end{equation}

%
\end{proposition}
The bound in \eqref{eq:spectral_bound_dpp} compares favorably with volume sampling \eqref{e:vsBound2} since the dimension $d$ has been replaced by the sparsity level $p$. For $\beta$ close to $1$, the bound in \eqref{eq:frobenius_bound_dpp} is better than the bound \eqref{e:vsBoundFr} of volume sampling since $(p-k)/(d-k) \leq 1$. Again, the sparser the $k$-leverage scores, the smaller the bounds.
\rb{Décommenter ici les Markov}

Now, one could argue that, in practice, sparsity is never exact: it can well be that $p = d$ while there still are a lot of small $k$-leverage scores. We will demonstrate in Section~\ref{s:numexpesection} that the DPP still performs better than volume sampling in this setting, which Proposition~\ref{hypo_one_two_proposition} doesn't reflect. We introduce two ideas to further tighten the bounds of Proposition~\ref{hypo_one_two_proposition}. First, we define an effective sparsity level in the vein of \cite{PaKyBo14}, see Section~\ref{subsec:k-lvs_sampling}. Second, we condition the DPP on a favourable event with controlled probability.
\begin{theorem}\label{prop:p_eff_proposition}
    Let $\pi$ be a permutation of $[d]$ such that leverage scores are reordered
    \begin{equation}
        \ell_{\pi_{1}}^{k}\geq \ell_{\pi_{2}}^{k} \geq ... \geq \ell_{\pi_{d}}^{k}.
    \end{equation}
    For $\delta \in [d]$, let $T_{\delta} = [\pi_{\delta},\dots,\pi_{d}]$. Let $\theta > 1$ and
    \begin{equation}\label{eq:leverage_score_decreasing_hypo}
    p_{\eff}(\theta) = \min\left\{q\in[d] ~\bigg/~ \sum\limits_{i \leq q} \ell_{\pi_{i}}^{k} \geq k -1+\frac{1}{\theta}\right\}.
    \end{equation}
    Finally, let $\mathcal{A}_\theta$ be the event $\{S \cap T_{p_{\eff}(\theta)} = \emptyset\}$. Then, on the one hand,
    \begin{equation}\label{eq:rejection_probability}
        \Prb_{\DPP}\left( \mathcal{A}_\theta\right) \geq \frac{1}{\theta},
    \end{equation}
    and, on the other hand,
    \begin{equation}
    \EX_{\DPP} \left[ \| \bm{X} - \Pi_{S}^{2}\bm{X} \|_{2}^{2} \, \big| \, \mathcal{A}_\theta \right] \leq \left(p_{\eff}(\theta)-k+1)(k-1+\theta \right)\| \bm{X} - \Pi_{k}\bm{X} \|_{2}^{2}
    \end{equation}
    and
    \begin{equation}
    	\label{eq:boundFrobenius_assumpt2}
    	\EX_{\DPP} \left[ \| \bm{X} - \Pi_{S}^{\Fr}\bm{X} \|_{\Fr}^{2} \, \big| \, \mathcal{A}_\theta\right] \leq \left(1+				\beta\frac{(p_{\eff}(\theta)+1-k)}{d-k}(k-1+\theta) \right)\| \bm{X} - \Pi_{k}\bm{X} \|_{\Fr}^{2}.
    \end{equation}
\end{theorem}
%
In Theorem~\ref{prop:p_eff_proposition}, the effective sparsity level $p_{\eff}(\theta)$ replaces the sparsity level $p$ of Proposition~\ref{hypo_one_two_proposition}. The key is to condition on $S$ not containing any index of column with too small a $k$-leverage score, that is, the event $\mathcal{A}_\theta$. In practice, this is achieved by rejection sampling: we repeatedly and independently sample $S \sim \DPP(\bm{K})$ until $S \cap T_{p_{\eff}}(\theta) = \emptyset$.

The caveat of any rejection sampling procedure is a potentially large number of samples required before acceptance. But in the present case, Equation~\eqref{eq:rejection_probability} guarantees that the expectation of that number of samples is less than $\theta$. The free parameter $\theta$ thus interestingly controls both the ``energy" threshold in \eqref{eq:leverage_score_decreasing_hypo}, and the complexity of the rejection sampling. The approximation bounds suggest picking $\theta$ close to $1$, which implies a compromise with the value of $p_{\eff}(\theta)$ that should not be too large either. We have empirically observed that the performance of the DPP is relatively insensitive to the choice of $\theta$.

\subsection{Bounds for the excess risk in sketched linear regression}
\label{sec:bounds_for_regression_under_dpp}
 In Section~\ref{s:related_work_sparse_regression}, we surveyed bounds on the excess risk of ordinary least squares estimators that relied on a subsample of the columns of $\bm{X}$.
Importantly, the generic bound \eqref{eq:prediction_bound_CSS} of \cite{LiHa18} has a bias term that depends on the maximum squared tangent of the principal angles between $\Span(\bm{S})$ and $\Span(\bm{V}_k)$. When $|S| = k$, this quantity is hard to control without making strong assumptions on the matrix $\bm{V}_{k}$.
But it turns out that, in expectation under the same DPP as in Section~\ref{sec:new_results_randomized}, this bias term drastically simplifies.

\begin{proposition}\label{prop:sparsity_prediction_under_dpp}
    We use the notation of Section~\ref{s:related_work_sparse_regression}. Under the projection DPP with marginal kernel $\bm{V}^{}_{k}\bm{V}^{\Tran}_{k}$, it holds
    \begin{equation}\label{eq:prediction_bound_dpp}
        \EX_{\DPP} \big[ \mathcal{E}(\bm{w}_{S}) \big] \leq \big(1+k(p-k)\big)\frac{\|\bm{w}^{*}\|^{2}\sigma_{k+1}^{2}}{N} + \frac{vk}{N}.
    \end{equation}
\end{proposition}
The sparsity level $p$ appears again in the bound \eqref{eq:prediction_bound_dpp}:
The sparser the $k$-leverage scores distribution, the smaller the bias term. The bound \eqref{eq:prediction_bound_dpp} only features an additional $(1+k(p-k))$ factor in the bias term, compared to the bound obtained by \cite{LiHa18} for PCR, see Proposition~\ref{prop:pcr_bound}. Loosely speaking, this factor is to be seen as the price we accept to pay in order to get more interpretable features  than principal components in the linear regression problem. Finally, a natural question is to investigate the choice of $k$ to minimize the bound in \eqref{eq:prediction_bound_dpp}, but this is out of the scope of this paper.

As in Theorem~\ref{prop:p_eff_proposition}, for practical purposes, it can be desirable to bypass the need for the exact sparsity level $p$ in Proposition~\ref{prop:sparsity_prediction_under_dpp}. We give a bound that replaces $p$ with the effective sparsity level $p_{\eff}(\theta)$ introduced in \eqref{eq:leverage_score_decreasing_hypo}.

\begin{theorem}\label{prop:relaxed_sparsity_prediction_under_dpp}
    Using the notation of Section~\ref{s:related_work_sparse_regression} for linear regression, and of Theorem~\ref{prop:p_eff_proposition} for leverage scores and their indices, it holds
    \begin{equation}\label{eq:prediction_bound_dpp2}
        \EX_{\DPP} \big[ \mathcal{E}(\hat{\bm{w}}_{S}) \, \big| \, \mathcal{A}_\theta \big] \leq  \big[1+\big(k-1+\theta\big)\big(p_{\eff}(\theta)-k+ 1\big)\big]\frac{\|\bm{w}^{*}\|^{2}\sigma_{k+1}^{2}}{N} + \frac{vk}{N}.
    \end{equation}
\end{theorem}
In practice, the same rejection sampling routine as in Theorem~\ref{prop:p_eff_proposition} can be used to sample conditionally on $\mathcal{A}_\theta$. Finally, to the best of our knowledge, bounding the excess risk in linear regression has not been investigated under volume sampling.\\

In summary, we have obtained two sets of results. We have proven a set of multiplicative bounds in spectral and Frobenius norm for $\EX_{\DPP} \| \bm{X} - \Pi_{S}^{\nu}\bm{X} \|_{\nu}$, $\nu\in\{2,\Fr\}$, under the projection DPP of marginal kernel $\bm{K} = \bm{V}^{}_{k}\bm{V}^{\Tran}_{k}$, see Propositions \ref{projection_dpp_theorem} \& \ref{hypo_one_two_proposition} and Theorem~\ref{prop:p_eff_proposition}. As far as the linear regression problem is concerned, we have proven bounds for the excess risk in sketched linear regression, see Proposition~\ref{prop:sparsity_prediction_under_dpp} and Theorem~\ref{prop:relaxed_sparsity_prediction_under_dpp}.


\section{Numerical experiments}
\label{s:numexpesection}
In this section, we empirically compare our algorithm to the state of the art in column subset selection. In Section~\ref{s:toyDatasets}, the projection DPP with kernel $\bm{K}=\bm{V}_k\bm{V}_k^{\Tran}$ and volume sampling are compared on toy datasets. In Section~\ref{s:realDatasets}, several column subset selection algorithms are compared to the projection DPP on four real datasets from genomics and text processing. In particular, the numerical simulations demonstrate the favorable influence of the sparsity of the $k$-leverage scores on the performance of our algorithm both on toy datasets and real datasets. Finally, we packaged all CSS algorithms in this section in a Python toolbox\footnote{\url{http://github.com/AyoubBelhadji/CSSPy}}.

\subsection{Toy datasets}
\label{s:toyDatasets}
This section is devoted to comparing the expected approximation error $\mathbb{E} \|\bm{X}- \Pi_{S}^{\Fr} \bm{X}\|_{\Fr}^{2}$ for the projection DPP and volume sampling. We focus on the Frobenius norm to avoid effects to due different choices of the projection $\Pi_{S}^{\Fr}$, see \eqref{eq:frob_as_estimator_of_spe}.

In order to be able to evaluate the expected errors \emph{exactly}, we generate matrices of low dimension ($d = 20$) so that the subsets of $[d]$ can be exhaustively enumerated. Furthermore, to investigate the role of leverage scores and singular values on the performance of CSS algorithms, we need to generate datasets $\bm{X}$ with prescribed spectra and $k$-leverage scores.

\subsubsection{Generating toy datasets}
Recall that the SVD of $\bm{X}\in\mathbb{R}^{N\times d}$ reads $\bm{X} = \bm{U}\bm{\Sigma}\bm{V}^{\Tran}$, where $\bm{\Sigma}$ is a diagonal matrix and $\bm{U}$ and $\bm{V}$ are orthogonal matrices. To sample a matrix $\bm{X}$, $\bm{U}$ is first drawn from the Haar measure of $\mathcal{O}_{N}(\mathbb{R})$, then $\bm{\Sigma}$ is chosen among a few deterministic diagonal matrices that illustrate various spectral properties. Sampling the matrix $\bm{V}$ is trickier. The first $k$ columns of $\bm{V}$ are structured as follows: the number of non vanishing rows of $\bm{V}_{k}$ is equal to $p$ and the norms of the nonvanishing rows are prescribed by a vector $\bm{\ell}$. By considering the matrix $\bm{K} = \bm{V}_{k}^{}\bm{V}_{k}^{\Tran}$, generating $\bm{V}$ boils down to the simulation of an Hermitian matrix with prescribed diagonal and spectrum (in this particular case the spectrum is included in $\{0,1\}$). For this reason, we propose an algorithm that takes as input a leverage scores profile $\bm{\ell}$ and a spectrum $\bm{\sigma}^2$, and outputs a corresponding random orthogonal matrix $\bm{X}$; see Appendix~\ref{app:framebuilding}. This algorithm is a randomization\footnote{\url{http://github.com/AyoubBelhadji/FrameBuilder}} of the algorithm proposed by \cite*{FMPS11}.  Finally, the matrix $\bm{V}_{k} \in \mathbb{R}^{d \times k}$ is completed by applying the Gram-Schmidt procedure to $d-k$ additional i.i.d. unit Gaussian vectors, resulting in a matrix $\bm{V} \in \mathbb{R}^{d \times d}$. Figure~\ref{f:Matrix_Builder} summarizes the algorithm proposed to generate matrices $\bm{X}$ with a $k$-leverage scores profile $\bm{\ell}$ and a sparsity level $p$.

\begin{figure}
\centering
\begin{algorithm}{$\Algo{MatrixGenerator}\big(\bm{\ell},\bm{\Sigma})$}
    \Aitem Sample $\bm{U}$ from the Haar measure $\mathbb{O}_{N}(\mathbb{R})$.
    \Aitem Pick $\bm{\Sigma}$ a diagonal matrix.
    \Aitem Pick $p \in [k+1:d]$.
    \Aitem Pick a $k$-leverage-scores profile $\bm{\ell} \in \mathbb{R}_{+}^{d}$ with a sparsity level $p$.
    \Aitem Generate a matrix $\bm{V}_{k}$ with the $k$-leverage-scores profile $\bm{\ell}$.
    \Aitem Extend the matrix $\bm{V}_{k}$ to an orthogonal matrix $\bm{V}$.
    \Aitem \Return $\bm{X} \longleftarrow \bm{U}\bm{\Sigma}\bm{V}^{\Tran}$
\end{algorithm}
%

\caption{The pseudocode of the algorithm generating a matrix $\bm{X}$ with prescribed profile of $k$-leverage scores.}
\label{f:Matrix_Builder}
\end{figure}

\subsubsection{volume sampling vs projection DPP}
This section sums up the results of numerical simulations on toy datasets. The number of observations is fixed to $N =100 $, the dimension to $d = 20$, and the number of selected columns to $k \in \{3,5\}$. Singular values of are chosen from the following profiles: a spectrum with a cutoff called the projection spectrum,
$$
	\bm{\Sigma}_{k=3,\text{proj}} = 100 \sum\limits_{i = 1}^{3} \bm{e}_{i}\bm{e}_{i}^{\Tran} + 0.1 \sum\limits_{i = 4}^{20} \bm{e}_{i}\bm{e}_{i}^{\Tran}, $$
$$
	\bm{\Sigma}_{k=5,\text{proj}} = 100 \sum\limits_{i = 1}^{5} \bm{e}_{i}\bm{e}_{i}^{\Tran} + 0.1 \sum\limits_{i = 6}^{20} \bm{e}_{i}\bm{e}_{i}^{\Tran}. $$
and a smooth spectrum
$$
	\bm{\Sigma}_{k=3,\text{smooth}} = 100\bm{e}_{1}\bm{e}_{1}^{\Tran} + 10\bm{e}_{2}\bm{e}_{2}^{\Tran} + \bm{e}_{3}\bm{e}_{3}^{\Tran} + 0.1 \sum\limits_{i = 4}^{20} \bm{e}_{i}\bm{e}_{i}^{\Tran},$$
$$
	\bm{\Sigma}_{k=5,\text{smooth}} = 10000\bm{e}_{1}\bm{e}_{1}^{\Tran} + 1000\bm{e}_{2}\bm{e}_{2}^{\Tran} + 100\bm{e}_{3}\bm{e}_{3}^{\Tran} +  10\bm{e}_{4}\bm{e}_{4}^{\Tran} +  \bm{e}_{5}\bm{e}_{5}^{\Tran} + 0.1 \sum\limits_{i = 6}^{20} \bm{e}_{i}\bm{e}_{i}^{\Tran}.
	$$
Note that all profiles satisfy $\beta=1$; see \eqref{e:defbeta}. \rb{We should justify such a favorable setting} In each experiment, for each spectrum, we sample $200$ independent leverage scores profiles that satisfy the sparsity constraints from a Dirichlet distribution with concentration parameter $1$ and equal means. For each leverage scores profile, we sample a matrix $\bm{X}$ from the algorithm in Appendix~\ref{app:framebuilding}.

\begin{figure}
    \centering
\setcounter{subfigure}{0}%
\subfloat[$\Sigma_{3,\text{proj}}$, $k=3$]{\includegraphics[width= 0.5\textwidth]{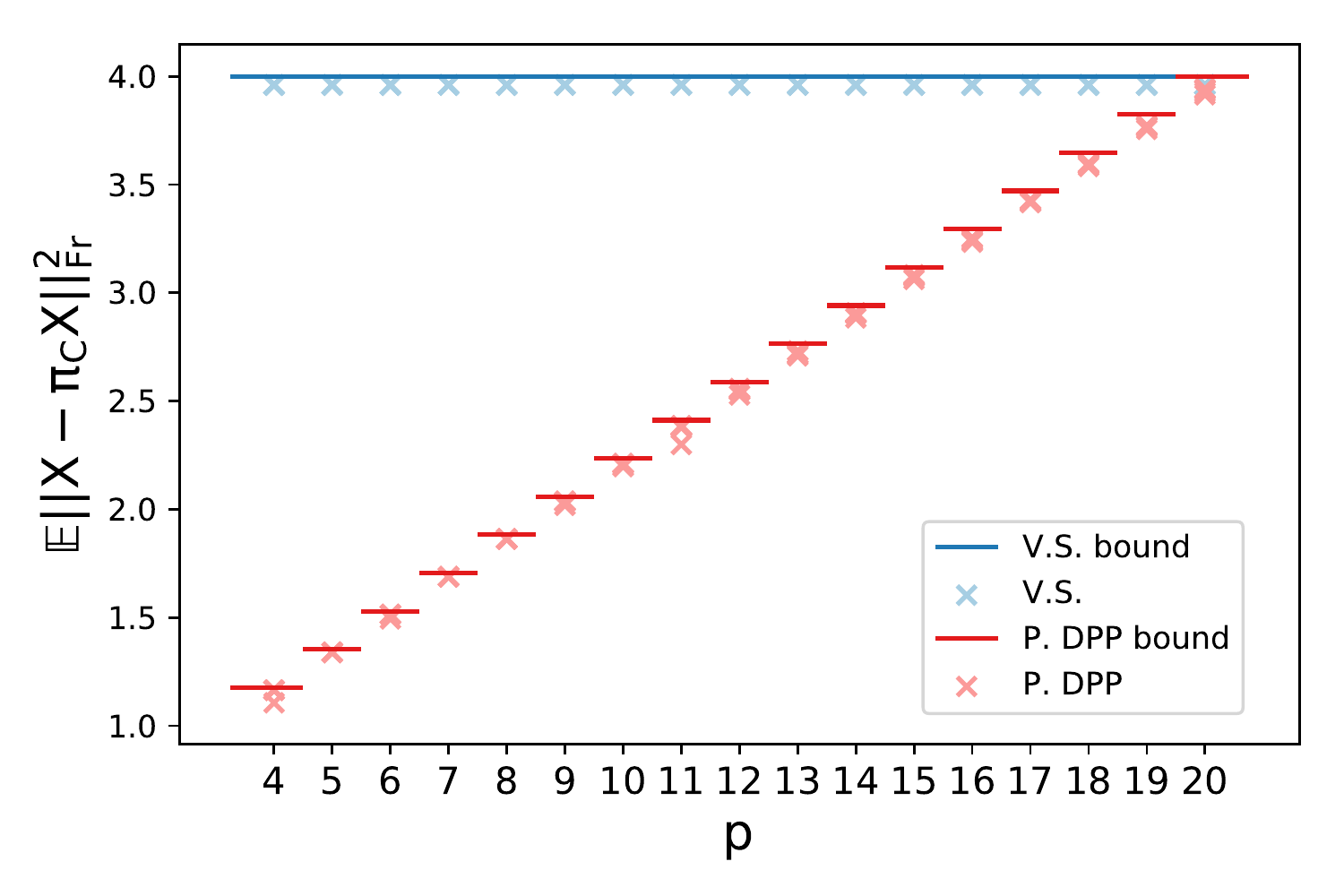}}
\setcounter{subfigure}{3}%
\subfloat[$\Sigma_{5,\text{proj}}$, $k=5$]{\includegraphics[width= 0.5\textwidth]{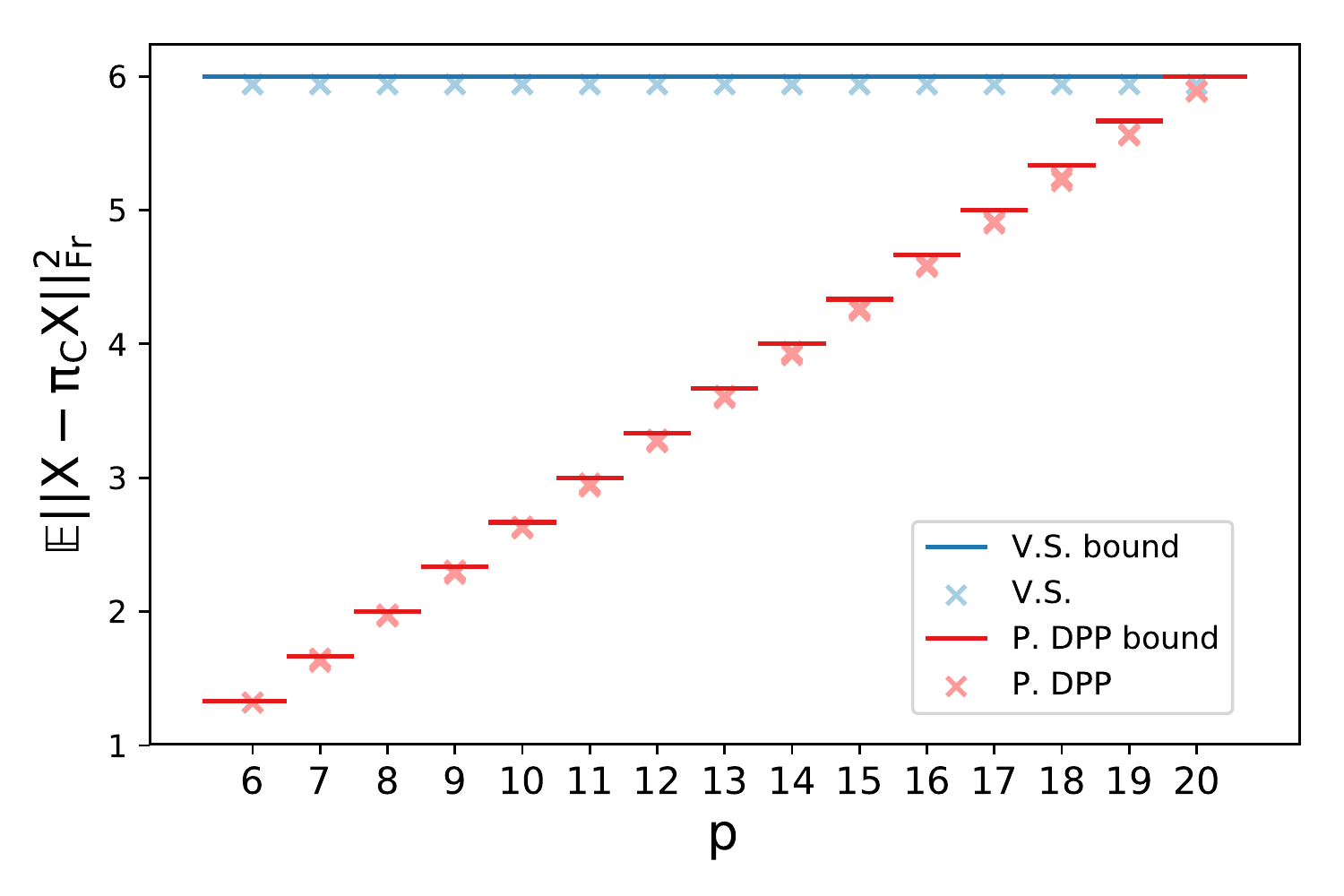}}\\
 \setcounter{subfigure}{1}%
\subfloat[$\Sigma_{3,\text{smooth}}$, $k=3$]{\includegraphics[width= 0.5\textwidth]{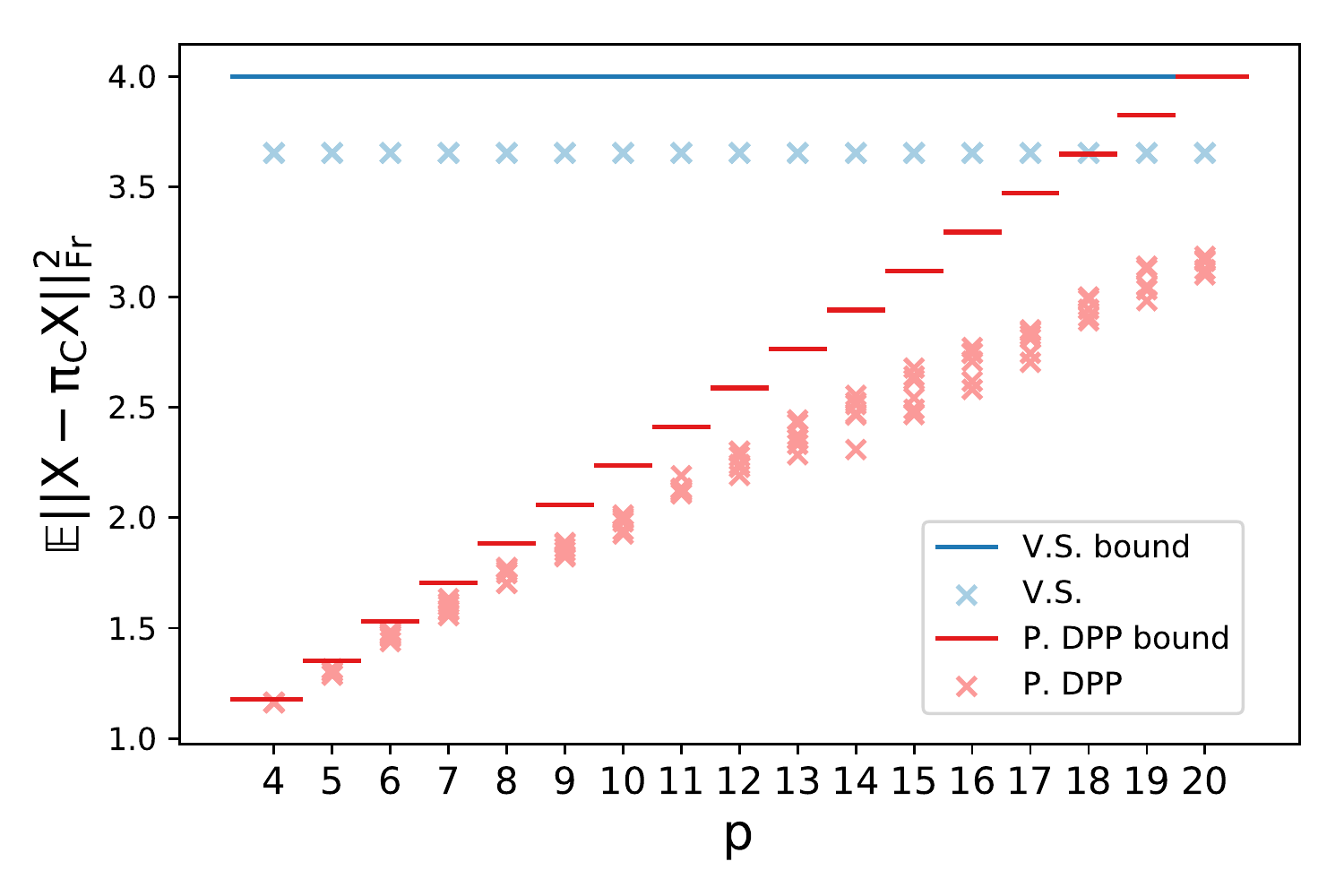}}
\setcounter{subfigure}{4}%
\subfloat[$\Sigma_{5,\text{smooth}}$, $k=5$]{\includegraphics[width= 0.5\textwidth]{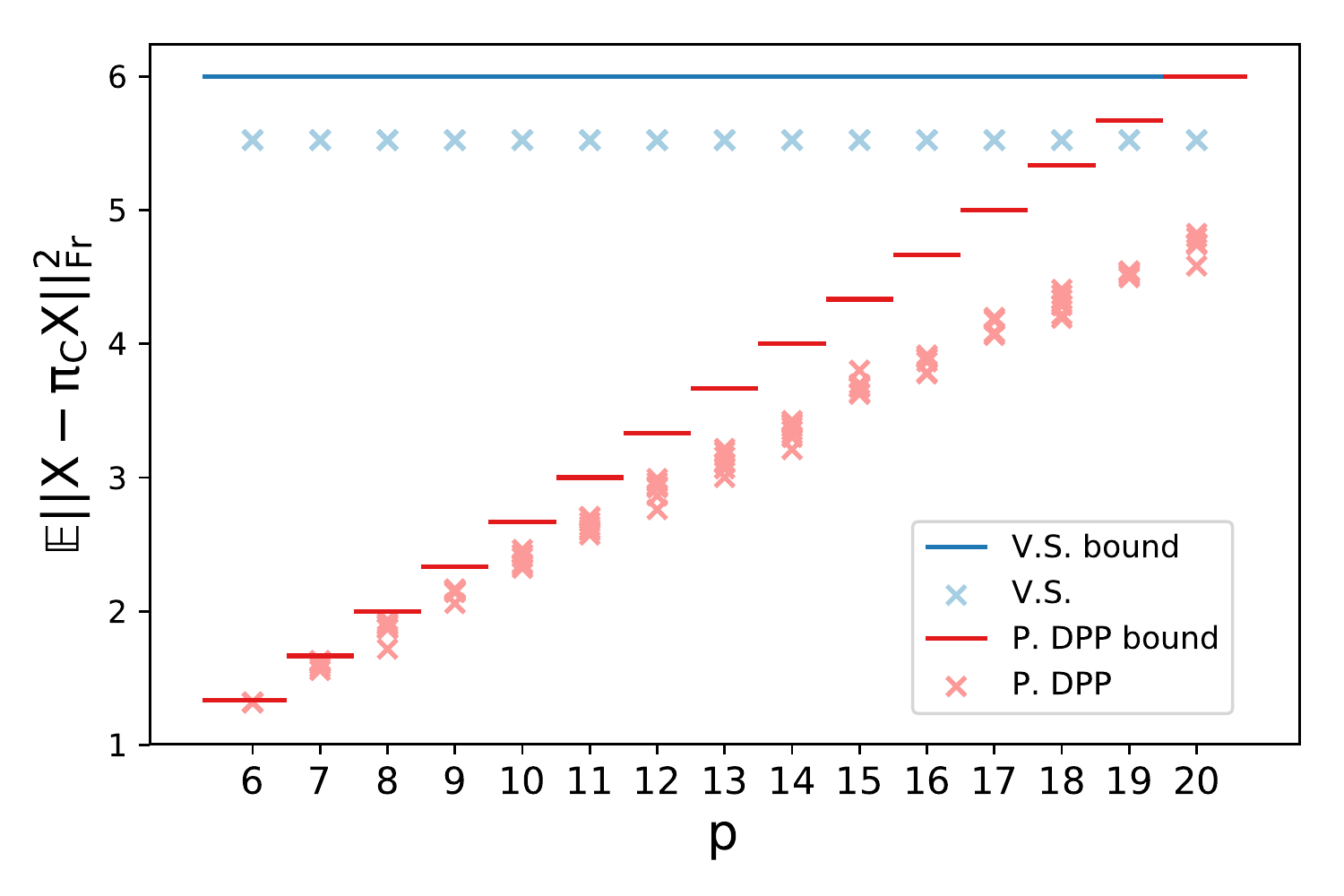}}\\

\setcounter{subfigure}{3}%
\subfloat[$I_{20}$, $k=3$]{\includegraphics[width= 0.5\textwidth]{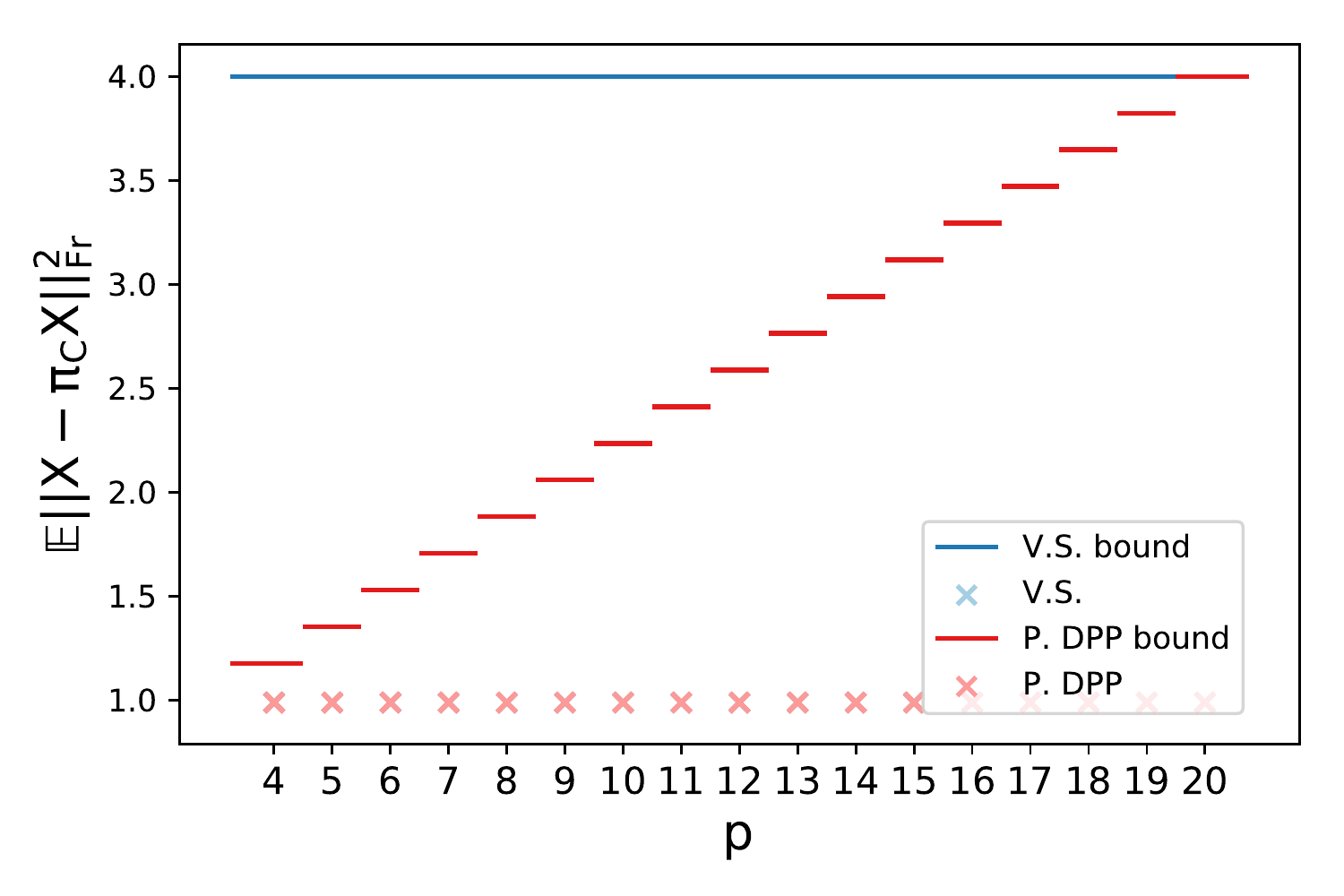}}
\setcounter{subfigure}{5}%
\subfloat[$I_{20}$, $k=5$]{\includegraphics[width= 0.5\textwidth]{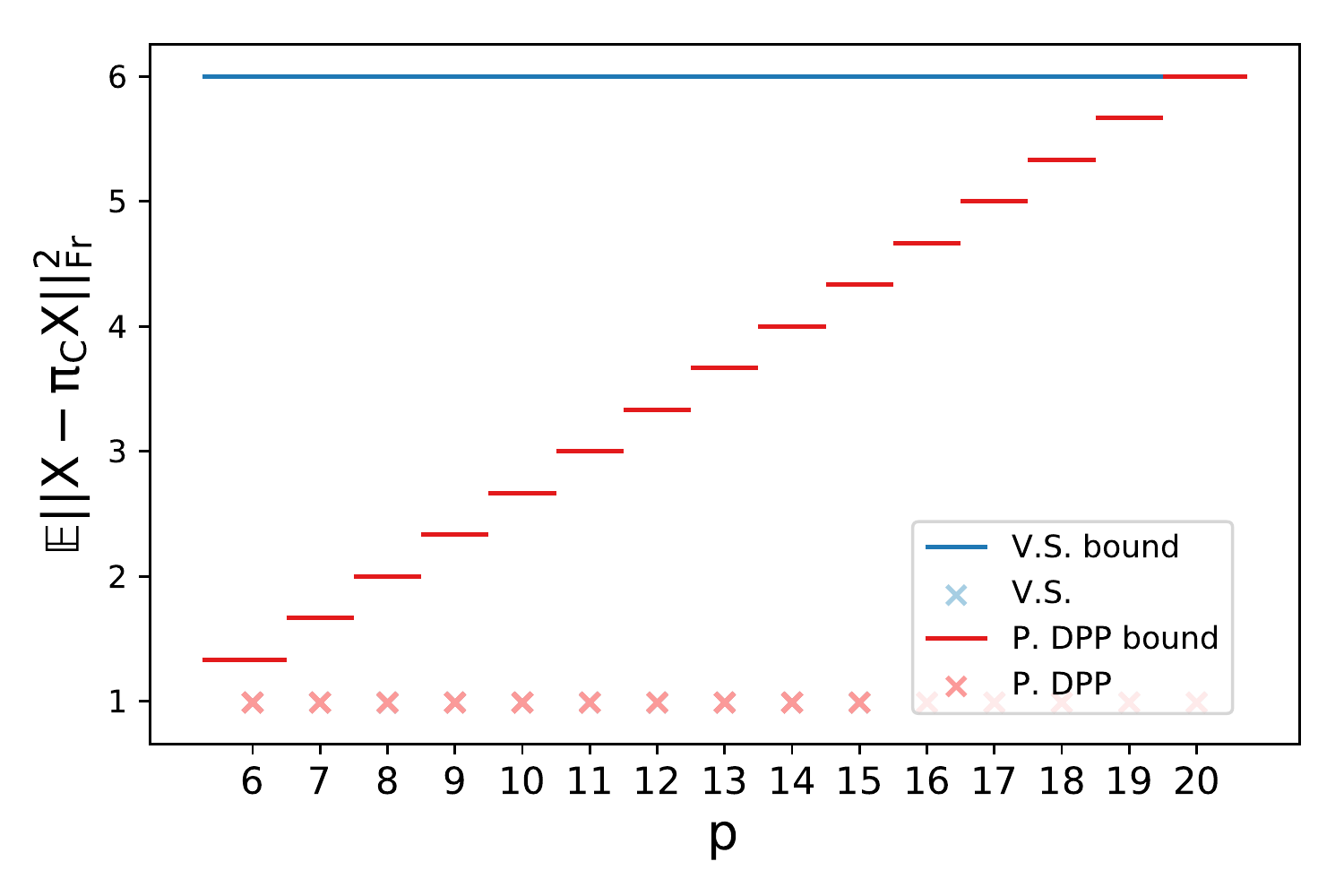}}

\caption{Realizations and bounds for $\mathbb{E} \|\bm{X}- \Pi_{S}^{\Fr} \bm{X}\|_{\Fr}^{2}$ as a function of the sparsity level $p$. \label{fig:toydatasets_vs_dpp_comparison_0}}
\end{figure}


Figure~\ref{fig:toydatasets_vs_dpp_comparison_0} compares, on the one hand, the theoretical bounds in Theorem~\ref{thrm:volume_sampling_theorem} for volume sampling and Proposition~\ref{hypo_one_two_proposition} for the projection DPP, to the numerical evaluation of the expected error for sampled toy datasets on the other hand. The x-axis indicates various sparsity levels $p$. The unit on the $y$-axis is the error of PCA. There are 400 crosses on each subplot: each of the 200 matrices appears once for both algorithms. The 200 matrices are spread evenly across the values of $p$.

The VS bounds in $(k+1)$ are independent of $p$. They appear to be tight for projection spectra, and looser for smooth spectra. For the projection DPP, the bound $(k+1)\frac{p-k}{d-k}$ is linear in $p$, and can be much lower than the bound of VS. The numerical evaluations of the error also suggest that this DPP bound is tight for a projection spectrum, and looser in the smooth case. In both cases, the bound is representative of the actual behavior of the algorithm.

The bottom row of Figure~\ref{fig:toydatasets_vs_dpp_comparison_0} displays the same results for identity spectra, again for $k=3$ and $k=5$. This setting is extremely nonsparse and represents an arbitrarily bad scenario where even PCA would not make much practical sense. Both VS and DPP sampling perform constantly badly, and all crosses superimpose at $y=1$, which indicates the PCA error. In this particular case, our linear bound in $p$ is not representative of the actual behavior of the error. This observation can be explained for volume sampling using Theorem~\ref{thm:schur_convex_volume_sampling}, which states that the expected squared error under VS is Schur-concave, and is thus minimized for flat spectra. We have no similar result for the projection DPP.

Figure~\ref{fig:toydatasets_vs_dpp_comparison_1} provides a similar comparison for the two smooth spectra $\bm{\Sigma}_{3,\text{smooth}}$ and $\bm{\Sigma}_{5,\text{smooth}}$, but this time using the effective sparsity level $p_{\eff}(\theta)$ introduced in Theorem~\ref{prop:p_eff_proposition}. We use $\theta=1/2$; qualitatively, we have observed the results to be robust to the choice of $\theta$. The $200$ sampled matrices are now unevenly spread across the $x$-axis, since we do not control $p_{\eff}(\theta)$. Note finally that the DPP here is conditioned on the event $\{S \cap T_{p_{\mathrm{eff}}(\theta)}=\emptyset\}$, and sampled using an additional rejection sampling routine as detailed below Theorem~\ref{prop:p_eff_proposition}.

\begin{figure}
    \centering
\subfloat[$\Sigma_{3,\text{smooth}}$]{\includegraphics[width= 0.5\textwidth]{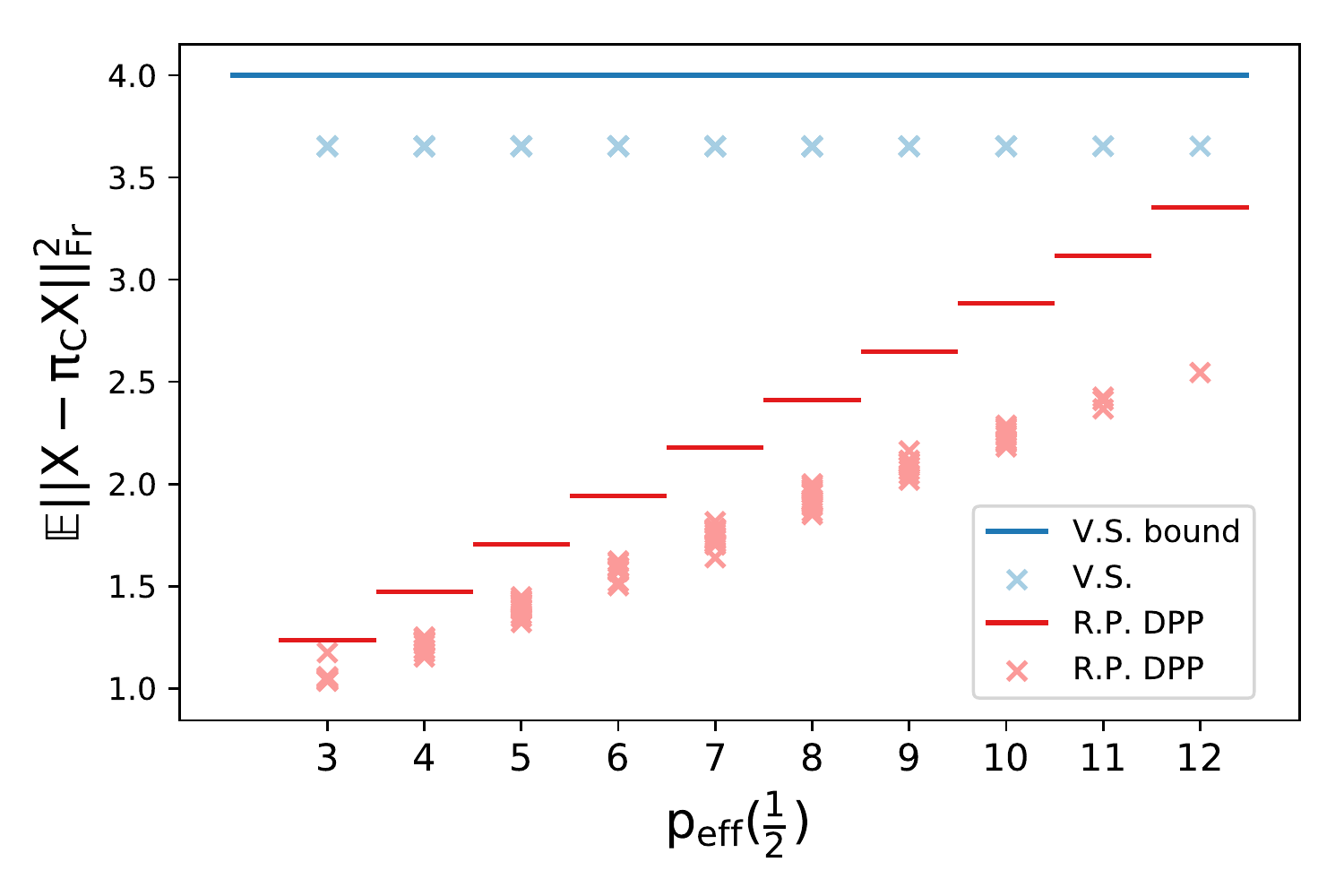}}
\subfloat[$\Sigma_{5,\text{smooth}}$]{\includegraphics[width= 0.5\textwidth]{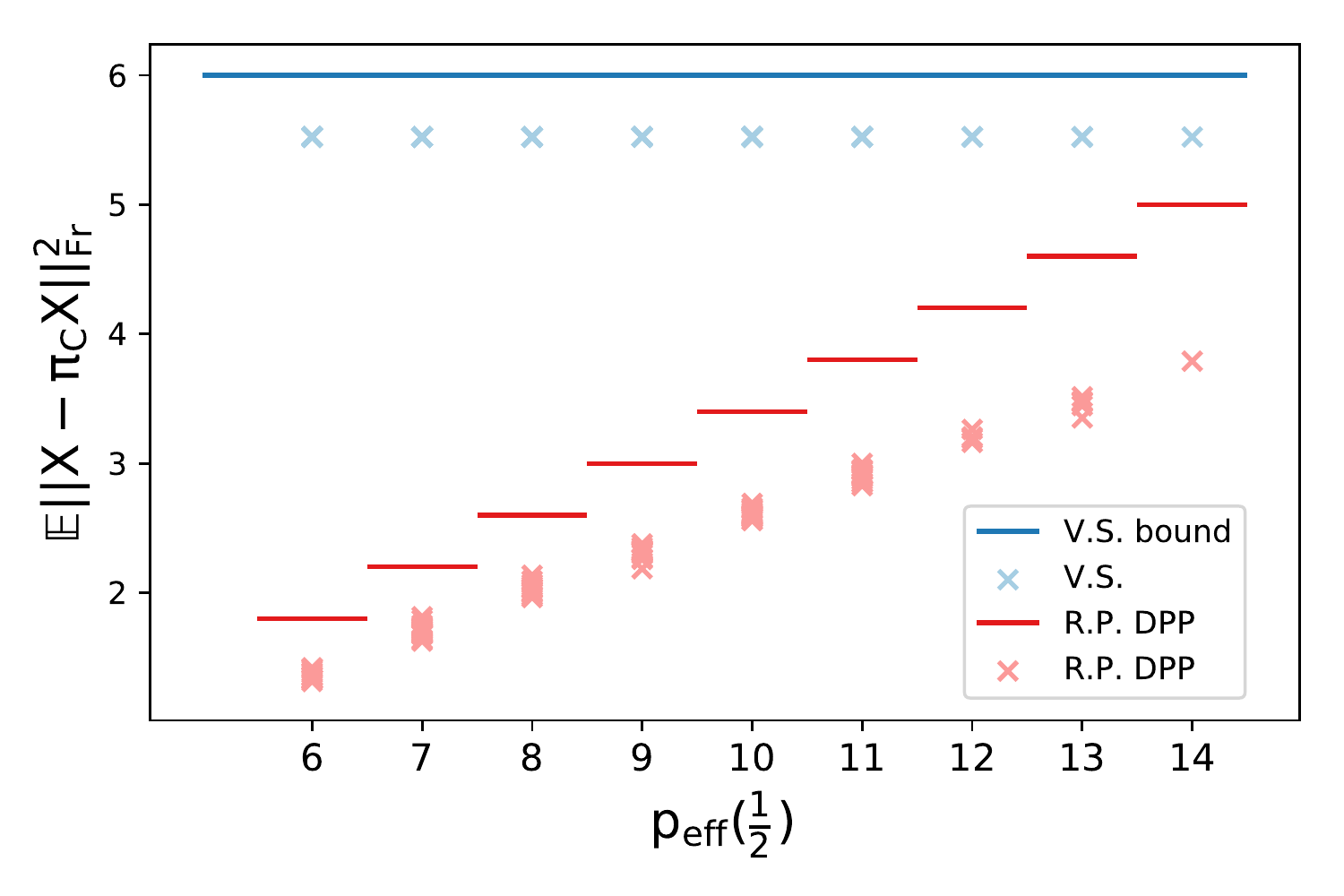}}\\

\caption{Realizations and bounds for $\mathbb{E} \|\bm{X}- \Pi_{S}^{\Fr} \bm{X}\|_{\Fr}^{2}$ as a function of the effective sparsity level $p_{\mathrm{eff}}(\frac{1}{2})$.
\label{fig:toydatasets_vs_dpp_comparison_1}}
\end{figure}

\begin{figure}
    \centering
\subfloat[$\Sigma_{3,\text{smooth}}$]{\includegraphics[width= 0.5\textwidth]{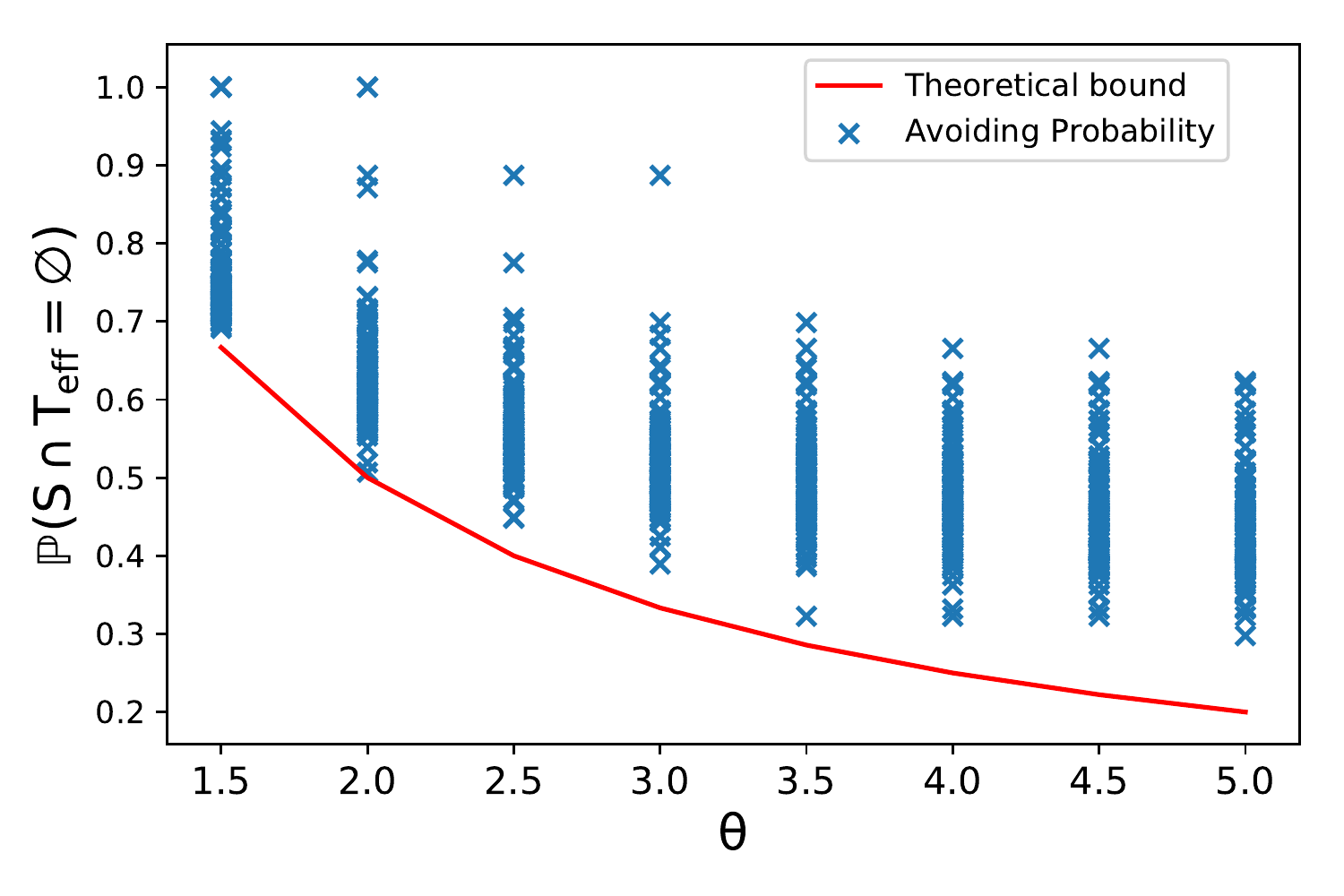}}
\subfloat[$\Sigma_{5,\text{smooth}}$]{\includegraphics[width= 0.5\textwidth]{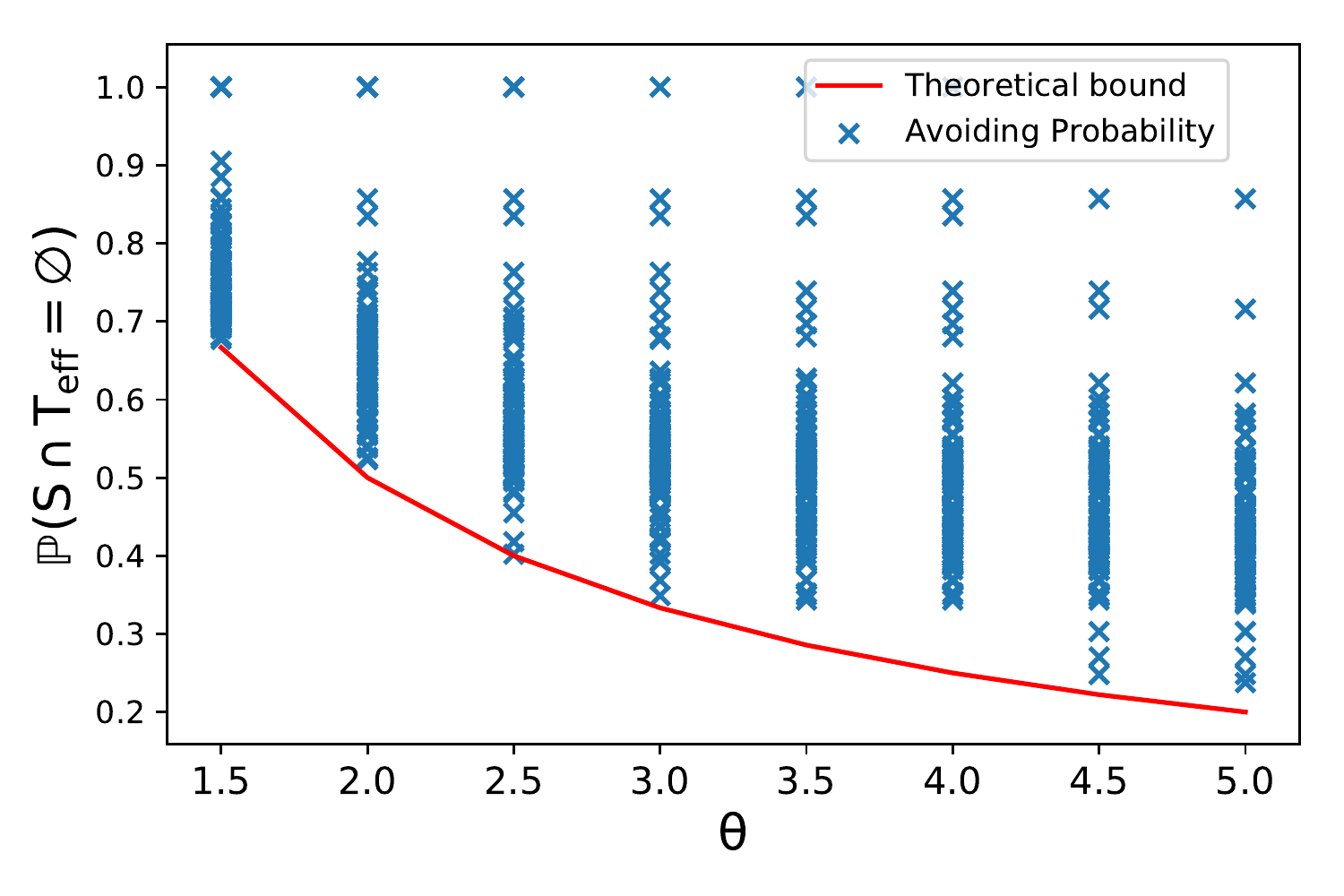}}\\

\caption{Realizations and bounds for the avoiding probability $\Prb(S \cap T_{p_{\mathrm{eff}}(\theta)} = \emptyset)$ in Theorem~\ref{prop:p_eff_proposition} as a function of $\theta$.
\label{fig:toydatasets_avoiding_proba}}
\end{figure}

For the DPP, the bound is again linear on the effective sparsity level $p_{\mathrm{eff}}(\frac{1}{2})$, and can again be much lower than the VS bound. The behavior of both VS and the projection DPP are similar to the exact sparsity setting of Figure~\ref{fig:toydatasets_vs_dpp_comparison_0}: the DPP has uniformly better bounds and actual errors, and the bound reflects the actual behavior, if relatively loosely when $p_{\eff}(1/2)$ is large.

Figure~\ref{fig:toydatasets_avoiding_proba} compares the theoretical bound in Theorem~\ref{prop:p_eff_proposition} for the avoiding probability $\Prb(S \cap T_{p_{\mathrm{eff}}(\theta)} = \emptyset)$ with 200 realizations, as a function of $\theta$. More precisely, we drew 200 matrices $\bm{X}$, and then for each $\bm{X}$, we computed exactly --~by enumeration~-- the value $\Prb(S \cap T_{p_{\mathrm{eff}}(\theta)} = \emptyset)$ for all values of $\theta$. The only randomness is thus in the sampling of $\bm{X}$, not the evaluation of the probability. The results suggest again that the bound is relatively tight.

\subsection{Real datasets}
\label{s:realDatasets}
This section compares the empirical performances of several column subset selection algorithms on the datasets in Table~\ref{table:real_datasets}.

\begin{table}[h]
\centering
 \begin{tabular}{| c| c | c| c|}
 \hline
  Dataset & Application domain & $N \times d$  &  References\\
 \hline
 Colon & genomics & $62 \times 2000$  & \citep{Al99}\\
 \hline
 Leukemia & genomics &$72 \times 7129$ & \citep{Go99}\\
 \hline
 Basehock & text processing &$1993 \times 4862$ & \citep{Li17} \\
 \hline
 Relathe & text processing & $1427 \times 4322$ & \citep{Li17}\\
 \hline
\end{tabular}
\caption{Datasets used in the experimental section.
\label{table:real_datasets}}
\end{table}

These datasets are illustrative of two extreme situations regarding the sparsity of the $k$-leverage scores. For instance, the dataset Basehock has a very sparse profile of $k$-leverage scores, while the dataset Colon has a quasi-uniform distribution of $k$-leverage scores, see Figures~\ref{fig:colon_vs_basehock_comparison} (a) \& (b).\rb{Use subfigure. Also, the legend is weird.}

We consider the following algorithms presented in Section~\ref{s:relatedwork}: 1) the projection DPP with marginal kernel $\bm{K}=\bm{V}_{k}^{}\bm{V}_{k}^{\Tran}$, 2) volume sampling, using the implementation proposed by \cite{conf/icml/KuleszaT11}, 3) deterministically picking the largest $k$-leverage scores, 4) pivoted QR as in \citep{Golu65}, although the only known bounds for this algorithm are for the spectral norm, and 5) double phase, with $c$ manually tuned to optimize the performance, usually around $c \approx 10k$. \rb{We should harmonize capitals or regular script for algorithms}

The rest of Figure~\ref{fig:colon_vs_basehock_comparison} sums up the empirical results of the previously described algorithms on the Colon and Basehock datasets.
Figures~\ref{fig:colon_vs_basehock_comparison} (c) \& (d) illustrate the results of the five algorithms in the following setting. An ensemble of 50 subsets are sampled from each algorithm. We give the corresponding boxplots for the Frobenius errors, on Colon and Basehock respectively.
We observe that the increase in performance using projection DPP compared to volume sampling is more important for the Basehock dataset than for the Colon dataset: this improvement can be explained by the sparsity of the $k$-leverage scores as predicted by our approximation bounds.
Deterministic methods (largest leverage scores and pivoted QR) perform well compared with other algorithms on the Basehock dataset; in contrast, they display very bad performances on the Colon dataset.
The double phase algorithm has the best results on both datasets. However its theoretical guarantees cannot predict such an improvement, as noted in Section~\ref{s:relatedwork}. The performance of the projection DPP is comparable to those Double Phase and makes it a close second, with a slightly larger gap on the Colon dataset. We emphasize that our approximation bounds are sharp compared to numerical observations. \rb{Make algorithm labels larger.}

Figures~\ref{fig:colon_vs_basehock_comparison} (e) \& (f) show results obtained using a classical boosting technique for randomized algorithms. We repeat 20 times: sample 50 subsets and take the best subset selection. Displayed boxplots are for these 20 best results. The same comments apply as without boosting.

\begin{figure}
    \centering
    \captionsetup[subfigure]{justification=centering}
\subfloat[$k$-leverage scores profile for the dataset Basehock (k=10).]{\includegraphics[width= 0.5\textwidth, height = 0.2\textheight]{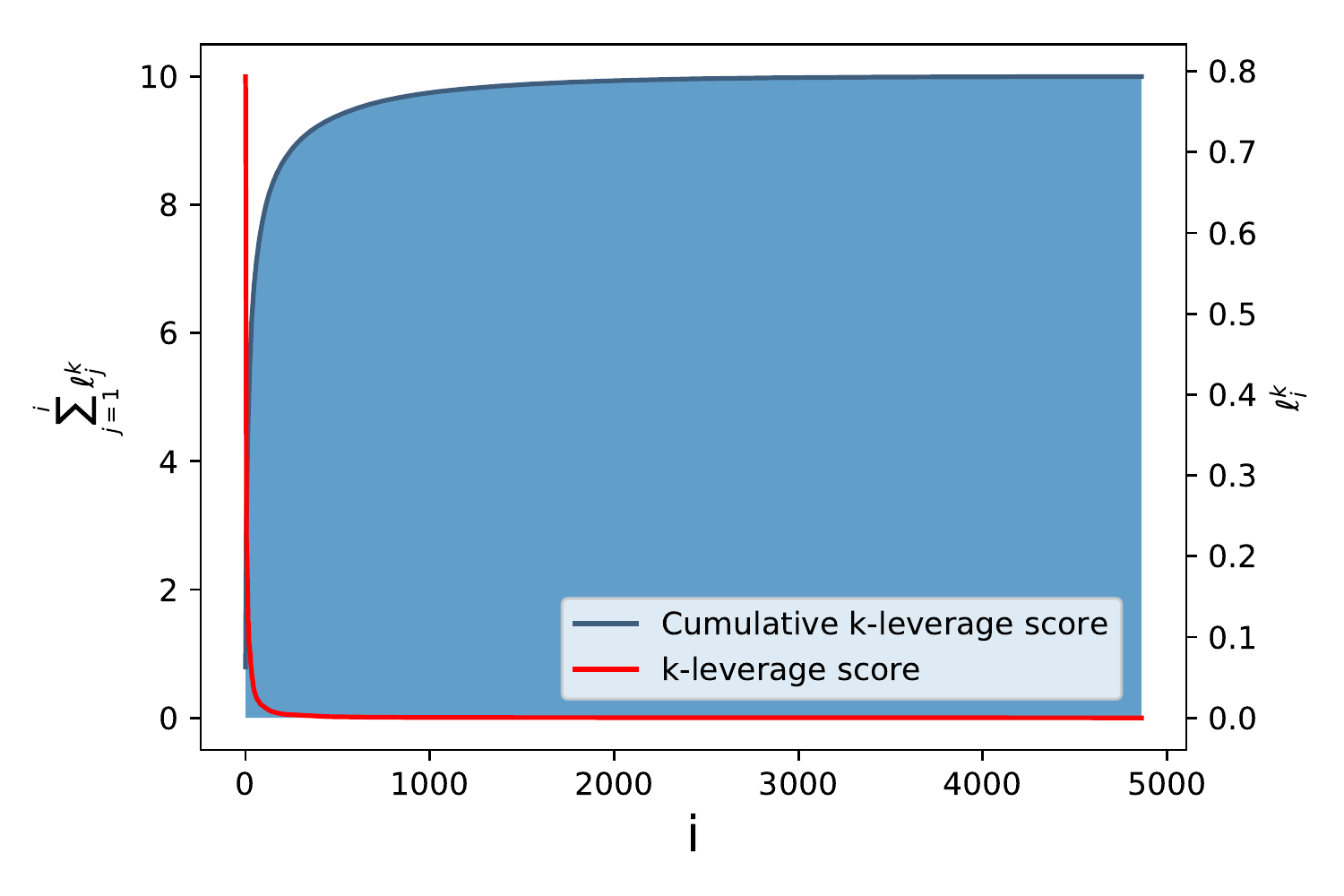}\label{subfig:basehock_cumul_klv}}~
\subfloat[$k$-leverage scores profile for the dataset Colon (k=10).]{\includegraphics[width= 0.5\textwidth, height = 0.2\textheight]{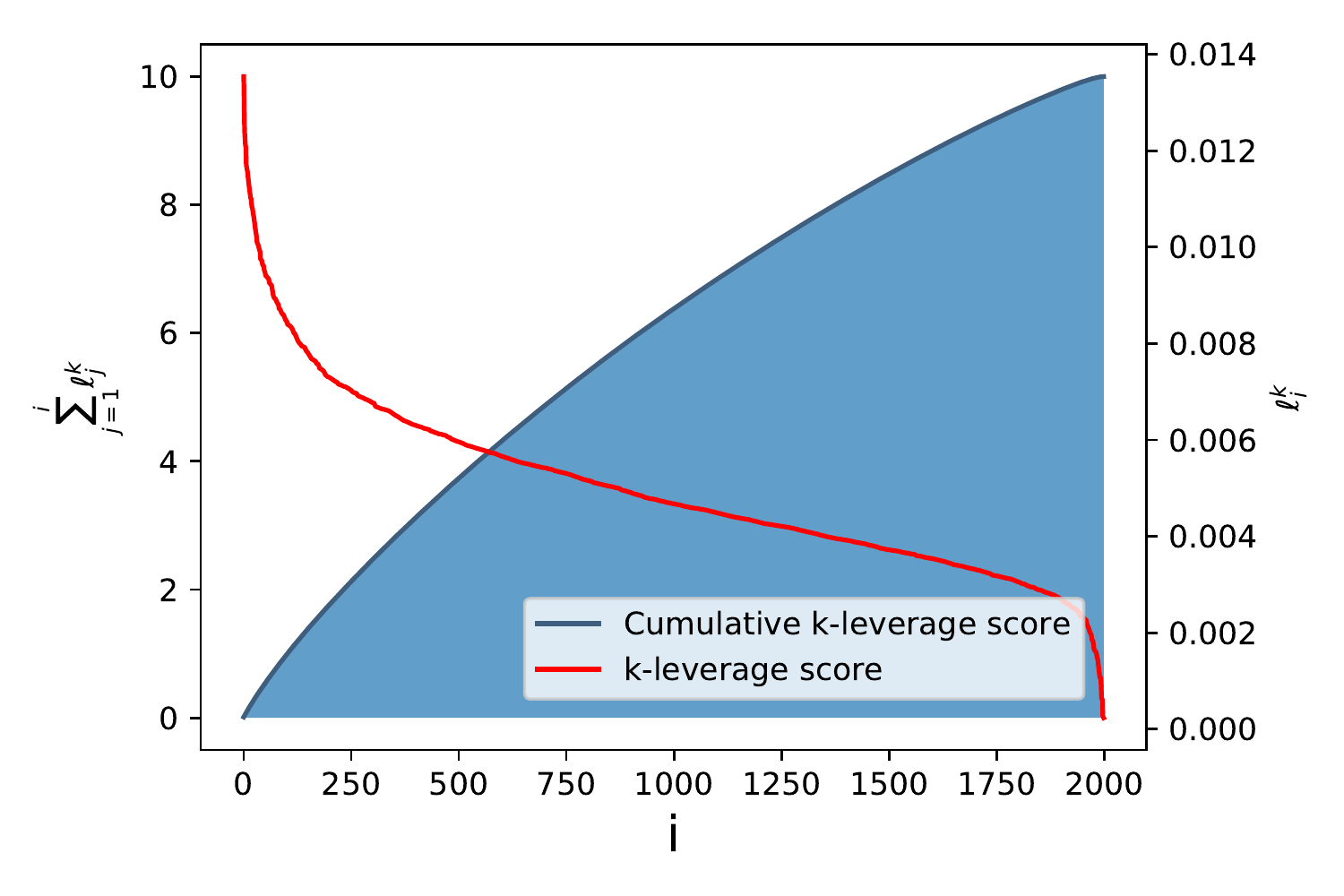}\label{subfig:colon_cumul_klv}}\\
\subfloat[Boxplots of $\|\bm{X}-\Pi_{S}^{\Fr}\bm{X}\|_{\Fr}$ on a batch of 50 samples for the five algorithms on the dataset Basehock (k=10).]{\includegraphics[width= 0.5\textwidth]{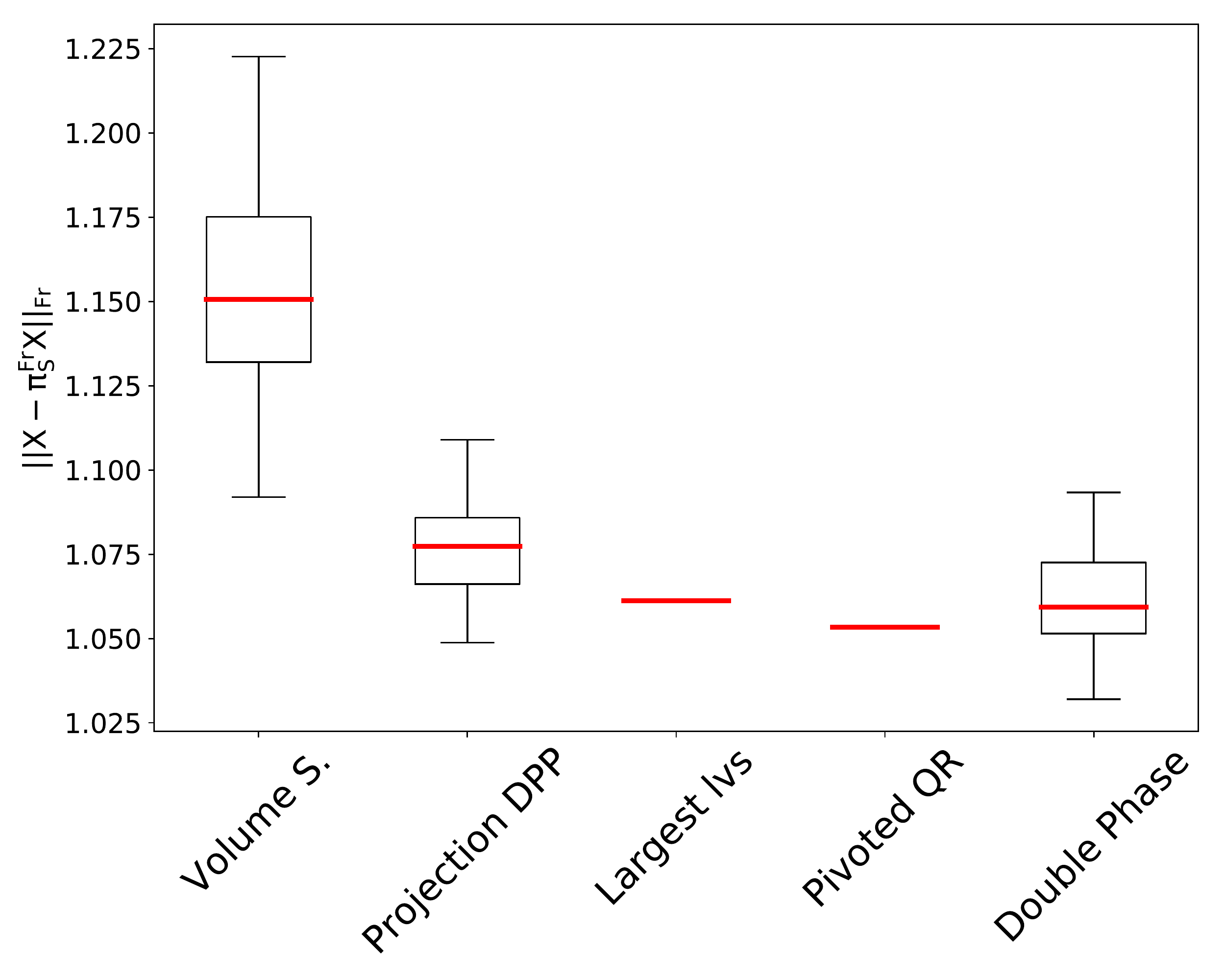}}~
\subfloat[Boxplots of $\|\bm{X}-\Pi_{S}^{\Fr}\bm{X}\|_{\Fr}$ on a batch of 50 samples for the five algorithms on the dataset Colon (k=10).]{\includegraphics[width= 0.5\textwidth]{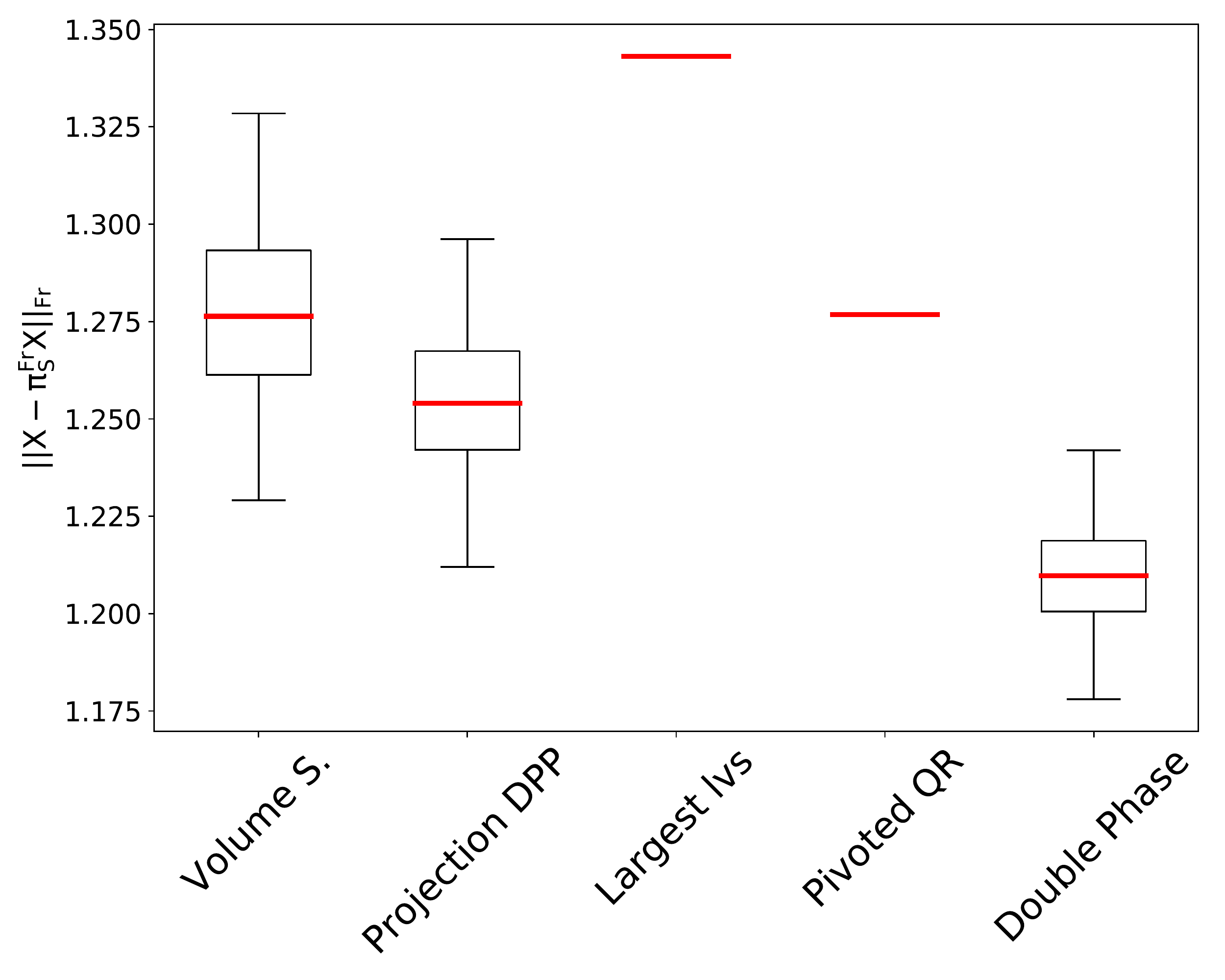}}\\
\subfloat[Boxplots of $\|\bm{X}-\Pi_{S}^{\Fr}\bm{X}\|_{\Fr}$ on a batch of 50 samples for the boosting of randomized algorithms on the dataset Basehock (k=10).]{\includegraphics[width= 0.5\textwidth]{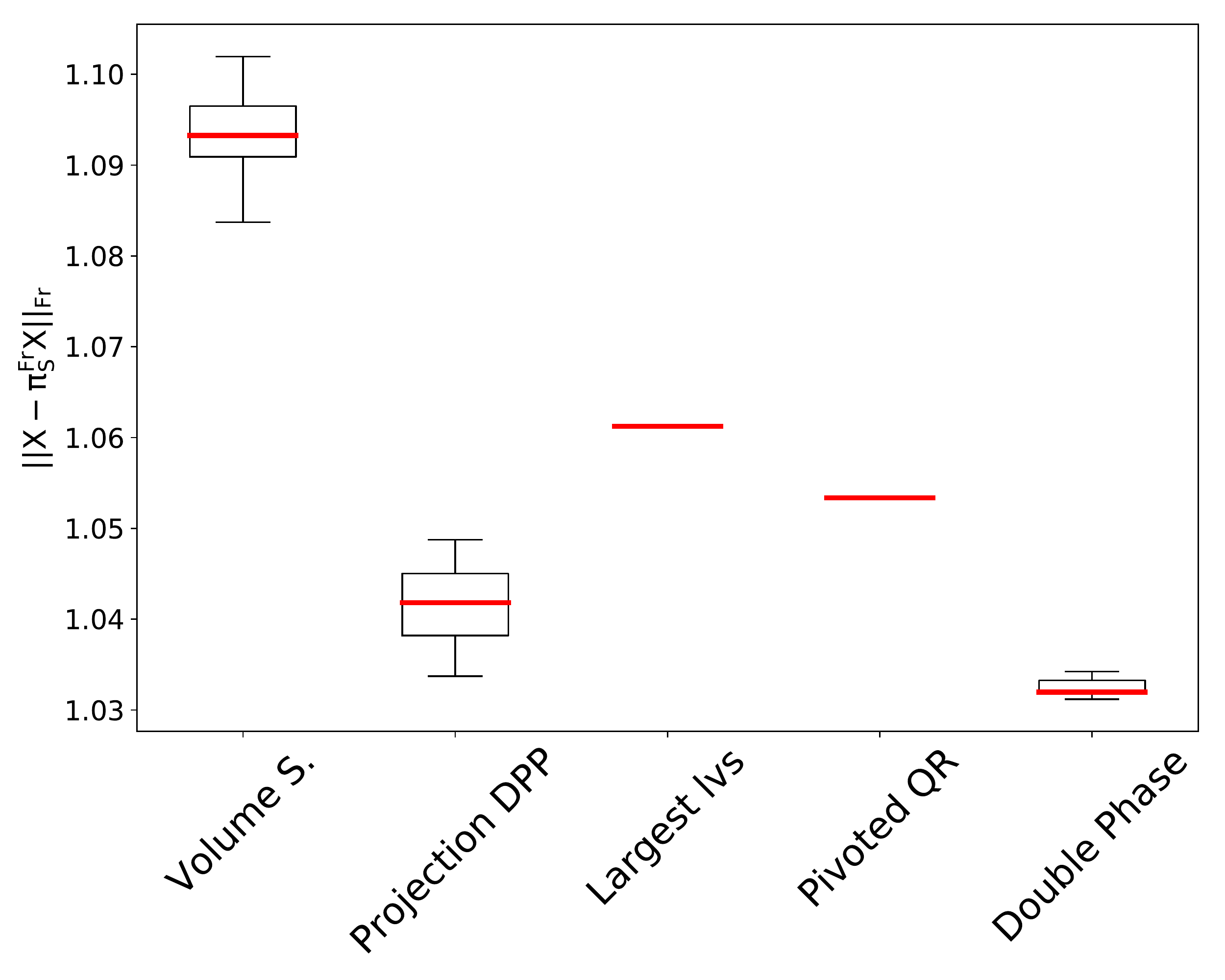}}~
\subfloat[Boxplots of $\|\bm{X}-\Pi_{S}^{\Fr}\bm{X}\|_{\Fr}$ on a batch of 50 samples for the boosting of randomized algorithms on the dataset Colon (k=10).]{\includegraphics[width= 0.5\textwidth]{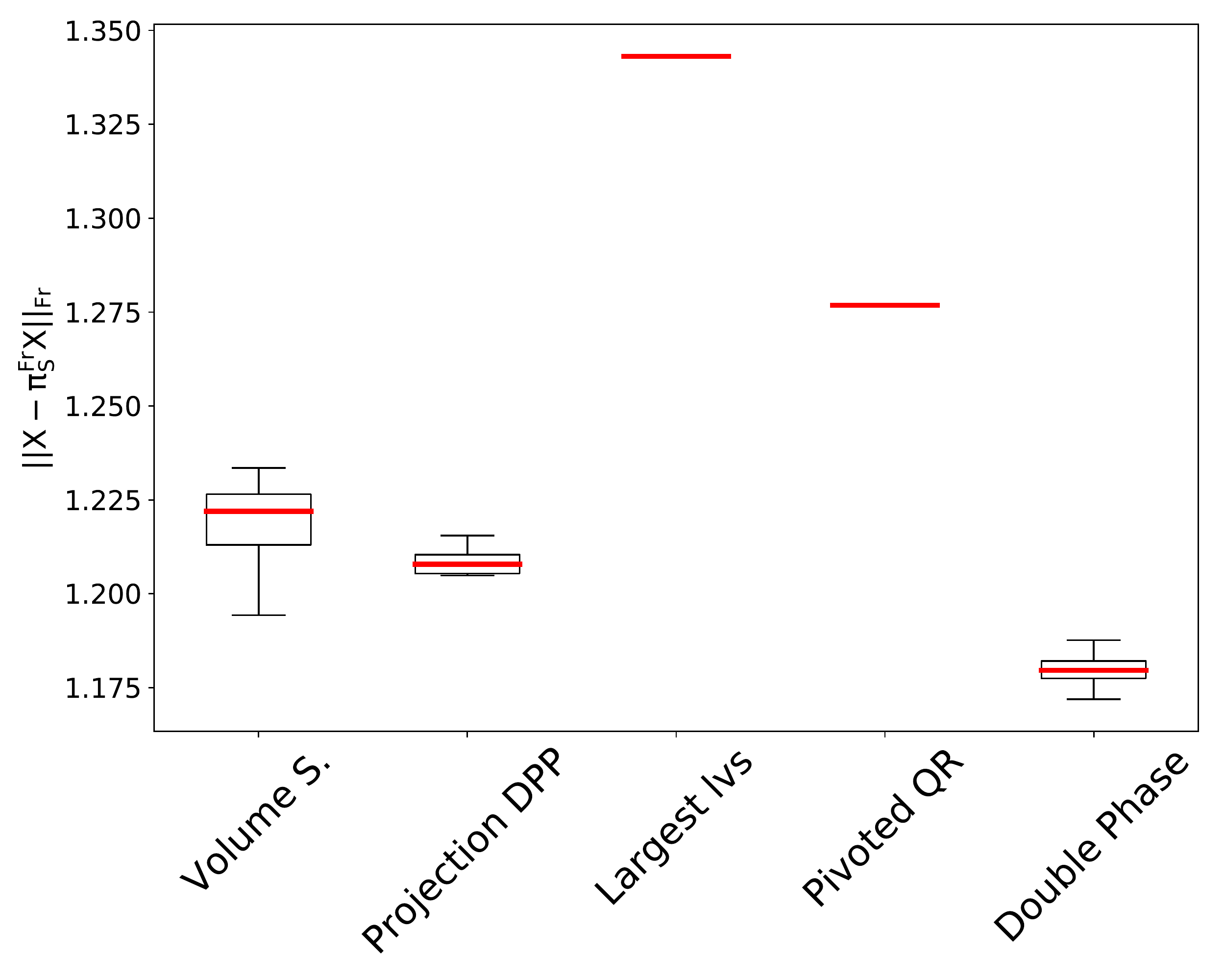}}\\
\caption{Comparison of several column subset selection algorithms for two datasets: Basehock and Colon.\label{fig:colon_vs_basehock_comparison}}
\end{figure}

Figure~\ref{fig:leukemia_vs_relathe_comparison} calls for similar comments, comparing this time the datasets Relathe (with concentrated profile of $k$-leverage scores) and Leukemia (with almost uniform profile of $k$-leverage scores).
\begin{figure}
    \centering
    \captionsetup[subfigure]{justification=centering}
\subfloat[$k$-leverage scores profile for the dataset Relathe (k=10).]{\includegraphics[width= 0.5\textwidth, height = 0.2\textheight]{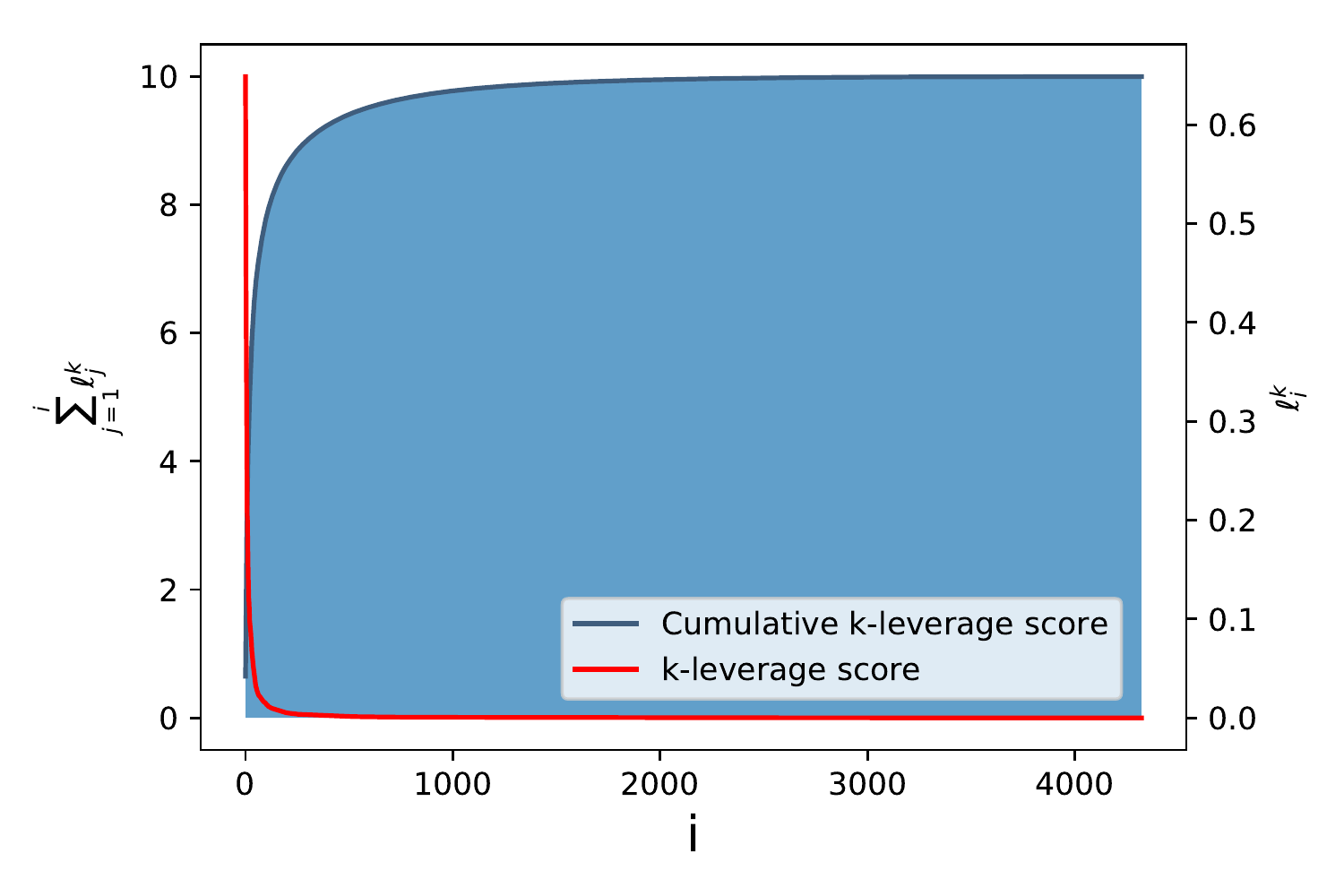}}~
\subfloat[$k$-leverage scores profile for the dataset Leukemia (k=10).]{\includegraphics[width= 0.5\textwidth, height = 0.2\textheight]{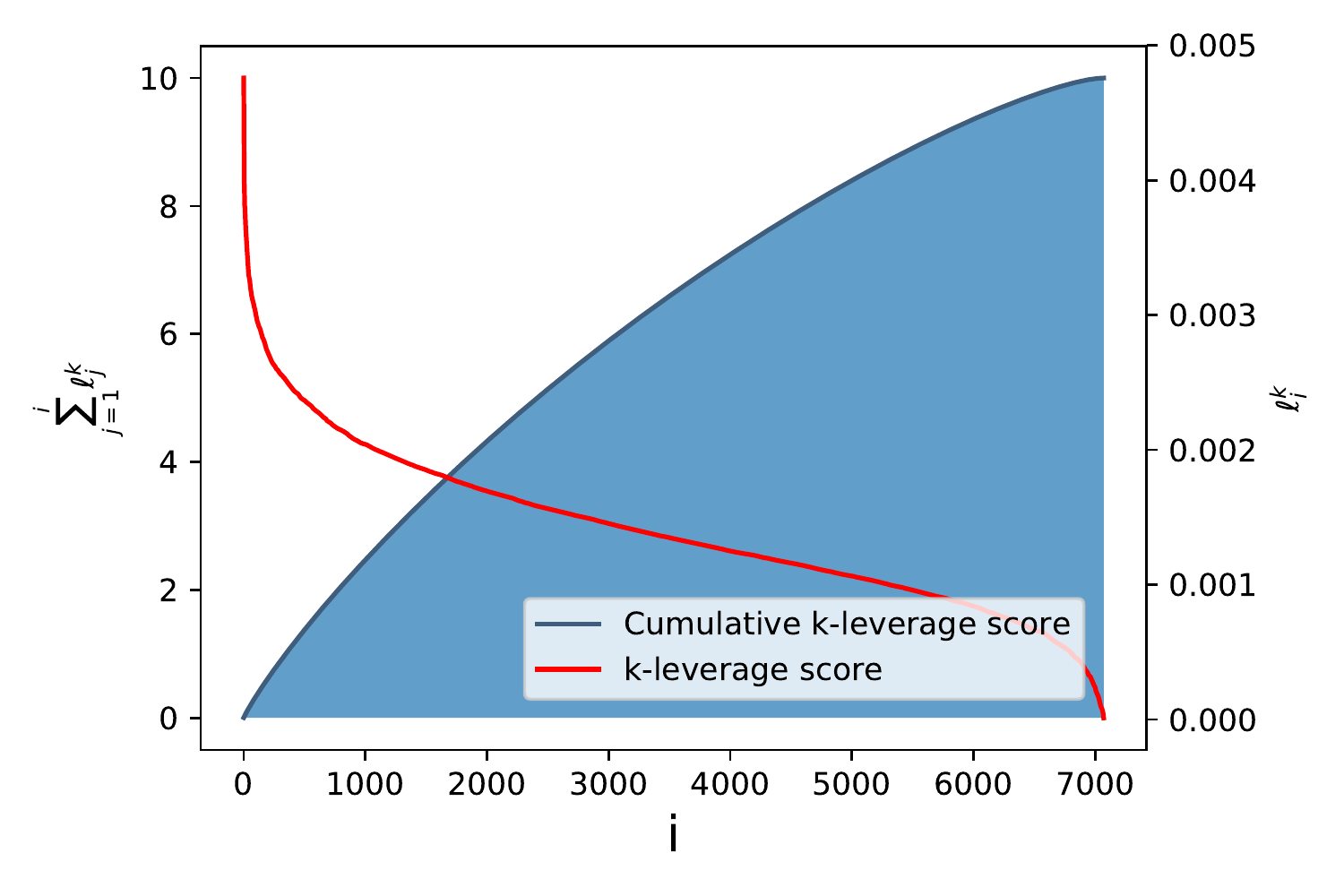}}\\
\subfloat[Boxplots of $\|\bm{X}-\Pi_{S}^{\Fr}\bm{X}\|_{\Fr}$ on a batch of 50 samples for the five algorithms on the dataset Relathe (k=10).]{\includegraphics[width= 0.5\textwidth]{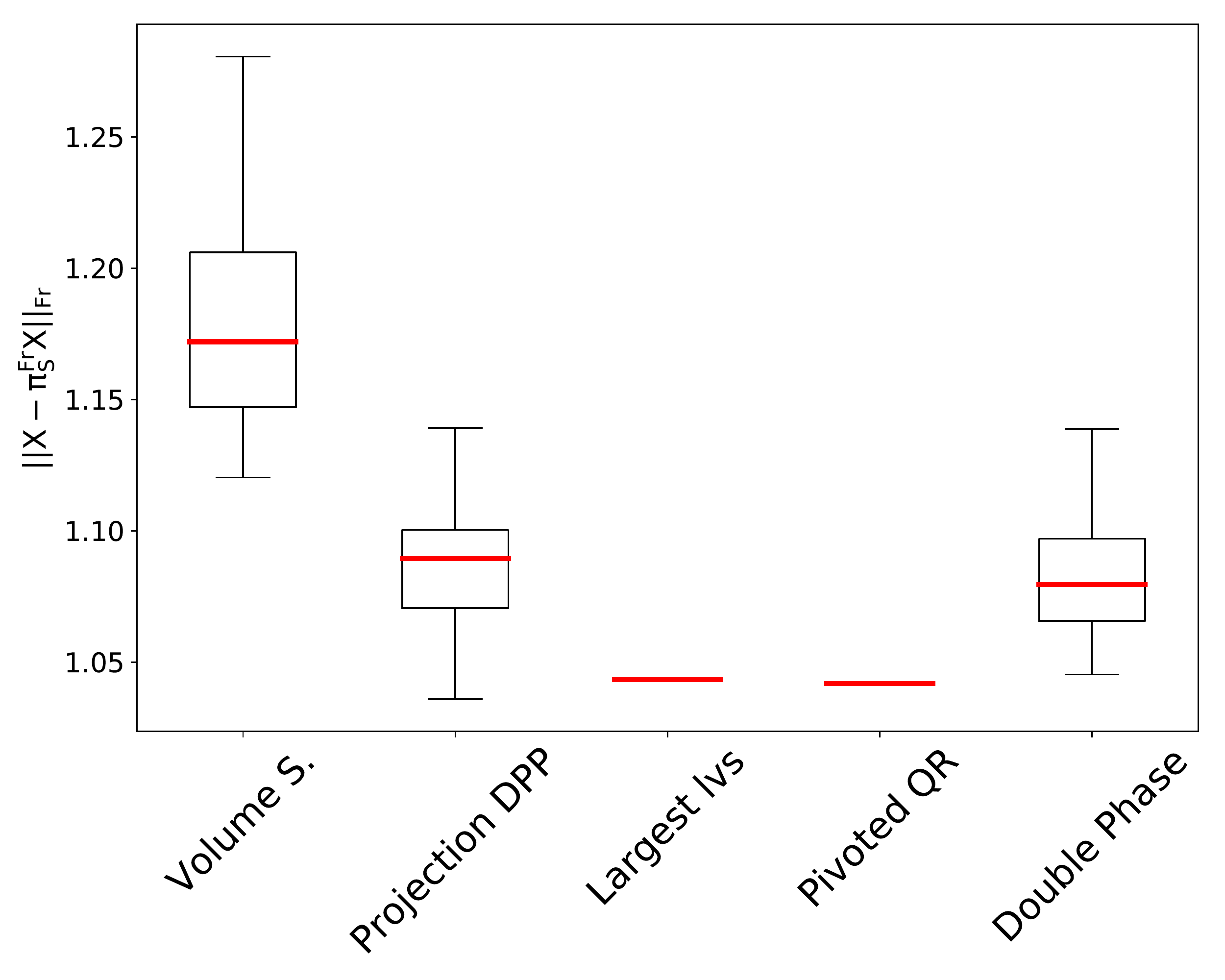}}~
\subfloat[Boxplots of $\|\bm{X}-\Pi_{S}^{\Fr}\bm{X}\|_{\Fr}$ on a batch of 50 samples for the five algorithms on the dataset Leukemia (k=10).]{\includegraphics[width= 0.5\textwidth, ]{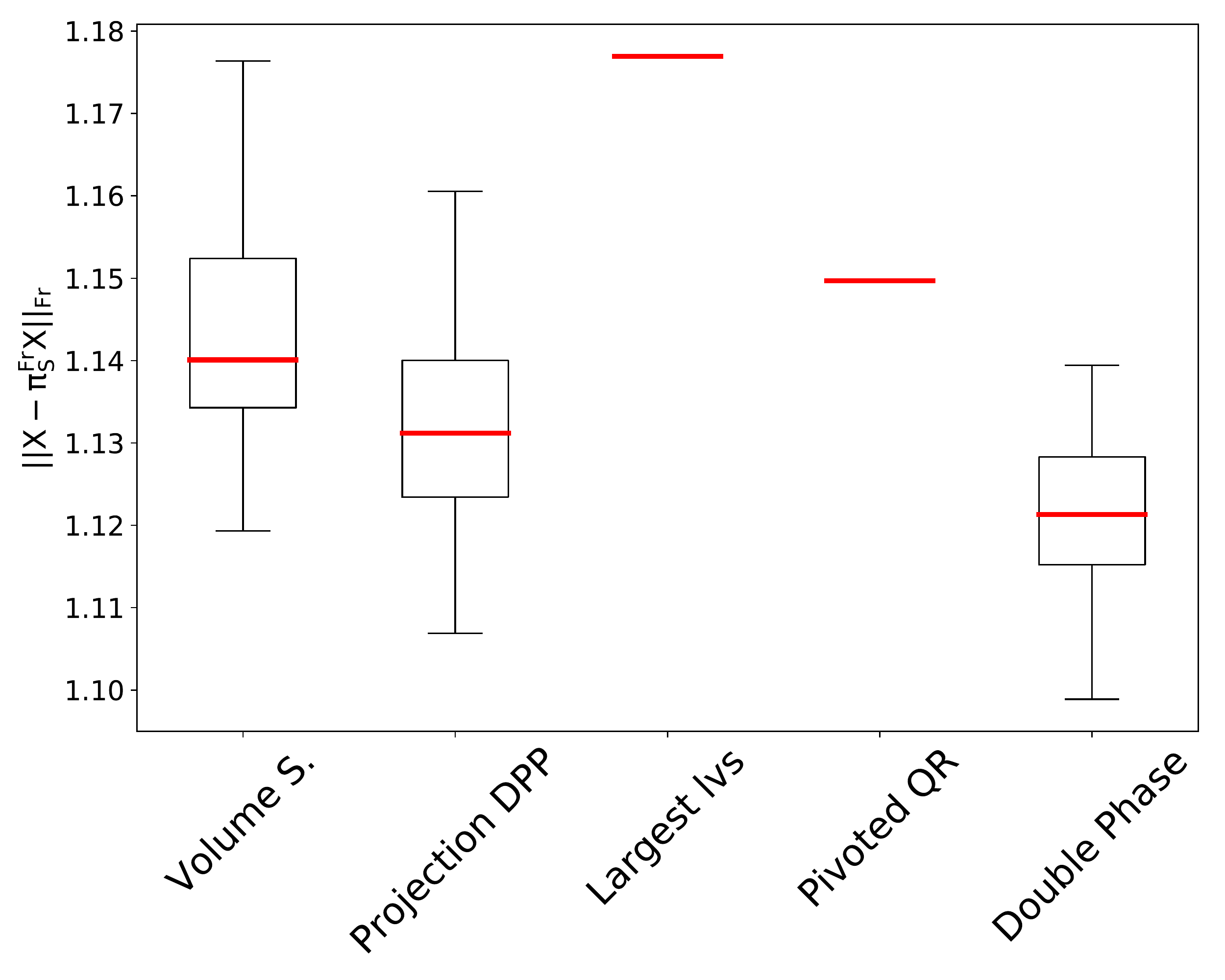}}\\
\subfloat[Boxplots of $\|\bm{X}-\Pi_{S}^{\Fr}\bm{X}\|_{\Fr}$ on a batch of 50 samples for the boosting of randomized algorithms on the dataset Relathe (k=10).]{\includegraphics[width= 0.5\textwidth]{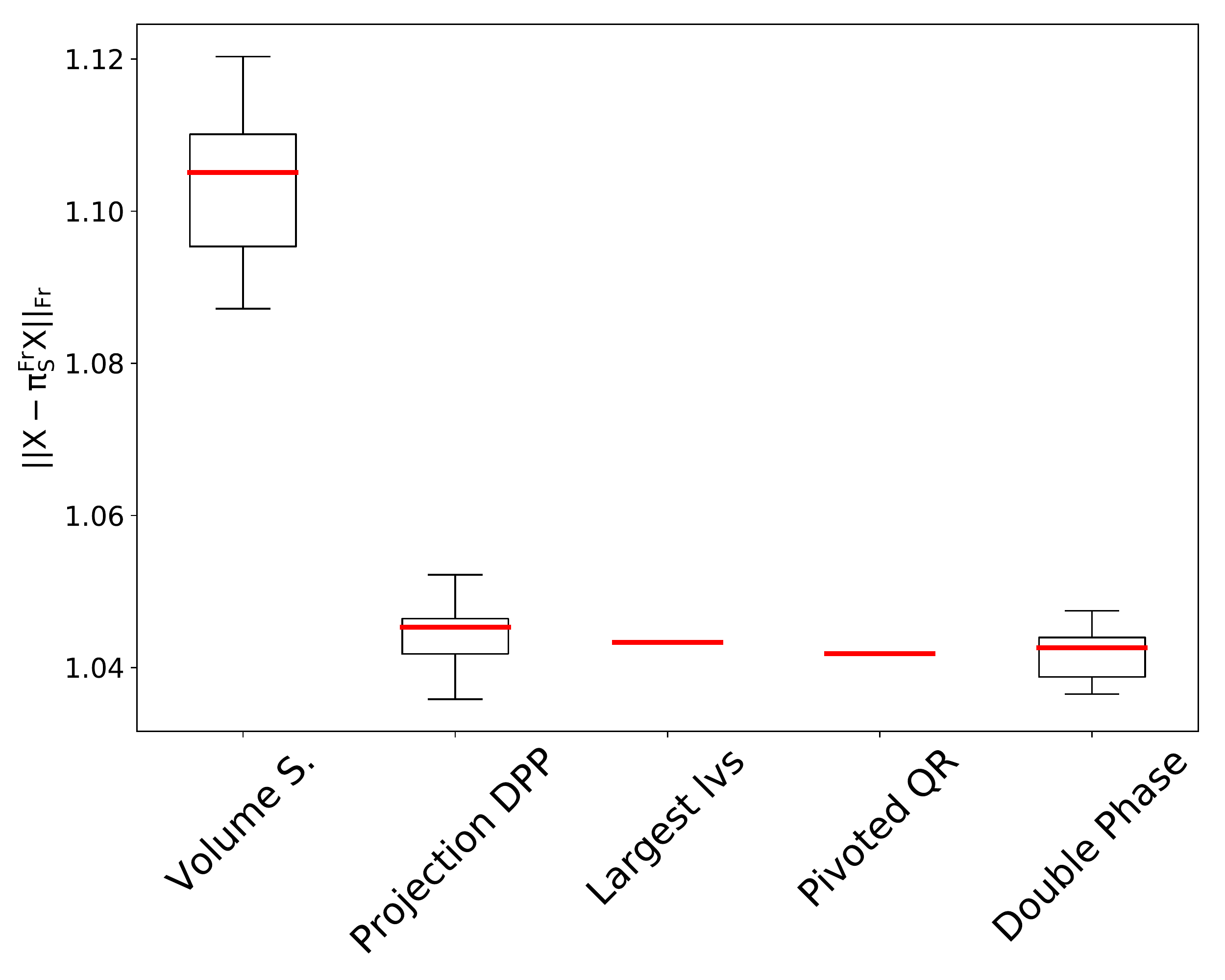}}~
\subfloat[Boxplots of $\|\bm{X}-\Pi_{S}^{\Fr}\bm{X}\|_{\Fr}$ on a batch of 50 samples for the boosting of randomized algorithms on the dataset Leukemia (k=10).]{\includegraphics[width= 0.5\textwidth]{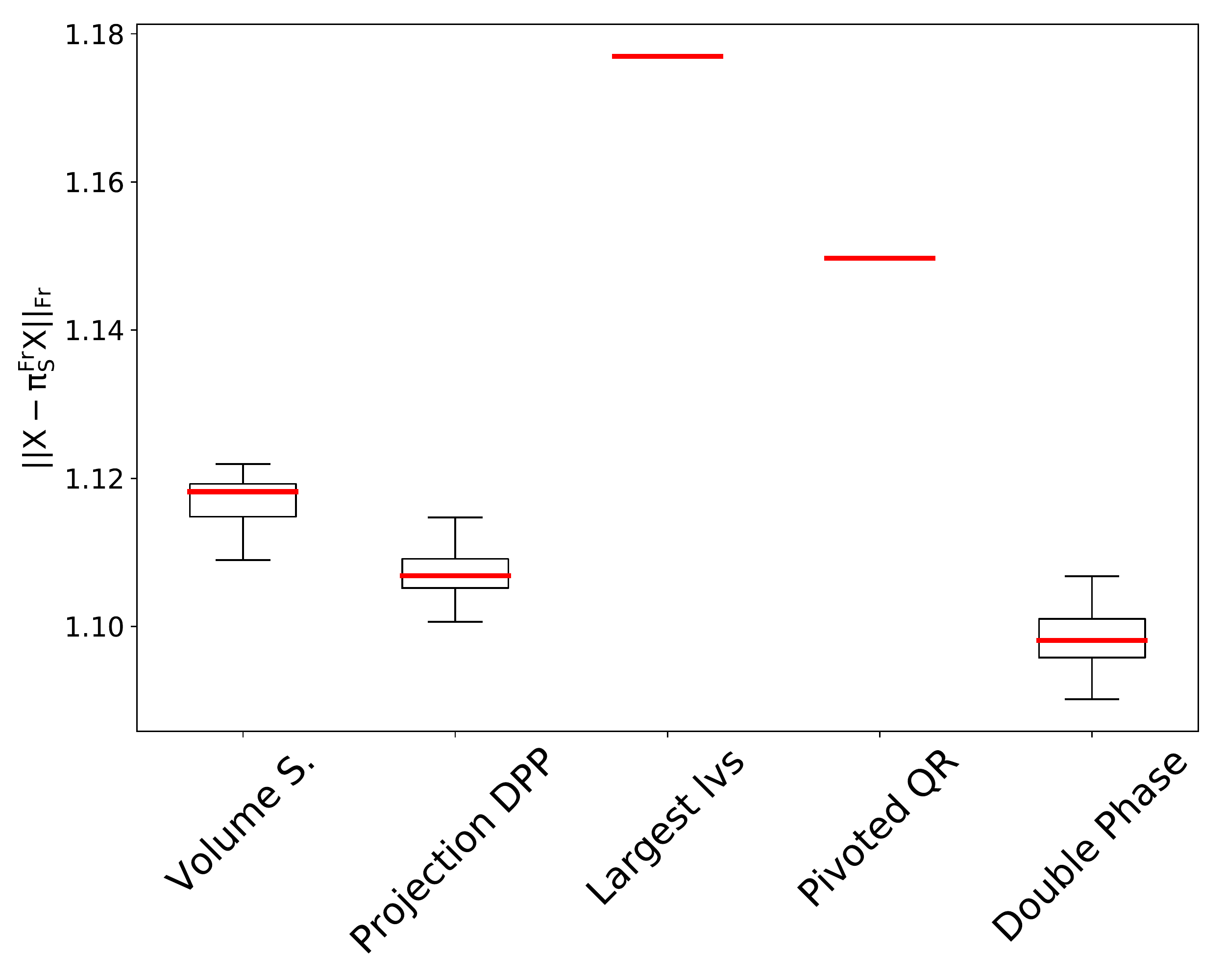}}\\
\caption{Comparison of several column subset selection algorithms for two datasets: Relathe and Leukemia. \label{fig:leukemia_vs_relathe_comparison}}
\end{figure}
\rb{We could also do an actual frequentist test. This'd be cleaner.}

\subsection{Discussion}
The performance of our algorithm has been compared to state-of-the-art column subset selection algorithms. We emphasize that the theoretical performances of the proposed approach take into account the sparsity of the $k$-leverage scores as in Proposition~\ref{hypo_one_two_proposition} or their fast decrease as in Proposition~\ref{prop:p_eff_proposition}, and that the bounds are in good agreement with the actual behavior of the algorithm. In contrast, state-of-the-art algorithms like volume sampling have looser bounds and worse performances, or like double phase display great performance but have overly pessimistic theoretical bounds.

\section{Conclusion}
\label{s:conclusion}

We have proposed, analyzed, and empirically investigated a new randomized column subset selection (CSS) algorithm. The crux of our algorithm is a discrete determinantal point process (DPP) that selects a diverse set of $k$ columns of a matrix $\bm{X}$. This DPP is tailored to CSS through its parametrization by the marginal kernel $\bm{K} = \bm{V}_{k}\bm{V}_{k}^{\Tran}$, where $\bm{V}_{k}$ are the first $k$ right singular vectors of the matrix $\bm{X}$. This specific kernel is related to volume sampling, the state-of-the-art for CSS guarantees in Frobenius and spectral norm.

We have identified generic conditions on the matrix $\bm{X}$ under which our algorithm has bounds that improve on volume sampling. In particular, our bounds highlight the importance of the sparsity and the decay of the $k$-leverage scores on the approximation performance of our algorithm. This resonates with the compressed sensing literature. We have further numerically illustrated this relation to the sparsity and decay of the $k$-leverage scores using toy and real datasets. In these experiments, our algorithm performs comparably to the so-called double phase algorithm, which is the empirical state-of-the-art for CSS despite more conservative theoretical guarantees than volume sampling. Thus, our DPP sampling inherits both favorable bounds and increased empirical performance under sparsity or fast decay of the $k$-leverage scores.

In terms of computational cost, our algorithms scale with the cost of finding the $k$ first right singular vectors, which is currently the main bottleneck. In the line of \citep{DMMW12} and \citep{BoDrMI11}, where the authors estimates the $k$-leverage scores using random projections, we plan to investigate the impact of random projections to estimate the full matrix $\bm{K}$ on the approximation guarantees of our algorithms.

Although often studied as an independent task, in practice CSS is often a prelude to a learning algorithm. We have considered linear regression and we have given a bound on the excess risk of a regression performed on the selected columns only. In particular, sparsity and decay of the $k$-leverage scores are again involved: the more localized the $k$-leverage scores, the smaller the excess risk bounds. Such an analysis of the excess risk in regression highlights the interest of the proposed approach since it would be difficult to conduct for both volume sampling or the double phase algorithms. Future work in this direction includes investigating the importance of the sparsity of the $k$-leverage scores on the performance of other learning algorithms such as spectral clustering or support vector machines.

Finally, in our experimental section, we used an adhoc randomized algorithm inspired by \citep{FMPS11} to sample toy datasets with a prescribed profile of $k$-leverage scores. An interesting question would be to characterize the distribution of the output of our algorithm. In particular, sampling from the uniform measure on the set of symmetric matrices with prescribed spectrum and leverage scores is still an open problem \citep{DhHeSuTr05}. 

\subsection*{Acknowledgments}
AB and RB acknowledge support from ANR grant BoB (ANR-16-CE23-0003), and all authors acknowledge support from ANR grant BNPSI (ANR-13-BS03-0006).

\vskip 0.2in
\bibliography{bibliography.bib}

\newpage
\appendix

\section{Another interpretation of the $k$-leverage scores}
\label{app:statisticalInterpretationOfLVSs}
For $i \in [d]$, the SVD of $\bm{X}$ yields
\begin{equation}
\bm{X}_{:,i} = \sum_{\ell = 1}^{r}V_{i,\ell}\bm{f}_{\ell},
\end{equation}
where $\bm{f}_{\ell} = \sigma_{\ell}\bm{U}_{:,\ell}$, $\ell\in[r]$, are orthogonal.
Thus
\begin{equation}
\bm{X}_{:,i}^{\Tran}\bm{f}_{j} = V_{i,j} \|\bm{f}_{j}\|^{2} = V_{i,j} \sigma_{j}^{2}.
\label{e:columnTool}
\end{equation}
Then
\begin{equation}
 \frac{V_{i,j}}{\|\bm{X}_{:,i}\|} = \frac{\bm{X}_{:,i}^{\Tran}\bm{f}_{j}}{\sigma_{j}\|\bm{X}_{:,i}\|\|\bm{f}_{j}\|} =:  \frac{\cos \eta_{i,j}}{\sigma_{j}},
\end{equation}
where $\eta_{i,j}\in[0,\pi/2]$ is the angle formed by $\bm{X}_{:,i}$ and $\bm{f}_j$. Finally, \eqref{e:columnTool} also yields
\begin{equation}
\ell^{k}_{i} = \|\bm{X}_{:,i}\|^{2} \sum_{j=1}^{k} \frac{\cos^2\eta_{i,j}}{\sigma_{j}^{2}}.
\end{equation}
Compared to the length-square distribution in Section~\ref{subsec:length_square_sampling}, $k$-leverage scores thus favour columns that are aligned with the principal features. The weight $1/\sigma_j^2$ corrects the fact that features associated with large singular values are typically aligned with more columns. One could also imagine more arbitrary weights $w_j/\sigma_j^2$ in lieu of $1/\sigma_j^2$, or, equivalently, modified $k$-leverage scores
$$\ell_i^k(\bm{w}) = \sum_{j=1}^k w_{j}V_{i,j}^2.$$
However, the projection DPP with marginal kernel $\bm{K} = \bm{V}^{}_{k}\bm{V}_{k}^{\Tran}$ that we study in this paper is invariant to such reweightings. Indeed, for any $S\subset [d]$ of cardinality $k$,
\begin{equation}
\Det \left[\bm{V}^{}_{S, [k]} \,\text{Diag}(\bm{w}_{[k]})\, \bm{V}_{[k],S}^{\Tran}\right] = \Det(\bm{V}_{S,[k]})^{2} \prod\limits_{j \in [k]} w_{j}^{2}\propto \Det(\bm{V}_{S,[k]})^{2}.
\end{equation}
Such a scaling is thus not a free parameter in $\bm{K}$.

\section{Majorization and Schur convexity}
\label{app:majorization}

This section recalls some definitions and results from the theory of majorization and the notions of Schur-convexity and Schur-concavity. We refer to \citep{MaOlAr11} for further details. In this section, a subset $\mathcal{D} \subset \mathbb{R}^{d}$ is a symmetric domain if $\mathcal{D}$ is stable under coordinate permutations. Furthermore, a function $f$ defined on a symmetric domain $\mathcal{D}$ is called symmetric if it is stable under coordinate permutations.

\begin{definition}\label{def:majorization}
Let $\bm{p},\bm{q} \in \mathbb{R}_{+}^{d}$. $\bm{p}$ is said to majorize $\bm{q}$ according to Schur order and we note $\bm{q} \prec_{S} \bm{p}$ if
\begin{equation}
\left\{
    \begin{array}{ll}
        q_{i_{1}} \leq p_{j_{1}} \\
        q_{i_{1}} + q_{i_{2}} \leq p_{j_{1}} + p_{j_{2}} \\
        ... \\
        \sum\limits_{k=1}^{d-1} q_{i_{k}} \leq \sum\limits_{k=1}^{d-1} p_{j_{k}}\\
        \sum\limits_{k=1}^{d} q_{i_{k}} = \sum\limits_{k=1}^{d} p_{j_{k}}
    \end{array}
\right.
\end{equation}
where $\bm{p},\bm{q}$ are reordered so that $p_{i_{d}} \leq ... \leq p_{i_{1}}$ and $q_{j_{d}} \leq ... \leq q_{j_{1}}$.
\end{definition}
The majorization order has an algebraic characterization using doubly stochastic matrices first proven by Hardy, Littlewood, and Polya in 1929.
\begin{proposition}[Theorem B.2. \citealp{MaOlAr11}]
The vector $\bm{p}$ majorizes the vector $\bm{q}$ if and only if there exists a $d \times d$ doubly stochastic matrix $\Pi$ such that $\bm{q} = \bm{p \Pi}$.
\end{proposition}
\begin{example}
Let $\bm{p} = (3,0,0)$ and $\bm{q} = (1,1,1)$.
We check easily that $\bm{p}$ majorizes $\bm{q}$. Note that we can 'redistribute' $\bm{p}$ over $\bm{q}$ as follows: $\bm{q} =\frac{1}{3} \bm{J} \bm{p}$, where $\bm{J}$ is a $3 \times 3$ matrix of ones. The matrix $\bm{\Pi}=\frac{1}{3}\bm{J}$ is a doubly stochastic matrix.
\end{example}
Schur order compares two vectors using multiple inequalities. To avoid such cumbersome calculations, a scalar metric of inequality in a vector is desired. This is possible using the notion of Schur-convex/concave function.

\begin{definition}
Let f be a function on a symmetric domain $\mathcal{D} \subset \mathbb{R}_{+}^{d}$.\\
f is said to be Schur convex if
\begin{equation}
    \forall \bm{p}, \bm{q} \in \mathbb{R}_{+}^{d}, \bm{q} \prec_{S} \bm{p} \implies f(\bm{q}) \leq f(\bm{p}).
\end{equation}
f is said to be Schur concave if
\begin{equation}
    \forall \bm{p}, \bm{q} \in \mathbb{R}_{+}^{d}, \bm{q} \prec_{S} \bm{p} \implies f(\bm{q}) \geq f(\bm{p}).
\end{equation}
\end{definition}

\begin{proposition}[Theorem A.3, \citealp{MaOlAr11}]\label{schur_order_partial_property}
Let f be a symmetric function defined on $\mathbb{R}_{+}^{d}$, let $\mathcal{D}$ be a permutation-symmetric domain in $\mathbb{R}_{+}^{d}$ and suppose that
\begin{equation}
\forall x_{i},x_{j} \in \mathbb{R}_{+},   (x_{i} - x_{j}) (\frac{\partial f}{\partial x_{i}} - \frac{\partial f}{\partial x_{j}}) >0 
\end{equation}
then
\begin{equation}
    \forall \bm{p}, \bm{q} \in \mathcal{D}, \bm{q} \prec_{S} \bm{p} \implies f(\bm{q}) \leq f(\bm{p}),
\end{equation}
and $f$ is Schur convex.
\end{proposition}
We get a similar result for Schur concavity by switching the orders in the previous proposition.

\begin{theorem}[Theorem 3.1, \citealp{GuSi12}]\label{thm:schur_convex_volume_sampling}
Let $\bm{X} \in \mathbb{R}^{N \times d}$, and let $\bm{\sigma}\in \mathbb{R}^{d}$ the vector containing the squares of the singular values of $\bm{X}$. The function
\begin{equation}
\bm{\sigma} \mapsto \mathbb{E}_{\VS} \|\bm{X}- \Pi_{S} \bm{X}\|_{\Fr}^{2} = (k+1)\frac{e_{k}(\bm{\sigma})}{e_{k-1}(\bm{\sigma})}
\end{equation}
is Schur-concave.
\end{theorem}





\section{Principal angles and the Cosine Sine decomposition}
\label{app:principal_angles}

\FloatBarrier
\subsection{Principal angles}
\FloatBarrier
This section surveys the notion of principal angles between subspaces, see \cite[Section 6.4.3]{GoVa96} for details.
\begin{definition}
  \label{d:angles}
Let $\mathcal{P},\mathcal{Q}$ be two subspaces in $\mathbb{R}^{d}$. Let $p= \dim\mathcal{P}$ and $q = \dim\mathcal{Q}$ and assume that $q \leq p$. To define the vector of principal angles $\bm{\theta} \in [0,\pi/2]^{q}$ between $\mathcal{P}$ and $\mathcal{Q}$, let
\begin{equation}\label{eq:max_def_principal_angle_1}
 \cos(\theta_{1}) = \max \left\{ \frac{\bm{x}^{T}\bm{y}}{\|\bm{x}\|\|\bm{y}\|}; \quad \bm{x} \in \mathcal{P}, \bm{y} \in \mathcal{Q} \right\}
\end{equation}
be the cosine of the smallest angle between a vector of $\mathcal{P}$ and a vector of $\mathcal{Q}$, and let $(\bm{x}_{1},\bm{y}_{1}) \in \mathcal{P}\times \mathcal{Q}$ be a pair of vectors realizing the maximum. For $i \in [2,q]$, define successively
\begin{equation}\label{eq:max_def_principal_angles}
 \cos(\theta_{i}) = \max \left\{\frac{\bm{x}^{T}\bm{y}}{\|\bm{x}\|\|\bm{y}\|}; \quad \bm{x} \in \mathcal{P}, \bm{y} \in \mathcal{Q}; \bm{x} \perp \bm{x}_{j}, \bm{y} \perp \bm{y}_{j}\:, \forall j \in [1:i-1] \right\}
\end{equation}
and denote $(\bm{x}_{i},\bm{y}_{i}) \in \mathcal{P}\times\mathcal{Q}$ such that $\cos(\theta_{i}) = \bm{x}_{i}^{\Tran}\bm{y}_{i} \:$.
\end{definition}
Note that although the so-called principal vectors $(\bm{x}_{i},\bm{y}_{i})_{i \in [q]}$
are not uniquely defined by \eqref{eq:max_def_principal_angle_1} and \eqref{eq:max_def_principal_angles}, the principal angles $\bm{\theta}$ are uniquely defined, see \citep{BjGo73}. The following result confirms this, while also providing a way to compute $\bm{\theta}$.
\begin{proposition}[\citealp{BjGo73}, \citealp{Ben92}]
  \label{principal_angles_theorem_1}
Let $\mathcal{P}$ and $\mathcal{Q}$ and $\bm{\theta}$ be as in Definition~\ref{d:angles}. Let $\bm{P} \in \mathbb{R}^{d \times p}$, $\bm{Q} \in \mathbb{R}^{d \times q}$ be two orthogonal matrices, whose columns are orthonormal bases of $\mathcal{P}$ and $\mathcal{Q}$, respectively. Then
\begin{equation}
 \forall i \in [q], \quad \cos(\theta_{i}) =\sigma_i(\bm{Q}^{\Tran}\bm{P}).
\end{equation}
In particular
\begin{equation}\label{principal_angles_formula}
\Vol_{q}^{2}(\bm{Q}^{\Tran}\bm{P}) = \prod\limits_{i \in [q]} \cos^{2}(\theta_{i}).
\end{equation}
\end{proposition}
An important case for our work arises when $q=k$, $\bm{Q}=\bm{V} \in \mathbb{R}^{d \times k}$, and $\bm{P}=\bm{S}\in \mathbb{R}^{d \times k}$ is a sampling matrix. The left-hand side of \eqref{principal_angles_formula} then equals $\Det(\bm{V}_{:,S})^2$.



\subsection{The Cosine Sine decomposition}
The Cosine Sine (CS) decomposition is useful for the study of the relative position of two subspaces. It generalizes the notion of cosine, sine and tangent to subspaces.
\begin{proposition}[\citealp{GoVa96}]
Let $q \leq d/2$ and $\bm{Q} =
\left[
\begin{array}{c}
\bm{Q}_{1}  \\
\hline
\bm{Q}_{2}
\end{array}
\right]
$ be a $d\times q$ orthogonal matrix, where $\bm{Q}_{1} \in \mathbb{R}^{q \times q}$ and $\bm{Q}_{2} \in \mathbb{R}^{(d-q)\times q}$. Assume that $\bm{Q}_{1}$ is non singular, then there exist orthogonal matrices $\bm{Y}\in\mathbb{R}^{d \times q}$ and
\begin{equation}
 \bm{W} =
\left[
\begin{array}{c|c}
\bm{W}_{1} & \bm{0} \\
\hline
\bm{0} & \bm{W}_{2}
\end{array}
\right]~\in~\mathbb{R}^{d \times d},
\end{equation}
and a matrix
\begin{equation}
\bm{\Sigma}~=~\left[
\begin{array}{c}
\Cosmatrix \\
\hline
\Sinmatrix \\
\hline
\bm{0}
\end{array}
\right] \in \mathbb{R}^{d \times q},
\end{equation}
such that
\begin{equation}
    \bm{Q} = \bm{W}\bm{\Sigma}\bm{Y}^{T},
\end{equation}
where $\bm{W}_{1} \in \mathbb{R}^{q \times q}$ and $\bm{W}_{2} \in \mathbb{R}^{d-q \times d-q}$, and $\Cosmatrix, \Sinmatrix \in \mathbb{R}^{q\times q}$ are diagonal matrices satisfying the identity $\Cosmatrix^{2} + \Sinmatrix^{2} = \mathbb{I}_{q}$.
In particular, each block $\bm{Q}_{i}$ factorizes as
\begin{equation}
\begin{split}
    \bm{Q}_{1} = & \bm{W}_{1}\Cosmatrix\bm{Y}^{T} \\
    \bm{Q}_{2} = & \bm{W}_{2}\left[
\begin{array}{c}
\Sinmatrix \\
\hline
\bm{0}
\end{array}
\right]\bm{Y}^{T} .\\
\end{split}
\end{equation}

\label{CSD_proposition}
\end{proposition}
The CS decomposition is defined for every orthogonal matrix. \rb{Il y a quand même une hypothèse sur $Q_1$, non ? Also: we need to relate the matrices C and S to the principal angles in the previous subsection, and define $\theta_i(S)$ before the next Corollary} An important case is when $\bm{Q}$ is the product of an orthogonal matrix $\bm{V} \in \mathbb{R}^{d \times d}$ and a sampling matrix $\bm{S} \in \mathbb{R}^{d \times k}$, that is $\bm{Q} = \bm{V}^{\Tran}\bm{S}$.

\begin{corollary}
Let $\bm{V} \in \mathbb{R}^{d \times d}$ be an orthogonal matrix and $\bm{S} \in \mathbb{R}^{d \times k}$ be a sampling matrix. Let
\begin{equation}
\bm{Q} = \bm{V}^{\Tran}\bm{S} = \left[
\begin{array}{c}
\bm{V}_{k}^{\Tran} \bm{S}\\
\hline
\bm{V}_{d-k}^{\Tran} \bm{S}\\
\end{array}
\right]
\end{equation}
be a $d \times k$ orthogonal matrix, with $\Det(\bm{V}_{k}^{\Tran} \bm{S})^{2} > 0$. Let further $\bm{Z}_{S} = \bm{V}_{d-k}^{\Tran} \bm{S}(\bm{V}_{k}^{\Tran}\bm{S})^{-1}$. Then
\begin{equation}\label{eq:trace_tan_relationship}
\Tr(\bm{Z}_{S}^{}\bm{Z}_{S}^{\Tran}) \leq \sum\limits_{i \in [k]} \tan^{2}(\theta_{i}(S)).
\end{equation}
\end{corollary}

\begin{proof}
In the case $k \leq d/2$, Proposition~\ref{CSD_proposition} applied to the matrix $\bm{Q} = \bm{V}^{\Tran}\bm{S}$ with $\bm{Q}_{1} = \bm{V}_{k}^{\Tran}\bm{S}$ and $\bm{Q}_{2} = \bm{V}_{d-k}^{\Tran} \bm{S}$ yields
\begin{align}\label{eq:common_factorization_Q_1_Q_2}
    \bm{Q}_{1} = & \bm{W}_{1}\Cosmatrix\bm{Y}^{T} \\
    \bm{Q}_{2} = & \bm{W}_{2}\left[
\begin{array}{c}
\Sinmatrix \\
\hline
\bm{0}
\end{array}
\right]\bm{Y}^{T}.
\end{align}
Thus, the diagonal matrix $\Cosmatrix$ contains the singular values of the matrix $\bm{V}_{k}^{\Tran}\bm{S}$ that are cosines of the principal angles $(\theta_{i}(S))_{i \in [k]}$ between $\Span(\bm{V}_{k})$ and $\Span(\bm{S})$ thanks to Proposition~\ref{principal_angles_theorem_1}.\\
The identity $\Cosmatrix^{2} + \Sinmatrix^{2} = \mathbb{I}_{k}$ and the fact that $\theta_{i}(S) \in [0, \frac{\pi}{2} ]$ imply  that the (diagonal) elements of $\Sinmatrix$ are equal to the sines of the principal angles between $\Span(\bm{V}_{k})$ and $\Span(\bm{S})$. Let $\Tanmatrix = \Sinmatrix \Cosmatrix^{-1}$. $\Tanmatrix \in \mathbb{R}^{k \times k}$ is a diagonal matrix containing the tangents of the principal angles $(\theta_{i}(S))_{i \in [k]}$. Using \eqref{eq:common_factorization_Q_1_Q_2}, we get
\begin{equation}
\bm{Z}_{S} = \bm{V}_{d-k}^{\Tran} \bm{S}(\bm{V}_{k}^{\Tran}\bm{S})^{-1} = \bm{W}_{2}\left[
\begin{array}{c}
\Sinmatrix \\
\hline
\bm{0}
\end{array}
\right] \bm{Y}^{\Tran}\bm{Y} \Cosmatrix^{-1} \bm{W}_{1}^{\Tran} = \bm{W}_{2}\left[
\begin{array}{c}
\Sinmatrix \\
\hline
\bm{0}
\end{array}
\right] \Cosmatrix^{-1} \bm{W}_{1}^{\Tran}  = \bm{W}_{2}\left[
\begin{array}{c}
\Sinmatrix \Cosmatrix^{-1}\\
\hline
\bm{0}
\end{array}
\right] \bm{W}_{1}^{\Tran}.
\end{equation}
Then,
\begin{equation}
\Tr(\bm{Z}_{S}^{}\bm{Z}_{S}^{\Tran}) = \Tr(\bm{W}_{2}\left[
\begin{array}{c|c}
\Tanmatrix^{2} & \bm{0}\\
\hline
\bm{0} & \bm{0}
\end{array}
\right] \bm{W}_{2}^{\Tran})=  \sum\limits_{i \in [k]} \tan^{2}(\theta_{i}(S)).
\end{equation}
\end{proof}

\section{Proofs}
\label{app:proofs}

\subsection{Technical lemmas}
We start with two useful emmas borrowed from the literature.
\begin{lemma}[Lemma 3.1, \citealp{BoDrMI11}]\label{refined_analysis_of_approximation_bound}
Let $S \subset [d]$, then

\begin{equation}
\| \bm{X} - \Pi_{S,k}^{\nu} \bm{X} \|_{\nu}^{2}  \leq  \| \bm{E}(\bm{I}-\bm{P}_{S})\|_{\nu}^{2}, \quad \nu \in \{2,\Fr\},
\end{equation}
where   $\bm{E} = \bm{X} - \Pi_{k}\bm{X}$ and $\bm{P}_{S} = \bm{S}(\bm{V}_{k}^{\Tran}\bm{S})^{-1}\bm{V}_{k}^{\Tran}$.
Furthermore,
\begin{equation}
\| \bm{X} - \Pi_{S,k}^{\nu} \bm{X} \|_{\nu}^{2} \leq  \frac{1}{\sigma_{k}^{2}(\bm{V}_{S,[k]})} \| \bm{X} - \Pi_{k}\bm{X}\|_{\nu}^{2} , \quad \nu \in \{2,\Fr\}.
\end{equation}
\end{lemma}
The following lemma was first proven by \citealp{DRVW06}, and later rephrased.
\begin{lemma}[Lemma 11, \citealp{DeRa10}]\label{minors_symmetric_polynomials_lemma}
Let $\bm{V} \in \mathbb{R}^{k \times d}$, $r = \rank(\bm{V})$ and $\ell \in [1:r]$. Then
\begin{equation}
\sum\limits_{S \subset [d], |S| = \ell} e_{\ell}(\Sigma(\bm{V}_{:,S})^{2}) = e_{\ell}(\Sigma(\bm{V})^{2})
\end{equation}
where $e_{\ell}$ is the $\ell$-th elementary symmetric polynomial on $r$ variables, see Section~\ref{s:notation}.
\end{lemma}
Elementary symmetric polynomials play an important role in the proof of Proposition~\ref{prop:p_eff_proposition}, in particular their interplay with the Schur order; see Appendix \ref{app:majorization} for definitions.

\begin{lemma}\label{symmetric_schur_convex_lemma}
Let $\phi, \psi:\mathbb{R}_{+}^{d} \rightarrow \mathbb{R}_{+}$ be defined by
\begin{equation}
\phi :  \bm{\sigma} \mapsto \frac{\displaystyle e_{k-1}(\bm{\sigma})}{\displaystyle e_{k}(\bm{\sigma})}
\end{equation}
and
\begin{equation}
\psi :  \bm{\sigma} \mapsto \displaystyle e_{k}(\bm{\sigma}).
\end{equation}
Then both functions are symmetric, $\phi$ is Schur-convex, and $\psi$ is Schur-concave.
\end{lemma}
\begin{proof}[of Lemma~\ref{symmetric_schur_convex_lemma}]
Let $i,j \in [r], i \neq j$. Let $\sigma_{i},\sigma_{j} \in \mathbb{R}_{+}$, it holds
\begin{align*}
    (\sigma_{i} - \sigma_{j})(\partial_{i}\phi(\bm{\sigma})-\partial_{j}\phi(\bm{\sigma})) & =  (\sigma_{i} - \sigma_{j})(-\frac{1}{\sigma_{i}^{2}} +\frac{1}{\sigma_{j}^{2}}) \\
    & = \frac{(\sigma_{i}-\sigma_{j})^{2}(\sigma_{i} +\sigma_{j})}{\sigma_{i}^{2}\sigma_{j}^{2}} \geq 0,
\end{align*}
so that $\phi$ is Schur-convex by Proposition~\ref{schur_order_partial_property}. Similarly,
\begin{align*}
    (\sigma_{i} - \sigma_{j})(\partial_{i}\psi(\bm{\sigma})-\partial_{j}\psi(\bm{\sigma})) & =  (\sigma_{i} - \sigma_{j})(\prod_{\ell \neq i}\sigma_{\ell} - \prod_{\ell \neq j}\sigma_{\ell}) \\
    & = -(\sigma_{i}-\sigma_{j})^{2}\prod_{\ell \neq i,j}\sigma_{\ell} \geq 0, 
\end{align*}
so that $\psi$ is Schur-concave by Proposition~\ref{schur_order_partial_property}.
\end{proof}

Elementary symmetric polynomials also interact nicely with ``marginalizing" sums.
\begin{lemma}\label{sum_k_1_symmetric_poly_det_lemma}
Let $\bm{V}$ be a real $k \times d$ matrix and let $ r = \rank(\bm{V})$. Denote by $p$ the number of non zero columns of $\bm{V}$. Then for all $k\leq r+1$,
\begin{equation}
    \sum\limits_{\substack{S \subset [d], |S| = k\\\Vol_{k}(\bm{V}_{:,S})^{2} >0}} \quad \sum\limits_{\substack{T \subset [S]\\ |T| = k-1}} e_{k-1}(\Sigma(\bm{V}_{:,T})^{2}) \leq (p-k+1)e_{k-1}(\Sigma(\bm{V})^{2}).
\end{equation}
A fortiori,
\begin{equation}
    \sum\limits_{\substack{S \subset [d], |S| = k\\\Vol_{k}(\bm{V}_{:,S})^{2} >0}} \quad \sum\limits_{\substack{T \subset [S]\\ |T| = k-1}} e_{k-1}(\Sigma(\bm{V}_{:,T})^{2}) \leq (d-k+1)e_{k-1}(\Sigma(\bm{V})^{2}).
\end{equation}
\end{lemma}

\begin{proof}[of Lemma~\ref{sum_k_1_symmetric_poly_det_lemma}]
For $T \subset [d], \: |T| = k-1$,
\begin{align*}
    \Omega_{1}(T) &= \left\{ S \subset[d]:\, |S| = k, T \subset S,\ \forall i \in S ,\ \bm{V}_{:,i} \neq \bm{0} \right\}\\
    \Omega_{2}(T) &= \left\{ S \subset[d] :\, |S| = k, T \subset S, \Vol_{k}(\bm{V}_{:,S})^{2} >0   \right\}.
\end{align*}
Note that $\Omega_{2}(T) \subset \Omega_{1}(T)$ so that
\begin{align*}
        \sum\limits_{\substack{S \subset [d], |S| = k\\\Vol_{k}(\bm{V}_{:,S})^{2} >0}} \quad \sum\limits_{\substack{T \subset S\\ |T| = k-1}} e_{k-1}(\Sigma(\bm{V}_{:,T})^{2})
        & = \sum\limits_{\substack{T \subset [d]\\ |T| = k-1}} \quad \sum\limits_{S \in \Omega_{2}(T)} e_{k-1}(\Sigma(\bm{V}_{:,T})^{2}) \\
        & \leq \sum\limits_{\substack{T \subset [d]\\ |T| = k-1}} \quad \sum\limits_{S \in \Omega_{1}(T)} e_{k-1}(\Sigma(\bm{V}_{:,T})^{2}).
\end{align*}
The set $\Omega_{1}(T)$ has at most $(p-k+1)$ elements so that
\begin{equation}
\sum\limits_{\substack{T \subset [d]\\ |T| = k-1}} \quad \sum\limits_{S \in \Omega_{1}(T)} e_{k-1}(\Sigma(\bm{V}_{:,T})^{2}) \leq (p-k+1)\sum\limits_{\substack{T \subset [d]\\ |T| = k-1}}  e_{k-1}(\Sigma(\bm{V}_{:,T})^{2}).
\end{equation}
Lemma~\ref{minors_symmetric_polynomials_lemma} for $\ell=k-1$ further yields
\begin{equation}
(p-k+1)\sum\limits_{\substack{T \subset [d]\\ |T| = k-1}}  e_{k-1}(\Sigma(\bm{V}_{:,T})^{2}) \leq (p-k+1)\,e_{k-1}(\Sigma(\bm{V})^{2}).
\end{equation}
\end{proof}

\subsection{Proof of Proposition~\ref{projection_dpp_theorem}}
First, Lemma~\ref{refined_analysis_of_approximation_bound} yields
	\begin{align}
	\sum\limits_{S \subset [d], |S| = k} \Det(\bm{V}_{S,[k]})^{2}\| \bm{X} - \Pi_{S}^{\nu}\bm{X} \|_{\nu}^{2} & \leq  \sum\limits_{S 	\subset [d], |S| = k} \frac{1}{\sigma_{k}^{2}(\bm{V}_{S,[k]})}\Det(\bm{V}_{S,[k]})^{2} \: \|\bm{X} - \Pi_{k}\bm{X}\|_{\nu}^{2}  \nonumber\\
 	& =  \|\bm{X} - \Pi_{k}\bm{X}\|_{\nu}^{2} \sum\limits_{S \subset [d], |S| = k} \prod_{\ell =1}^{k-1}\sigma_{\ell}^{2}(\bm{V}_{S,[k]}),
 \label{eq:begin_proof_theo16}
\end{align}
where the last equality follows from
\begin{equation}
	\Det(\bm{V}_{S,[k]})^{2} = \prod_{\ell =1}^{k}\sigma_{\ell}^{2}(\bm{V}_{S,[k]}).
\end{equation}
By definition of the polynomial $e_{k-1}$, it further holds
\begin{equation}
	\prod_{\ell =1}^{k-1}\sigma_{\ell}^{2}(\bm{V}_{S,[k]}) \leq e_{k-1}(\Sigma(\bm{V}_{S,[k]})^{2}),
\end{equation}
so that \eqref{eq:begin_proof_theo16} leads to
\begin{align}
\sum\limits_{S \subset [d], |S| = k} \Det(\bm{V}_{S,[k]})^{2}\| \bm{X} - \Pi_{S}^{\nu}\bm{X} \|_{\nu}^{2} & \leq  \|\bm{X} - \Pi_{k}\bm{X}\|_{\nu}^{2} \sum\limits_{S \subset [d], |S| = k} e_{k-1}(\Sigma(\bm{V}_{S,[k]})^{2}).
\label{eq:middle_proof_theo16}
\end{align}

Now, Lemma~\ref{minors_symmetric_polynomials_lemma} applied to the matrix $\bm{V}^{\Tran}_{S,[k]}$ gives
\begin{equation}
	e_{k-1}(\Sigma(\bm{V}_{S,[k]})^{2}) = \sum\limits_{T \subset S, |T| = k-1} e_{k-1}(\Sigma(\bm{V}_{T,[k]})^{2}),
\end{equation}
Therefore, Lemma~\ref{sum_k_1_symmetric_poly_det_lemma} yields
\begin{equation}
\label{eq:const_dpluskmoins1}
\begin{split}
\sum\limits_{S \subset [d], |S| = k} e_{k-1}(\Sigma(\bm{V}_{S,[k]})^{2})
  & \leq (d-k+1)\sum\limits_{T \subset [d], |T| = k-1} e_{k-1}(\Sigma(\bm{V}_{T,[k]})^{2}).\\
\end{split}
\end{equation}
Using Lemma~\ref{minors_symmetric_polynomials_lemma} and the fact that $\bm{V}_{k}$ is orthogonal, we finally write
\begin{equation}
\label{eq:const_k}
\sum\limits_{T \subset [d], |T| = k-1} e_{k-1}(\Sigma(\bm{V}_{T,[k]})^{2}) = e_{k-1}(\Sigma(\bm{V}_{k})^{2}) = k.
\end{equation}
Plugging \eqref{eq:const_k} into \eqref{eq:const_dpluskmoins1}, and then into (\ref{eq:middle_proof_theo16}) concludes the proof of Proposition~\ref{projection_dpp_theorem}.

\subsection{Proof of Proposition \ref{hypo_one_two_proposition}}
\label{s:proofOfExactSparsitySetting}
We first prove the Frobenius norm bound, which requires more work. The spectral bound is easier and uses a subset of the arguments for the Frobenius norm.

\subsubsection{Frobenius norm bound}
\label{s:frobNormThm19}
Recall that $\bm{E} = \bm{X} - \Pi_{k}\bm{X}$.
We start with Lemma~\ref{refined_analysis_of_approximation_bound}:
\begin{equation}
\label{eq:begin_proof_prop17}
	\begin{split}
		\| \bm{X} - \Pi_{S}^{\Fr}\bm{X} \|_{\Fr}^{2}  & \leq  \| \bm{E}(\bm{I}-\bm{P}_{S})\|_{\Fr}^{2}\\
		& \leq \| \bm{E}\|_{\Fr}^{2} + \Tr(\bm{E}^{\Tran}\bm{E}\bm{P}_{S}\bm{P}_{S}^{\Tran}) - 2\Tr(\bm{P}_{S}^{\Tran}		\bm{E}^{\Tran}\bm{E}).
	\end{split}
\end{equation}
Since $\bm{E}^{\Tran}\bm{E} = \bm{V}_{r-k}^{\phantom{\Tran}}\bm{\Sigma}_{r-k}^{2}\bm{V}_{r-k}^{\Tran}$ and $\bm{P}_{S} = \bm{S}(\bm{V}_{k}^{\Tran}\bm{S})^{-1}\bm{V}_{k}^{\Tran}$,
\begin{equation}
        \begin{split}
        		\Tr(\bm{P}_{S}^{\Tran}\bm{E}^{\Tran}\bm{E})  & =  \Tr \bigg(\bm{V}_{k}^{\phantom{\Tran}}((\bm{V}_{k}^{\Tran}		\bm{S})^{\Tran})^{-1}\bm{S}^{\Tran}\bm{V}_{r-k}^{\phantom{\Tran}}\bm{\Sigma}_{r-k}^{\phantom{\Tran}}\bm{V}		_{r-k}^{\Tran} \bigg)  \\
        		& = \Tr \bigg( \bm{V}_{r-k}^{\Tran}\bm{V}_{k}^{\phantom{\Tran}}((\bm{V}_{k}^{\Tran}\bm{S})^{\Tran})^{-1}\bm{S}		^{\Tran}\bm{V}_{r-k}^{\phantom{\Tran}}\bm{\Sigma}_{r-k}^{\phantom{\Tran}} \bigg)  \\
      		  & = 0,
        \end{split}
\end{equation}
where the last equality follows from $\bm{V}_{r-k}^{\Tran}\bm{V}_{k} = \bm{0}$.
Therefore, (\ref{eq:begin_proof_prop17}) becomes
\begin{equation}
	\| \bm{X} - \Pi_{S}^{\Fr}\bm{X} \|_{\Fr}^{2}  \leq \| \bm{E}\|_{\Fr}^{2} + \Tr(\bm{E}^{\Tran}\bm{E}\bm{P}^{\phantom{\Tran}}	_{S}\bm{P}^{\Tran}_{S}).
\end{equation}
Taking expectations,
\begin{equation}
  \label{eq:inequality_approximation_error_by_EEPP}
		\EX_{\DPP} \| \bm{X} - \Pi_{S}^{\Fr}\bm{X} \|_{\Fr}^{2} \leq \|\bm{E}\|_{\Fr}^{2} + \sum_{S \subset [d], |S| = k}	\Det(\bm{V}_{S,[k]})^{2} \Tr(\bm{E}^{\Tran}\bm{E}\bm{P}^{\phantom{\Tran}}_{S}\bm{P}^{\Tran}_{S}).
\end{equation}
Proposition~\ref{principal_angles_theorem_1} expresses $\Det(\bm{V}_{S,[k]})^{2}$ as a function of the principal angles $(\theta_i(S))$ between $\Span(\bm{V}_k)$ and $\Span(\bm{S})$, namely
\begin{equation}\label{eq:det_cos_identity}
	\Det(\bm{V}_{S,[k]})^{2} = \prod\limits_{i \in [k]} \cos^{2}(\theta_{i}(S)).
\end{equation}
 The remainder of the proof is in two steps. First, we bound the second factor in the sum in the right-hand side of \eqref{eq:inequality_approximation_error_by_EEPP} with a similar geometric expression. This allows trigonometric manipulations. Second, we work our way back to elementary symmetric polynomials of spectra, and we conclude after some simple algebra.

First, for $S \subset [d], |S| =k$, let
$$
\bm{Z}_{S} = \bm{V}_{d-k}^{\Tran} \bm{S}(\bm{V}_{k}^{\Tran}\bm{S})^{-1} = \bm{V}_{d-k}^{\Tran} \bm{P}_{S}^{\phantom{\Tran}} \bm{V}_{k}.
$$
It allows us to write
\begin{equation}
	\Tr(\bm{E}^{\Tran}\bm{E}\bm{P}^{\phantom{\Tran}}_{S}\bm{P}^{\Tran}_{S})= \Tr(\bm{V}_{d-k}\bm{\Sigma}_{d-k}		^{2}\bm{V}_{d-k}^{\Tran}\bm{P}^{\phantom{\Tran}}_{S}\bm{P}^{\Tran}_{S}) = \Tr(\bm{\Sigma}_{r-k}^{2}\bm{Z}_{S}	^{\phantom{\Tran}}\bm{Z}_{S}^{\Tran}).
\end{equation}
However, for real symmetric matrices $\bm{A}$ and $\bm{B}$ with the same size, a simple diagonalization argument yields
\begin{equation}
	\Tr(\bm{A}\bm{B}) \leq \|\bm{A}\|_{2}\Tr(\bm{B}),
\end{equation}
so that
\begin{equation}\label{eq:majoration_by_sigma_k_1}
	\Tr(\bm{E}^{\Tran}\bm{E}\bm{P}^{\phantom{\Tran}}_{S}\bm{P}^{\Tran}_{S}) = \Tr(\bm{\Sigma}_{r-k}^{2}\bm{Z}_{S}	^{\phantom{\Tran}}\bm{Z}_{S}^{\Tran}) \leq \bm{\sigma}_{k+1}^{2} \Tr(\bm{Z}_{S}^{\phantom{\Tran}}\bm{Z}_{S}^{\Tran}).
\end{equation}
In Appendix~\ref{app:principal_angles}, we characterize $\Tr(\bm{Z}_{S}^{\phantom{\Tran}}\bm{Z}_{S}^{\Tran})$ using principal angles, see \eqref{eq:trace_tan_relationship}. This reads
\begin{equation}\label{eq:trace_tan_identity}
 	\Tr( \bm{Z}_{S}^{\phantom{\Tran}}\bm{Z}_{S}^{\Tran}) = \sum_{j \in [k]}\tan^{2}(\theta_{j}(S)).
\end{equation}
Combining \eqref{eq:inequality_approximation_error_by_EEPP}, \eqref{eq:majoration_by_sigma_k_1}, \eqref{eq:det_cos_identity}, and \eqref{eq:trace_tan_identity}, we obtain the following intermediate bound
\begin{equation}\label{eq:sum_product_cos_tan}
 	\EX_{\DPP} \| \bm{X} - \Pi_{S}^{\Fr}\bm{X} \|_{\Fr}^{2} \leq \| \bm{E}\|_{\Fr}^{2} + \bm{\sigma}_{k+1}^{2} \sum_{S \subset 	[d], |S| = k} \,\left[\prod\limits_{i \in [k]} \cos^{2}(\theta_{i}(S))\right] \left[\sum_{j \in [k]}\tan^{2}(\theta_{j}(S))\right].
\end{equation}
Distributing the sum and using trigonometric identities, the general term of the sum in \eqref{eq:sum_product_cos_tan} becomes
\begin{align}
    	\left[\prod\limits_{i \in [k]} \cos^{2}(\theta_{i}(S))\right] \left[\sum_{j \in [k]}\tan^{2}(\theta_{j}(S))\right] & =  \sum_{i \in [k]}(1-\cos^{2}(\theta_{i}(S)))\prod_{j \in [k], j 	\neq i} \cos^{2}(\theta_{j}(S)) \nonumber \\
    	& = \sum_{i \in [k]}\prod_{j \in [k], j \neq i} \cos^{2}(\theta_{j}(S)) - \sum_{i \in [k]}\prod_{j \in [k]}\cos^{2}(\theta_{j}(S)).
	\label{eq:prod_cos_tan_ineq}
\end{align}
The $(\cos(\theta_{j}(S)))_{j \in [k]}$ are the singular values of the matrix $\bm{V}_{S,[k]}$ so that
\begin{equation}\label{eq:cos_k_1_symmetric_polynomial_identity}
	\sum_{i \in [k]}\prod_{j \in [k], j \neq i} \cos^{2}(\theta_{j}(S)) = e_{k-1}(\Sigma(\bm{V}_{S,[k]})^{2}),
\end{equation}
and
\begin{equation}\label{eq:cos_k_symmetric_polynomial_identity}
	\prod_{j \in [k]}\cos^{2}(\theta_{j}(S)) = e_{k}(\Sigma(\bm{V}_{S,[k]})^{2}).
\end{equation}
Back to \eqref{eq:prod_cos_tan_ineq}, one gets
\begin{align}
    \left[\prod\limits_{i \in [k]} \cos^{2}(\theta_{i}(S))\right] \left[\sum_{j \in [k]}\tan^{2}(\theta_{j}(S))\right]
    & = e_{k-1}(\Sigma(\bm{V}_{S,[k]})^{2}) - \sum_{i \in [k]} e_{k}(\Sigma(\bm{V}_{S,[k]})^{2}) \nonumber\\
    & = e_{k-1}(\Sigma(\bm{V}_{S,[k]})^{2}) -  k e_{k}(\Sigma(\bm{V}_{S,[k]})^{2}). \label{eq:sumcostan} 
\end{align}
Thus, plugging \eqref{eq:sumcostan} back into the intermediate bound \eqref{eq:sum_product_cos_tan}, it comes
\begin{align}\label{eq:bound_tan_in_expectation_under_sparsity}
    \EX_{\DPP} \| \bm{X} &- \Pi_{S}^{\Fr}\bm{X} \|_{\Fr}^{2}\nonumber\\
      &\leq \| \bm{E}\|_{\Fr}^{2} + \bm{\sigma}_{k+1}^{2} \left[\sum\limits_{\substack{S \subset [d]\\ |S| = k}} e_{k-1}(\Sigma(\bm{V}_{S,[k]})^{2}) -  k \sum_{\substack{S \subset [d]\\ |S| = k}}e_{k}(\Sigma(\bm{V}_{S,[k]})^{2})\right]\nonumber.\\
\end{align}
Using Lemma~\ref{minors_symmetric_polynomials_lemma} twice, it comes
\begin{align}
  \EX_{\DPP} \| \bm{X} &- \Pi_{S}^{\Fr}\bm{X} \|_{\Fr}^{2} \nonumber\\
  &\leq \| \bm{E}\|_{\Fr}^{2} + \bm{\sigma}_{k+1}^{2} \left[\sum_{\substack{S \subset [d]\\ |S| = k}} \,\sum_{\substack{T \subset S\\ |T| = k-1}} e_{k-1}(\Sigma(\bm{V}_{T,[k]})^{2}) -  ke_{k}(\Sigma(\bm{V}_{:,[k]})^{2})\right]
  \label{e:doubleSumTrick}.
\end{align}

Lemmas~\ref{sum_k_1_symmetric_poly_det_lemma} and the identities $e_{k-1}(\Sigma(\bm{V}_{:,[k]})^{2}) = k$ and $e_{k}(\Sigma(\bm{V}_{:,[k]})^{2}) = 1$ allow us to conclude
\begin{align}
  \EX_{\DPP} \| \bm{X} - \Pi_{S}^{\Fr}\bm{X} \|_{\Fr}^{2} &\leq \| \bm{E}\|_{\Fr}^{2} + \bm{\sigma}_{k+1}^{2} \left[ (p-k+1)e_{k-1}(\Sigma(\bm{V}_{:,[k]})^{2}) -  k \right]\\
  &=\| \bm{E}\|_{\Fr}^{2} + \bm{\sigma}_{k+1}^{2} (p-k)k.
\end{align}
By definition of $\beta$ \eqref{e:defbeta}, we have proven \eqref{eq:frobenius_bound_dpp}, i.e.,
\begin{align*}
    \EX_{\DPP} \| \bm{X} - \Pi_{S}^{\Fr}\bm{X} \|_{\Fr}^{2}
  & \leq \| \bm{E}\|_{\Fr}^{2} \left(1 + \beta \frac{p-k}{d-k}k \right).
\end{align*}

\subsubsection{Spectral norm bound}
The bound in spectral norm is easier to derive. We start from Lemma~\ref{refined_analysis_of_approximation_bound}:
\begin{align}
    \EX_{\DPP} \| \bm{X} - \Pi_{S}^{2}\bm{X} \|_{2}^{2} &= \sum_{S \subset [d], |S| = k}\Det(\bm{V}_{S,[k]})^{2}\| \bm{X} - \Pi_{S}\bm{X} \|_{2}^{2}\\
     & \leq \| \bm{E}\|_{2}^{2}\sum_{\substack{S \subset [d], |S| = k\\   \Det(\bm{V}_{S,[k]})^{2}>0}} \prod\limits_{\ell =1}^{k-1} \sigma_{\ell}^{2}(\bm{V}_{S,[k]})
\end{align}
By definition of $e_{k-1}$, it comes
\begin{align*}
    \EX_{\DPP} \| \bm{X} - \Pi_{S}^{2}\bm{X} \|_{2}^{2} & \leq  \| \bm{E}\|_{2}^{2}\sum_{\substack{S \subset [d], |S| = k\\ \Det(\bm{V}_{S,[k]})^{2}>0}} e_{k-1}(\Sigma(\bm{V}_{S,[k]})^{2}) \\
    & \leq (p-k+1)\,e_{k-1}(\Sigma(\bm{V}_{:,[k]})^{2})\,\| \bm{E}\|_{2}^{2}\\
    &= (p-k+1)\,k\,\| \bm{E}\|_{2}^{2},
\end{align*}
where we again used the double sum trick of \eqref{e:doubleSumTrick} and Lemma~\ref{sum_k_1_symmetric_poly_det_lemma}.

\subsection{Proof of Theorem~\ref{prop:p_eff_proposition}}
\label{s:proofOfEffectiveSparsitySetting}
We start with a lemma on evaluations of elementary symmetric polynomials on specific sequences. \rb{Add a sentence to describe the main steps of the proof}
\begin{lemma}\label{lemma_k_minus_half}
Let $\bm{\lambda}\in [0,1]^{k}$ such that
\begin{equation}
	\left\{
	\begin{array}{l}
	    \lambda_{1} \geq \dots \geq \lambda_{k},\\
		\label{eq:lambda_schur_hypo}
   		\Lambda = \sum\limits_{i=1}^{k} \lambda_{i} \geq k-1 + \frac{1}{\theta}.
	\end{array}
	\right.
\end{equation}
Then, with the functions $\phi,\psi$ introduced in Lemma~\ref{symmetric_schur_convex_lemma}, 
\begin{equation}
    \left\{
    \begin{array}{ll}
    	\psi(\bm{\lambda})  & \displaystyle{\geq \frac{1}{\theta},}\\
	\phi(\bm{\lambda})  & \leq k-1 +\theta.
    \end{array}
    \right.
\end{equation}


\end{lemma}

\begin{proof}
Let $\hat{\bm{\lambda}} = (1,...,1,\Lambda -k +1) \in \mathbb{R}^{k}$. Then
\begin{equation}
\left\{
    \begin{array}{ll}
        \lambda_{1} \leq \hat{\lambda}_{1} \\
        \lambda_{1} + \lambda_{2} \leq \hat{\lambda}_{1} + \hat{\lambda}_{2} \\
        ... \\
        \sum\limits_{i=1}^{k-1} \lambda_{i} \leq \sum\limits_{i=1}^{k-1} \hat{\lambda}_{i}\\
        \sum\limits_{i=1}^{k} \lambda_{i} = \sum\limits_{i=1}^{k} \hat{\lambda}_{i}
    \end{array}
\right.
\end{equation}
so that, according to Definition~\ref{def:majorization},
\begin{equation}
    \bm{\lambda} \prec_{S} \hat{\bm{\lambda}}.
\end{equation}
Lemma~\ref{symmetric_schur_convex_lemma} ensures the Schur-convexity of $\phi$ and the Schur-concavity of $\psi$, so that
$$\phi(\bm{\lambda})  \leq \phi(\hat{\bm{\lambda}}) = k-1 + \frac{1}{\Lambda - k +1} \leq k-1+\theta,$$
and
$$\psi(\bm{\lambda}) \geq \psi(\hat{\bm{\lambda}}) = \Lambda - k +1\geq \frac{1}{\theta}.$$
\end{proof}


\subsubsection{Frobenius norm bound}
Let $\bm{K} = \bm{V}_{k}^{}\bm{V}_{k}^{\Tran}$, and $\pi$ be a permutation of $[d]$ that reorders the leverage scores decreasingly,
\begin{equation}
   \ell_{\pi_{1}}^{k}\geq \ell_{\pi_{2}}^{k} \geq ... \geq \ell_{\pi_{d}}^{k}.
\end{equation}
By construction, $T_{p_{\eff}}=[\pi_{p_{\eff}},...,\pi_d]$ thus collects the indices of the smallest leverage scores. Finally, denoting by $\bm{\Pi} = (\delta_{i,\pi_{j}})_{(i,j) \in [d] \times [d]}$ the matricial representation of permutation $\pi$, we let
$$ \bm{K}^{\pi} = \bm{\Pi}\bm{K}\bm{\Pi}^{\Tran} = ((\bm{K}_{\pi_i,\pi_j}))_{1\leq i,j\leq d}.$$
The goal of the proof is to bound
\begin{equation}
	\EX_{\DPP} \bigg[ \| \bm{X} - \Pi_{S}^{\Fr}\bm{X} \|_{\Fr}^{2}| S \cap T_{p_{\eff}} = \emptyset \bigg]
	= \frac{{\sum} \Det(\bm{V}	_{S,[k]})^{2}\| \bm{X} - \Pi_{S}^{\Fr}\bm{X} \|_{\Fr}^{2}}{{\sum}\Det(\bm{V}_{S,[k]})^{2}},
\end{equation}
where both sums run over subsets $S\subset[d]$ such that $|S| = k$ and $S \cap T_{p_{\eff}(\theta)}=\emptyset$. For simplicity, let us write
\begin{eqnarray}
    	Z_{k,p_{\eff}(\theta)} & = & \mathlarger{\sum}\limits_{\substack{S \subset [d],|S| = k\\  S \cap T_{p_{\eff}(\theta)} = 		\emptyset}} \Det(\bm{V}_{S,[k]})^{2},\\
	Y_{k,p_{\eff}(\theta)} & = & \mathlarger{\sum}\limits_{\substack{S \subset 	[d],|S| = k\\  S \cap T_{p_{\eff}(\theta)} = \emptyset}} \Det(\bm{V}_{S,[k]})^{2} \Tr(\bm{Z}_{S}^{}\bm{Z}_{S}^{\Tran}).
\end{eqnarray}
Following steps \eqref{eq:inequality_approximation_error_by_EEPP} to \eqref{eq:majoration_by_sigma_k_1} of the previous proof, one obtains
\begin{align}
	\EX_{\DPP} \bigg[ \| \bm{X} - \Pi_{S}^{\Fr}\bm{X} \|_{\Fr}^{2} \: | \: S \cap T_{p_{\eff}} = \emptyset \bigg]
 	  & \leq \| \bm{X}-\Pi_k\bm{X}\|_{\Fr}^{2} + \bm{\sigma}_{k+1}^{2} \frac{Y_{k,p_{\eff}(\theta)}}{Z_{k,p_{\eff}(\theta)}} \label{eq:firstineq_EDDP_proj}.
\end{align}
By definition \eqref{e:defbeta} of the flatness parameter $\beta$,
\begin{equation}
	\label{eq:upperbound_sigmakplus1}
	\bm{\sigma}_{k+1}^{2}
	= \beta \frac{1}{d-k} \sum\limits_{j \geq k+1} \bm{\sigma}_{j}^{2}
	= \beta \frac{1}{d-k}\| \bm{X}-\Pi_k\bm{X}\|_{\Fr}^{2}.
\end{equation}
Then, it remains to upper bound the ratio $Y_{k,p_{\eff}(\theta)}/Z_{k,p_{\eff}(\theta)}$ in \eqref{eq:firstineq_EDDP_proj}, which is the important part of the proof. We first evaluate $Z_{k,p_{\eff}(\theta)}$ and then bound $Y_{k,p_{\eff}(\theta)}$.
%

The matrix $\bm{\Pi}\bm{V}_{k}\in\mathbb{R}^{d\times k}$ has its rows ordered by decreasing leverage scores. Let $\tilde{\bm{V}}^\pi_{p_{\eff}(\theta)} \in \mathbb{R}^{p_{\eff}(\theta)\times k} $ be the submatrix corresponding to the first $p_{\eff}(\theta)$ rows of $\bm{\Pi}\bm{V}_{k}$. Let also
$$\hat{\bm{V}}_{p_{\eff}(\theta)}^\pi = \begin{pmatrix}\tilde{\bm{V}}_{\pi,p_{\eff}(\theta)}\\ \bm{0}_{d-p_{\eff}(\theta),k}\end{pmatrix}$$
be padded with zeros. Then
\begin{equation}
	\bm{K}^{\pi}_{p_{\eff}(\theta)} = \left[
	\begin{array}{c|c}
		\tilde{\bm{V}}_{\pi,p_{\eff}(\theta)} \tilde{\bm{V}}_{\pi,p_{\eff}(\theta)}^{\Tran}& \bm{0} \\
		\hline
		\bm{0} & \bm{0}
	\end{array}
	\right] = \hat{\bm{V}}_{p_{\eff}(\theta)}^\pi (\hat{\bm{V}}^\pi_{p_{\eff}(\theta)})^{\Tran} \in\mathbb{R}^{d\times d}.
\end{equation}
%
The nonzero block of $\bm{K}^{\pi}_{p_{\eff}(\theta)}$ is a submatrix of $\bm{K}^\pi$, and $\rank \bm{K}^\pi = \rank \bm{K} = k$. Hence $\bm{K}^{\pi}_{p_{\eff}(\theta)}$ has at most $k$ nonzero eigenvalues
\begin{equation}
\lambda_{1} \geq \lambda_{2} \geq \dots \geq \lambda_{k}\geq 0 = \lambda_{k+1} = \dots =\lambda_d.
\end{equation}
Therefore,
\begin{equation}
	e_{k}(\Lambda(\bm{K}^{\pi}_{p_{\eff}(\theta)})) = \sum_{\substack{T \subset [d]\\ |T| = k}} ~\prod\limits_{j \in T} \lambda_{j} = \prod\limits_{i \in [k]} \lambda_{i}.
\end{equation}
Note moreover that
\begin{equation}
	\forall \ell \in [k], \:\: e_{\ell}(\Sigma(\hat{\bm{V}}_{\pi,p_{\eff}(\theta)})^{2}) = e_{\ell}(\Lambda(\bm{K}^{\pi}_{p_{\eff}(\theta)})).
\end{equation}
By construction,
\begin{align}
	\label{eq:equal_Zkpeff}
	Z_{k,p_{\eff}(\theta)} &= \mathlarger{\sum}\limits_{\substack{S \subset [d],|S| = k\\  S \cap T_{p_{\eff}(\theta)}	= \emptyset}} \Det(\bm{V}_{S,[k]})^{2}
   = \mathlarger{\sum}\limits_{S \subset [d],|S| = k} \Det\left[\left(\hat{\bm{V}}^\pi_{p_{\eff}(\theta)}\right)_{S,:}\right]^{2}
\end{align}
Then, Lemma~\ref{minors_symmetric_polynomials_lemma} yields
\begin{align}
  Z_{k,p_{\eff}(\theta)} &= e_{k}(\Sigma(\hat{\bm{V}}_{\pi,p_{\eff}(\theta)})^{2}) = e_{k}(\Lambda(\bm{K}	^{\pi}_{p_{\eff}(\theta)})) = \prod_{i \in [k]}\lambda_{i}.
\end{align}
Now we bound $Y_{k,p_{\eff}(\theta)}$. We use again principal angles and trigonometric identities. Using \eqref{eq:trace_tan_identity} and \eqref{eq:sumcostan} above, it holds 
\begin{align}
    	Y_{k,p_{\eff}(\theta)} & = \mathlarger{\sum}\limits_{\substack{S \subset [d],|S| = k\\  S \cap T_{p_{\eff}(\theta)} = 		\emptyset}} \Det(\bm{V}_{S,[k]})^{2} \Tr(\bm{Z}_{S}^{}\bm{Z}_{S}^{\Tran}) \nonumber\\
  	  & = \sum_{\substack{S \subset [d],|S| = k\\  S \cap T_{p_{\eff}(\theta)} = \emptyset}} \prod\limits_{i \in [k]} \cos^{2}	(\theta_{i}(S)) \sum_{j \in [k]}\tan^{2}(\theta_{j}(S))\nonumber\\
  	  & =\sum_{\substack{S \subset [d],|S| = k\\  S \cap T_{p_{\eff}(\theta)} = \emptyset}}  e_{k-1}\left(\Sigma(\bm{V}_{S,[k]})^{2}\right) -  k\, e_{k}\left(\Sigma(\bm{V}_{S,[k]}\right)^{2}\label{e:myTool}\\
      & = \sum_{S \subset [d],|S| = k} e_{k-1}\left(\Sigma\left(\left[\hat{\bm{V}}^\pi_{p_{\eff}(\theta)}\right]_{S,:}\right)^{2}\right) -  k\, e_{k}\left(\Sigma\left(\left[\hat{\bm{V}}^\pi_{p_{\eff}(\theta)}\right]_{S,:}\right)^{2}\right)
\end{align}
By Lemma~\ref{sum_k_1_symmetric_poly_det_lemma} applied to the matrix $\hat{\bm{V}}_{\pi,p_{\eff}(\theta)}$ combined to \eqref{eq:equal_Zkpeff}, we get
\begin{align}
	Y_{k,p_{\eff}(\theta)} & \leq (p_{\eff}(\theta)-k+1)e_{k-1}(\Sigma(\hat{\bm{V}}^\pi_{p_{\eff}(\theta)})^{2}) -  	k\, e_{k}(\Sigma(\hat{\bm{V}}^\pi_{p_{\eff}(\theta)})^{2}) \nonumber\\
    	& \leq (p_{\eff}(\theta)-k+1)e_{k-1}(\Lambda(\bm{K}^{\pi}_{p_{\eff}(\theta)})) -  k\, e_{k}(\Lambda(\bm{K}^{\pi}_{p_{\eff}	(\theta)})) \nonumber\\
    	& \leq \bigg( (p_{\eff}(\theta)-k+1)\phi(\tilde{\bm{\lambda}}) -  k \bigg) Z_{k,p_{\eff}(\theta)}. \label{eq:majore_Y}
\end{align}
where $\tilde{\bm{\lambda}} = (1,\dots,1,\Tr(\bm{K}^\pi_{p_{\eff}(\theta)})-k+1)\in\mathbb{R}^{k}$, see Lemma~\ref{lemma_k_minus_half}. Now, as in the proof of Lemma~\ref{lemma_k_minus_half},
$$ \phi(\tilde{\bm{\lambda}}) = k-1+\frac{1}{\Tr(\bm{K}^\pi_{p_{\eff}(\theta)})-k+1} \leq k-1+\theta$$
by \eqref{eq:leverage_score_decreasing_hypo}. Thus \eqref{eq:majore_Y} yields
\begin{equation}\label{eq:bound_tan_in_expectation_eff}
 	\frac{Y_{k,p_{\eff}(\theta)}}{Z_{k,p_{\eff}(\theta)}} \leq (p_{\eff}(\theta)-k+1) (k-1+\theta) -  k \leq (p_{\eff}(\theta)-k+1) (k-1+\theta).
\end{equation}
Finally, plugging \eqref{eq:bound_tan_in_expectation_eff} and  \eqref{eq:upperbound_sigmakplus1} in \eqref{eq:firstineq_EDDP_proj} concludes the proof of \eqref{eq:boundFrobenius_assumpt2}.

\subsubsection{Spectral norm bound}
We proceed as for the Frobenius norm, using the notation of Section~\ref{s:frobNormThm19}. Lemma~\ref{refined_analysis_of_approximation_bound}, Equations~\eqref{e:myTool} and \eqref{eq:bound_tan_in_expectation_eff} yield
\begin{align*}
	\EX_{\DPP} \bigg[ \| \bm{X} - \Pi_{S}^{2}\bm{X} \|_{2}^{2}\: &| \: S \cap T_{p_{\eff}} = \emptyset \bigg] \\
   &= Z_{k,p_{\eff}(\theta)}^{-1} \mathlarger{\sum}\limits_{\substack{S \subset [d],|S| = k\\  S \cap T_{p_{\eff}(\theta)} = \emptyset}}\Det(\bm{V}	_{S,[k]})^{2}\| \bm{X} - \Pi_{S}^{2}\bm{X} \|_{2}^{2},\\
     &  \leq Z_{k,p_{\eff}(\theta)}^{-1}  \| \bm{X}-\Pi_k\bm{X}\|_{2}^{2} \mathlarger{\sum}_{\substack{S \subset [d], |S| = k\\ S \cap T_{p_{\eff}(\theta)} = \emptyset,\\ \Det(\bm{V}_{S,[k]})^{2}>0}} 	\prod\limits_{\ell = 1}^{k-1} \sigma_{\ell}^{2}(\bm{V}_{S,[k]})\\
     	& \leq Z_{k,p_{\eff}(\theta)}^{-1} \| \bm{X}-\Pi_k\bm{X}\|_{2}^{2}\mathlarger{\sum}_{\substack{S \subset [d], |S| = k\\ S \cap T_{p_{\eff}(\theta)} = \emptyset\\ \Det(\bm{V}_{S,[k]})^{2}>0}} 	e_{k-1}(\Sigma(\bm{V}_{S,[k]})^{2})\\
     & \leq \frac{Y_{k,p_{\eff}(\theta)}}{Z_{k,p_{\eff}(\theta)}} \| \bm{X}-\Pi_k\bm{X}\|_{2}^{2} \\
   	 & \leq (p_{\eff}(\theta)-k+1)(k-1+\theta)\| \bm{X}-\Pi_k\bm{X}\|_{2}^{2},
\end{align*}
which is the claimed spectral bound.

\subsubsection{Bounding the probability of rejection}
Still with the notation of Section~\ref{s:frobNormThm19}, \eqref{eq:equal_Zkpeff} yields
\begin{align}
    \mathbb{P}(S \cap T_{p_{\eff}(\theta)} = \emptyset) &= \sum\limits_{\substack{S \subset [d], |S| = k\\ S \cap T_{p_{\eff}(\theta)} = \emptyset}} \Det(\bm{V}_{S,[k]})^{2} \nonumber\\
    &= e_k(\bm{K}^\pi_{p_{\eff}(\theta)})\\
    &= \prod_{i \in [k]}\lambda_{i} \nonumber\\
    & = \psi(\hat{\bm{\lambda}}).
\end{align}
Lemma~\ref{lemma_k_minus_half} concludes the proof since
\begin{equation}
	\psi(\hat{\bm{\lambda}}) \geq \frac{1}{\theta}.
\end{equation}

\subsection{Proof of Proposition~\ref{prop:relaxed_sparsity_prediction_under_dpp}}
First, Proposition~\ref{prop:sparse_regression_bound} gives
\begin{equation}
    	\mathcal{E}(\bm{w}_{S})  \leq \frac{(1+\max\limits_{i \in [k]} \tan^{2}\theta_{i}(S))	\|\bm{w}^{*}\|^{2}\sigma_{k+1}^{2}}{N} + \frac{k}{N}\nu.
\end{equation}
Now \eqref{eq:trace_tan_relationship} further gives
\begin{equation}
	\max\limits_{i \in [k]} \tan^{2}\theta_{i}(S) \leq \sum\limits_{i \in [k]} \tan^{2}\theta_{i}(S) = \Tr(\bm{Z}_{S}^{}\bm{Z}_{S}	^{\Tran}).
\end{equation}
The proof now follows the same lines as for the approximation bounds. First, following the lines of Section~\ref{s:proofOfExactSparsitySetting}, \rb{This is quite elusive, maybe needs one or two more sentences}, we straightforwardly bound
\begin{equation}
	\EX_{\DPP} \sum\limits_{i \in [k]} \tan^{2}(\theta_{i}(S)) = \sum\limits_{S \subset [d], |S| = k} \quad \prod\limits_{i \in 	[k]} 	\cos^{2}(\theta_{i}(S)) \sum_{j \in [k]}\tan^{2}(\theta_{j}(S))
\end{equation}
and obtain \eqref{eq:prediction_bound_dpp}. In a similar vein, the same lines as in  Section~\ref{s:proofOfEffectiveSparsitySetting} allow bounding
\begin{equation}
	\EX_{\DPP} \bigg[ \sum\limits_{i \in [k]}  \tan^{2}(\theta_{i}(S)) \: | \: S \cap T_{p_{\eff}} = \emptyset \bigg] = 			\sum\limits_{\substack{S \subset [d],|S| = k\\  S \cap T_{p_{\eff}(\theta)} = \emptyset}} \quad \prod\limits_{i \in [k]} 	\cos^{2}	(\theta_{i}(S)) \sum_{j \in [k]}\tan^{2}(\theta_{j}(S).
\end{equation}
and yield \eqref{eq:prediction_bound_dpp2}.

\section{Generating orthogonal matrices with prescribed leverage scores}
\label{app:framebuilding}

In this section, we describe an algorithm that samples a random orthonormal matrix with a prescribed profile of $k$-leverage scores. This algorithm was used to generate the matrices $\bm{F} = \bm{V}_{k}^{\Tran} \in \mathbb{R}^{k \times d}$ for the toy datasets of Section~\ref{s:numexpesection}. The orthogonality constraint can be expressed as a condition on the spectrum of the matrix $\bm{K} = \bm{V}_{k}^{}\bm{V}_{k}^{\Tran}$, namely $\Sp(\bm{K}) \subset \{0,1\}$. On the other hand, the constraint on the $k$-leverage scores can be expressed as a condition on the diagonal of $\bm{K}$. Thus, the problem of generating an orthogonal matrix with a given profile of $k$-leverage scores boils down to enforcing conditions on the spectrum and the diagonal of a symmetric matrix $\bm{K}$.

\subsection{Definitions and statement of the problem}
 We denote by $(\bm{f}_{i})_{i \in [d]}$ the columns of the matrix $\bm{F}$. For $n \in \mathbb{N}$, we write $\mathbb{1}_{n}$ the vector containing ones living in $\mathbb{R}^{n}$, and $\mathbb{0}_{n}$ the vector containing zeros living in $\mathbb{R}^{n}$. We say that the vector $\bm{u} \in \mathbb{R}^{n}$ interlaces on $\bm{v} \in \mathbb{R}^{n}$ and we denote $$\bm{u} \sqsubseteq \bm{v}$$
if $u_{n} \leq v_{n}$ and $\forall i \in [1:n-1], \: v_{i+1} \leq u_{i} \leq v_{i}$.
\begin{definition}

\begin{figure}[!ht]
    \centering
    \begin{tikzpicture}[scale = 0.3]
\filldraw
(-12,0) circle (0.1pt) node[align=center, below] {$\dots$} --
(-9.5,0) circle (6pt) node[align=left,   below] {$v_{i+2}$} --
(-5.8,0) circle (6pt) node[align=center, below] {$u_{i+1}$} --
(-3,0) circle (6pt) node[align=left,   below] {$v_{i+1}$} --
(0,0) circle (6pt) node[align=left,   below] {$u_{i}$} --
(1.5,0) circle (6pt) node[align=left,   below] {$v_{i}$} --
(5.4,0) circle (6pt) node[align=center, below] {$u_{i-1}$} --
(9.9,0) circle (6pt) node[align=left,   below] {$v_{i-1}$} --
(12,0) circle (0.1pt) node[align=right,  below] {$\dots$};
\end{tikzpicture}
    \caption{Illustration of the interlacing of $\bm{u}$ on $\bm{v}$.}
    \label{f:Interlacing_eigenvalues}
\end{figure}
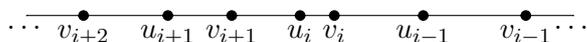

Let $k, d \in \mathbb{N}$, with $k \leq d$. Let $\bm{F} \in \mathbb{R}^{k \times d}$ be a full rank matrix\footnote{A \emph{frame}, using the definitions of \citep{FiMiPo11} and \citep{FMPS11}.}.
Within this section, we denote $\bm{\sigma}^2 = (\sigma_{1}^2, \sigma_{2}^2, \dots ,\sigma_{k}^2)$ the squares of the nonvanishing singular values of the matrix $\bm{F}$, and $\bm{\ell} = (\ell_{1}=\|\bm{f}_{1}\|^{2}, \ell_{2}=\|\bm{f}_{2}\|^{2}, \dots, \ell_{d}=\|\bm{f}_{d}\|^{2})$ are the squared norms of the columns of $\bm{F}$, which we assume to be ordered decreasingly:
$$\ell_{1} \geq \ell_{2} \geq \dots \geq \ell_{d}.$$
 When $\bm{F}$ is orthogonal, we can think of $\bm{\ell}$ as a vector of leverage scores.
\end{definition}

We are interested in the problem of constructing an orthogonal matrix given its leverage scores.
\begin{problem}\label{prob:orthogonal_frame_existence}
Let $k,d \in \mathbb{N}$, with $k \leq d$, and let $\bm{\ell} \in \mathbb{R}_{+}^{d}$ such that $\sum\limits_{i =1}^{d} \ell_{i} = k$. Build a matrix $\bm{F} \in \mathbb{R}^{k\times d}$ such that
\begin{equation}
\Sp(\bm{F}^{\Tran}\bm{F}) = [\mathbb{1}_{k},\mathbb{0}_{d-k}],
\end{equation}
and
\begin{equation}
\Diag(\bm{F}^{\Tran}\bm{F}) = \bm{\ell}.
\end{equation}
\end{problem}
We actually consider here the generalization of Problem~\ref{prob:general_frame_existence} to an arbitrary spectrum.

\begin{problem}\label{prob:general_frame_existence}
Let $k,d \in \mathbb{N}$, with $k \leq d$, and let $\bm{\ell} \in \mathbb{R}_{+}^{d}$ such that $\sum\limits_{i =1}^{d} \ell_{i} = \sum\limits_{i = 1}^{k} \sigma_{i}^2$. Build a matrix $\bm{F} \in \mathbb{R}^{k\times d}$ such that
\begin{equation}
\Sp(\bm{F}^{\Tran}\bm{F}) = [\bm{\sigma}^2,\mathbb{0}_{d-k}] =:\bm{\hat\sigma}^2
\end{equation}
and
\begin{equation}
\Diag(\bm{F}^{\Tran}\bm{F}) = \bm{\ell}.
\end{equation}
\end{problem}

Denote by
\begin{equation}
\mathcal{M}_{(\bm{\ell},\bm{\sigma})} =  \{ \bm{M} \in \mathbb{R}^{d\times d}\text{ symmetric } \big/~  \Diag(\bm{M}) = \bm{\ell}, ~\Sp(\bm{M}) = \bm{\hat{\sigma}}^2\}.
\end{equation}
The non-emptiness of $\mathcal{M}_{(\bm{\ell},\bm{\sigma})}$ is determined by a majorization condition between $\bm{\ell}$ and $\hat{\bm{\sigma}}$, see Appendix~\ref{app:majorization} for definitions. More precisely, we have the following theorem.
\begin{theorem}[Schur-Horn]\label{thm:majorization_equivalence}
Let $k,d \in \mathbb{N}$, with $k \leq d$, and let $\bm{\ell} \in \mathbb{R}_{+}^{d}$. We have
\begin{equation}\label{eq:nonemptiness_majorization_equivalence}
\mathcal{M}_{(\bm{\ell},\bm{\sigma})} \neq \emptyset \Leftrightarrow \bm{\ell} \prec_{S} \hat{\bm{\sigma}}.
\end{equation}
\end{theorem}
The proof by \cite{Hor54} of the reciprocal in Theorem~\ref{thm:majorization_equivalence} is non constructive. In the next section, we survey algorithms that output an element of $\mathcal{M}_{(\bm{\ell},\bm{\sigma})}$.

\subsection{Related work}

Several articles (\citealp{RaMa14}, \citealp{MaMaYu14}) in the randomized linear algebra community propose the use of non Gaussian random matrices to generate matrices with a fast decreasing profile of leverage scores (so-called \emph{heavy hitters}) without controlling the exact profile of the leverage scores.

\cite{DhHeSuTr05} showed how to generate matrices from $\mathcal{M}_{(\bm{\ell},\bm{\sigma})}$ using Givens rotations; see the algorithm in Figure~\ref{f:Algo_Givens}. The idea of the algorithm is to start with a frame with the exact spectrum and repeatedly apply orthogonal matrices (Lines 4 and 6 of Figure~\ref{f:Algo_Givens}) that preserve the spectrum while changing the leverage scores of only two columns, setting one of their leverage scores to the desired value. The orthogonal matrices are the so-called \emph{Givens rotations}.
\begin{definition}
Let $\theta \in [0, 2\pi[$ and $i,j \in [d]$. The Givens rotation $\bm{G}_{i,j}(\theta) \in \mathbb{R}^{d \times d}$ is defined by
\begin{equation}
\bm{G}_{i,j}(\theta)  = \begin{bmatrix}
    1 & & & & & & & & & & \\
     & \ddots  & & & & & & & & & \\
     &  & 1 & & & & & & & &  \\
     &  &  &  \cos(\theta)&  &  & & -\sin(\theta)& & & \\
     &  &  &  & 1 &   & & & & & \\
     &  &  &  &  & \ddots  & & & & & \\
     &  &  &  &  &   & 1 & & & & \\
     &  &  &  \sin(\theta)&  &  & & \cos(\theta)& & & \\
     &  & & & & & & &1 & & \\
     &   & & & & & & & &\ddots & \\
     &  &  & & & & & & & & 1 \\
    \end{bmatrix}.
\end{equation}

\end{definition}

\begin{figure}[!ht]
\centerline{
\begin{algorithm}{$\Algo{GivensAlgorithm}\big(\bm{\ell},\bm{\sigma})$}\label{algo_givens}
\Aitem $\bm{F} \longleftarrow \left[
\begin{array}{c|c}
\Diag(\bm{\sigma}) & \bm{0}
\end{array}
\right] \in \mathbb{R}^{k \times d}$
\Aitem \While $\exists i,j,k \in [d]$, $i < k <j: \|\bm{f}_{i}\|^{2} < \ell_{i}, \|\bm{f}_{k}\|^{2} = \ell_{k} , \|\bm{f}_{j}\|^{2} > \ell_{j}$
\Aitem \mt \If $\ell_{i} - \|\bm{f}_{i}\|^{2} \leq \|\bm{f}_{j}\|^{2} - \ell_{j}$
\Aitem \mtt $\bm{F}\setto \bm{G}_{i,j}(\theta)\bm{F}$, where  $\|(\bm{G}_{i,j}(\theta)\bm{F})_{i}\|^{2} = \ell_{i}$.
\Aitem \mt \Else
\Aitem \mtt $\bm{F}\setto \bm{G}_{i,j}(\theta)\bm{F}$, where $\|(\bm{G}_{i,j}(\theta)\bm{F})_{j}\|^{2} = \ell_{j}$,
\Aitem \Return $\bm{F} \in \mathbb{R}^{k \times d}$.
\end{algorithm}
}
\caption{The pseudocode of the algorithm proposed by \cite{DhHeSuTr05} for generating a matrix given its leverage scores and spectrum by successively applying Givens rotations.}
\label{f:Algo_Givens}
\end{figure}

Figure~\ref{f:highly_structured_matrix} shows the output of the algorithm in Figure~\ref{f:Algo_Givens}, for the input $(\bm{\ell},\bm{\sigma}) = (\bm{\ell},\mathbb{1})$ for three different values of $\bm{\ell}$. The main drawbacks of this algorithm are first that it is deterministic, so that it outputs a unique matrix $\bm{F}$ for a given input $(\bm{\ell},\bm{\sigma})$, and second that the output is a highly structured matrix, as observed on Figure~\ref{f:highly_structured_matrix}.

 We propose an algorithm that outputs random, more ``generic" matrices belonging to $\mathcal{M}_{(\bm{\ell},\bm{\sigma})}$. This algorithm is based on a parametrization of $\mathcal{M}_{(\bm{\ell},\bm{\sigma})}$ using the collection of spectra of all minors of $\bm{F} \in \mathcal{M}_{(\bm{\ell},\bm{\sigma})}$. This parametrization was introduced by \cite{FMPS11}, and we recall it in Section~\ref{s:gt}. For now, let us simply look at  Figure~\ref{f:gt_generator_outputs}, which displays a few outputs of our algorithm for the same input as in Figure~\ref{f:highly_structured_matrix:a}. We now obtain different matrices for the same input $(\bm{\ell},\bm{\sigma})$, and these matrices are less structured than the output of Algorithm~\ref{f:Algo_Givens}, as required.

\subsection{The restricted Gelfand-Tsetlin polytope}
\label{s:gt}
\begin{definition}
Recall that $(\bm{f}_{i})_{i \in [d]}$ are the columns of the matrix $\bm{F}\in\mathbb{R}^{k\times d}$. For $r \in [d]$, we further define
\begin{equation}
\bm{F}_{r} = \bm{F}_{:,[r]} \in \mathbb{R}^{k \times r},
\end{equation}
\begin{equation}
\bm{C}_{r} = \sum\limits_{i \in [r]} \bm{f}_{i}\bm{f}_{i}^{\Tran} \in \mathbb{R}^{k \times k},
\end{equation}
\begin{equation}
\bm{G}_{r} = \bm{F}_{r}^{\Tran}\bm{F}_{r} \in \mathbb{R}^{r \times r}.
\end{equation}
Furthermore, we note for $r \in [d]$,
\begin{equation}
    (\lambda_{r,i})_{i \in [k]} = \Lambda(\bm{C}_{r}),
\end{equation}
\begin{equation}
    (\tilde{\lambda}_{r,i})_{i \in [r]} = \Lambda(\bm{G}_{r}).
\end{equation}
The $(\lambda_{r,i})_{i \in [k]}$, $r\in [d]$, are called the outer eigensteps of $\bm{F}$, and we group them in the matrix $$\Lambda^{\text{out}}(\bm{F}) = (\lambda_{r,i})_{i \in [k],r \in [d]} \in \mathbb{R}^{k \times d}.$$ Similarly, the $(\tilde{\lambda}_{r,i})_{i \in [r]}$ are called inner eigensteps of $\bm{F}$.
\end{definition}

\begin{figure}[!ht]
    \centering
\subfloat[]{\includegraphics[width= 0.25\textwidth]{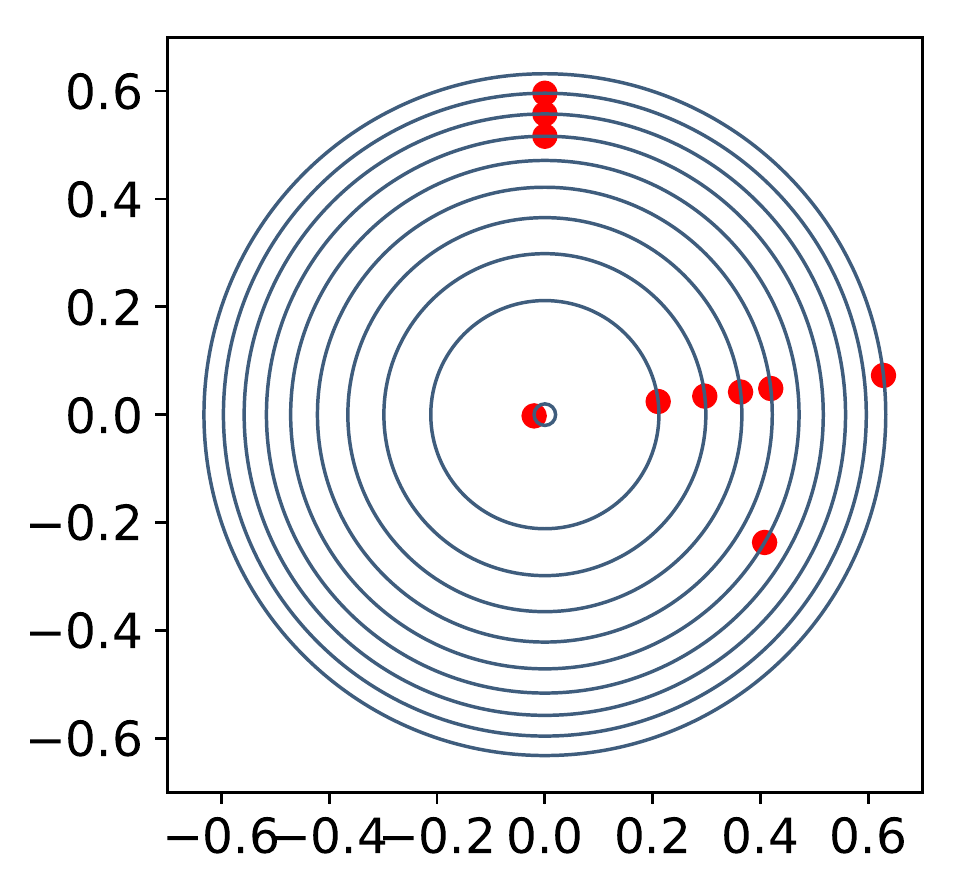}
\label{f:highly_structured_matrix:a}}
\subfloat[]{\includegraphics[width= 0.25\textwidth]{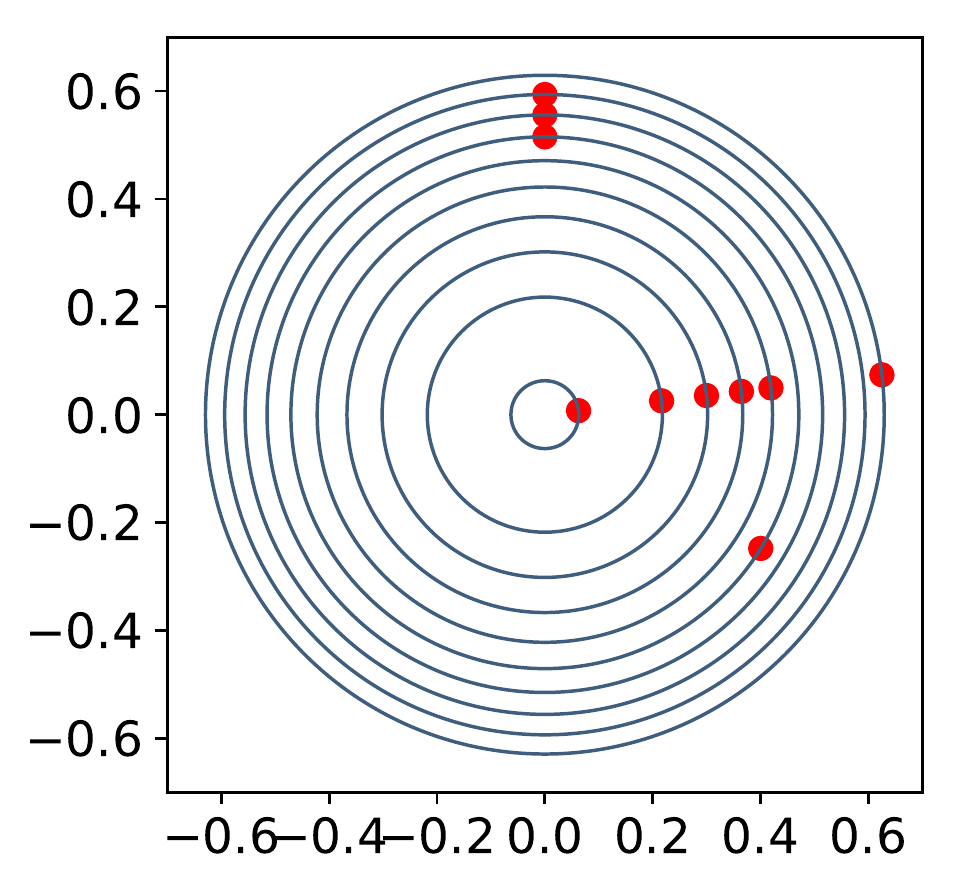}}
\subfloat[]{\includegraphics[width= 0.25\textwidth]{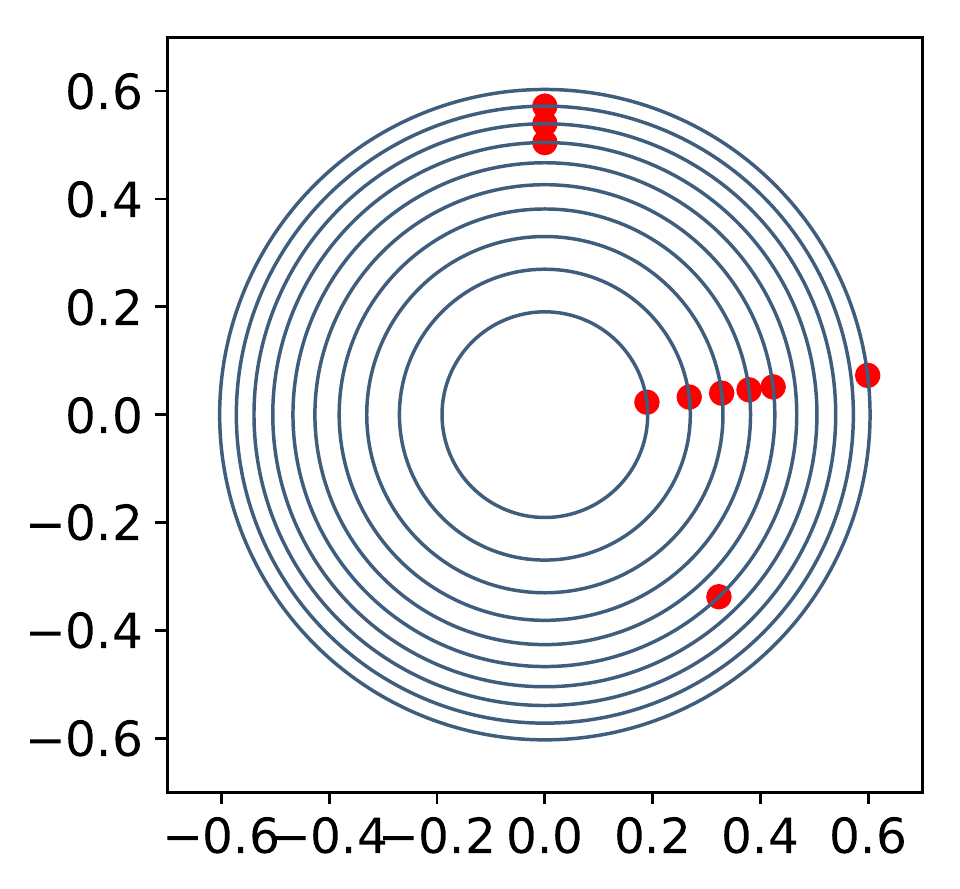}}
\caption{The output of the algorithm in Figure~\ref{f:Algo_Givens} for $k=2, \: d = 10$, $\bm{\sigma} = (1,1)$, and three different values of $\bm{\ell}$ that each add to $k$. Each red dot has coordinates a column of $\bm{F}$. The blue circles have for radii the prescribed $(\sqrt{\ell_i})$.
\label{f:highly_structured_matrix}}
\end{figure}

\begin{figure}[!ht]
    \centering
\subfloat[]{\includegraphics[width= 0.25\textwidth]{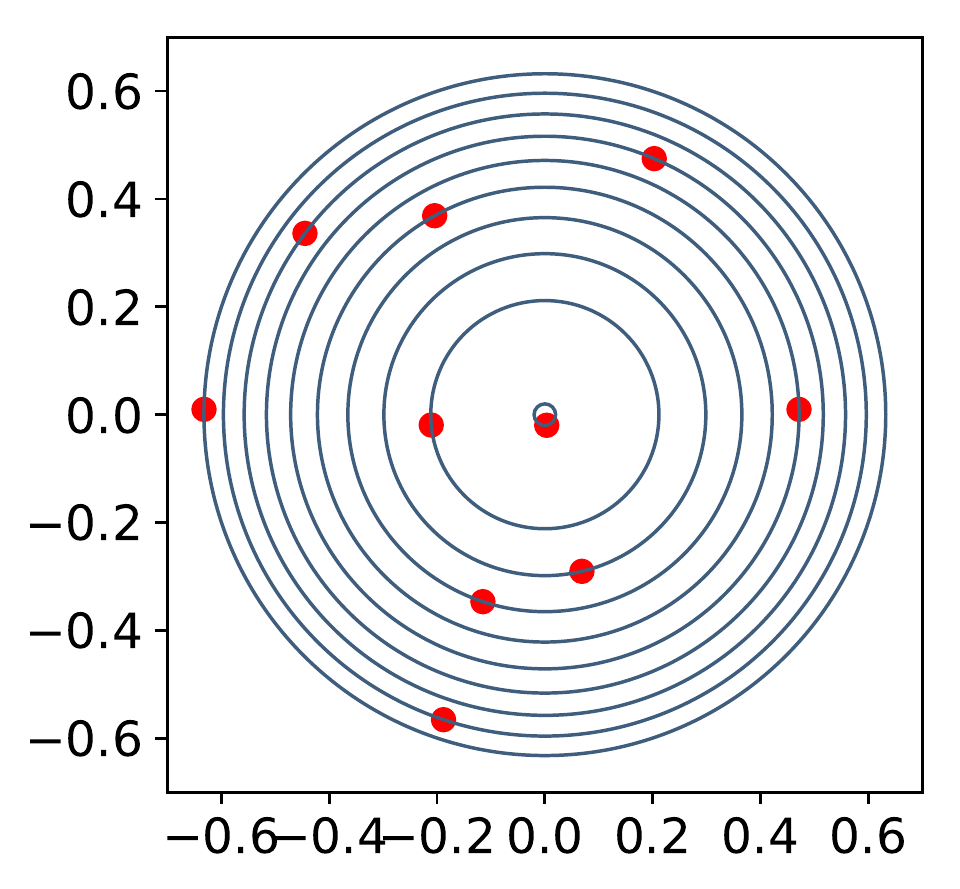}}
\subfloat[]{\includegraphics[width= 0.25\textwidth]{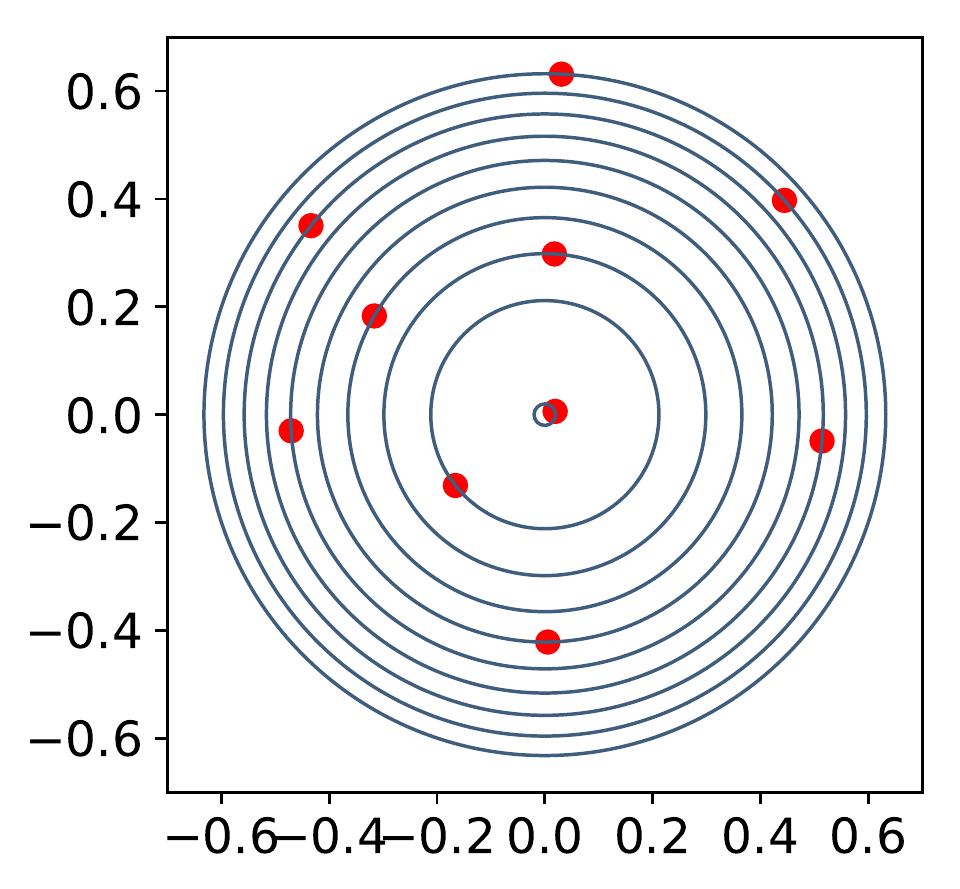}}
\subfloat[]{\includegraphics[width= 0.25\textwidth]{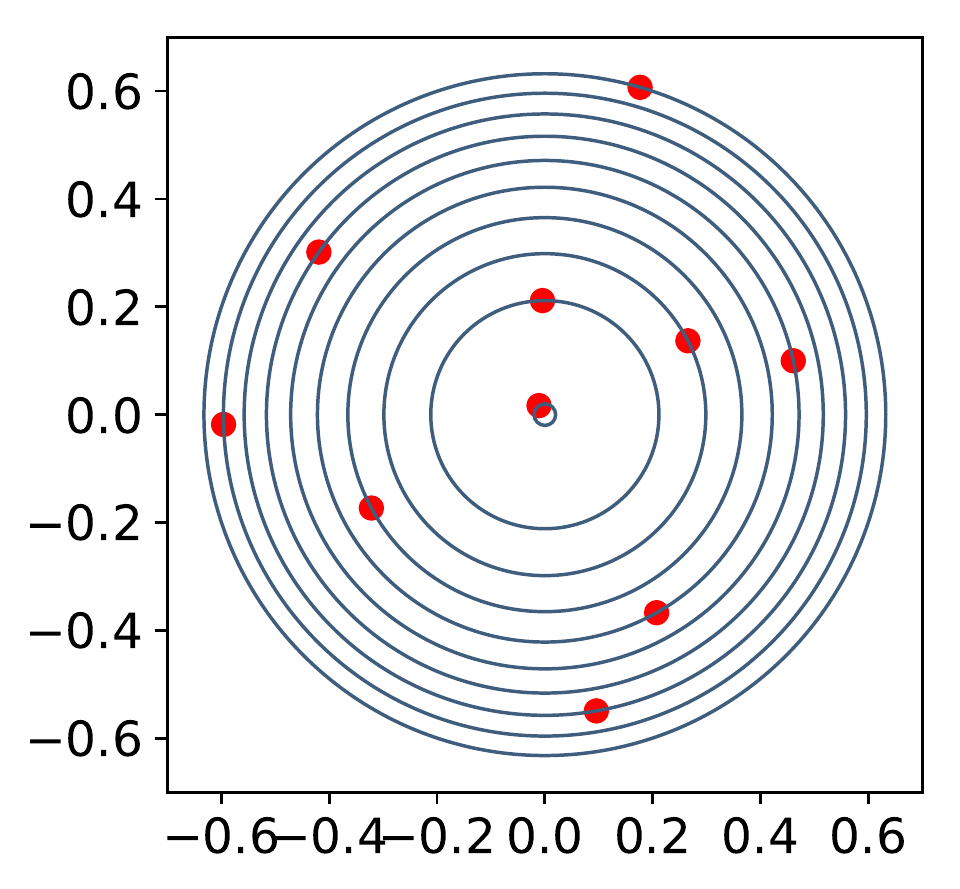}}
\subfloat[]{\includegraphics[width= 0.25\textwidth]{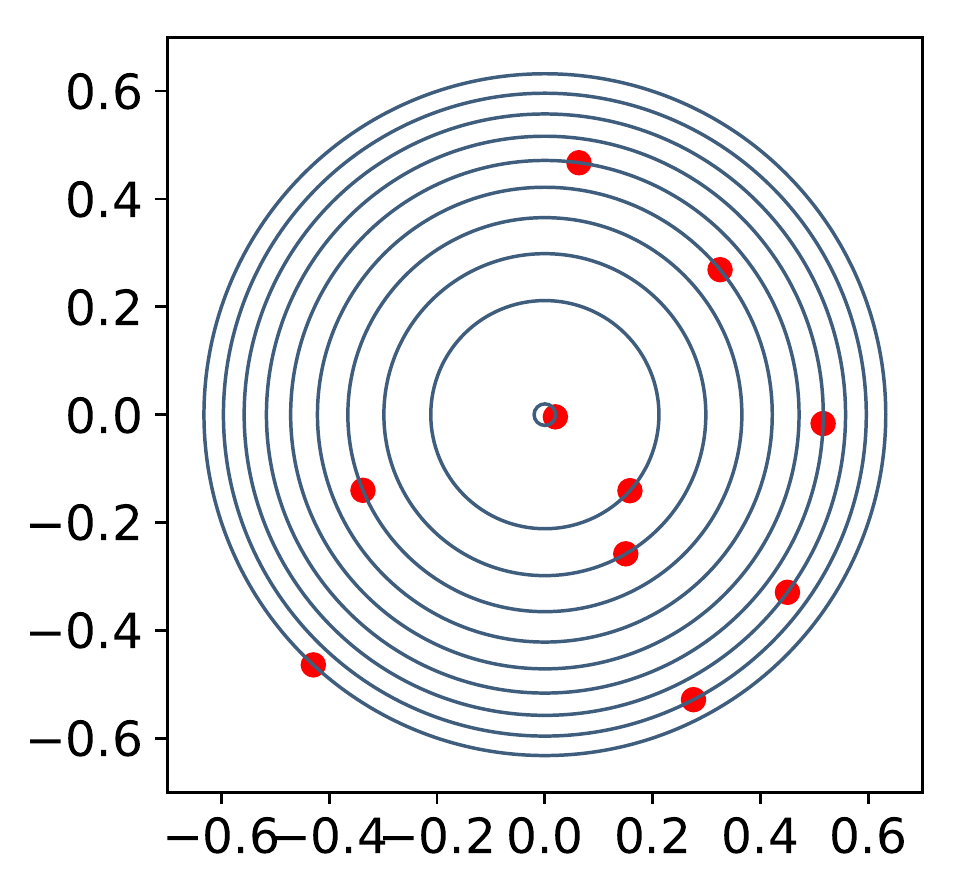}}\\
\subfloat[]{\includegraphics[width= 0.25\textwidth]{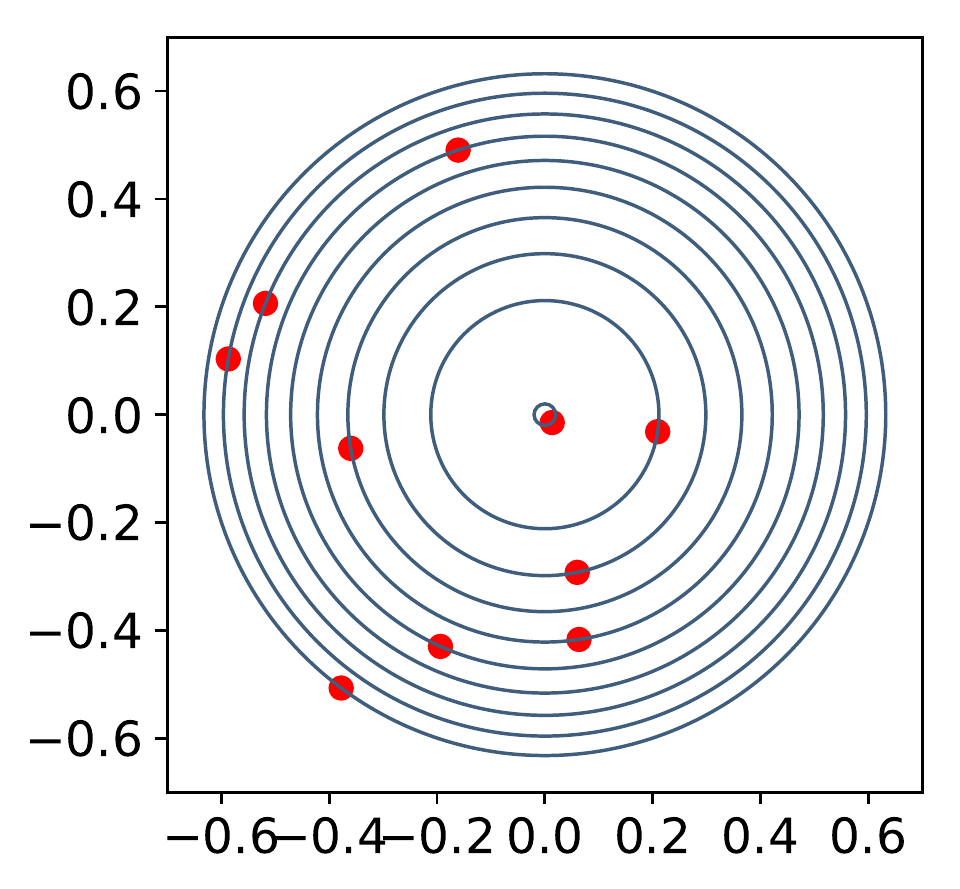}}
\subfloat[]{\includegraphics[width= 0.25\textwidth]{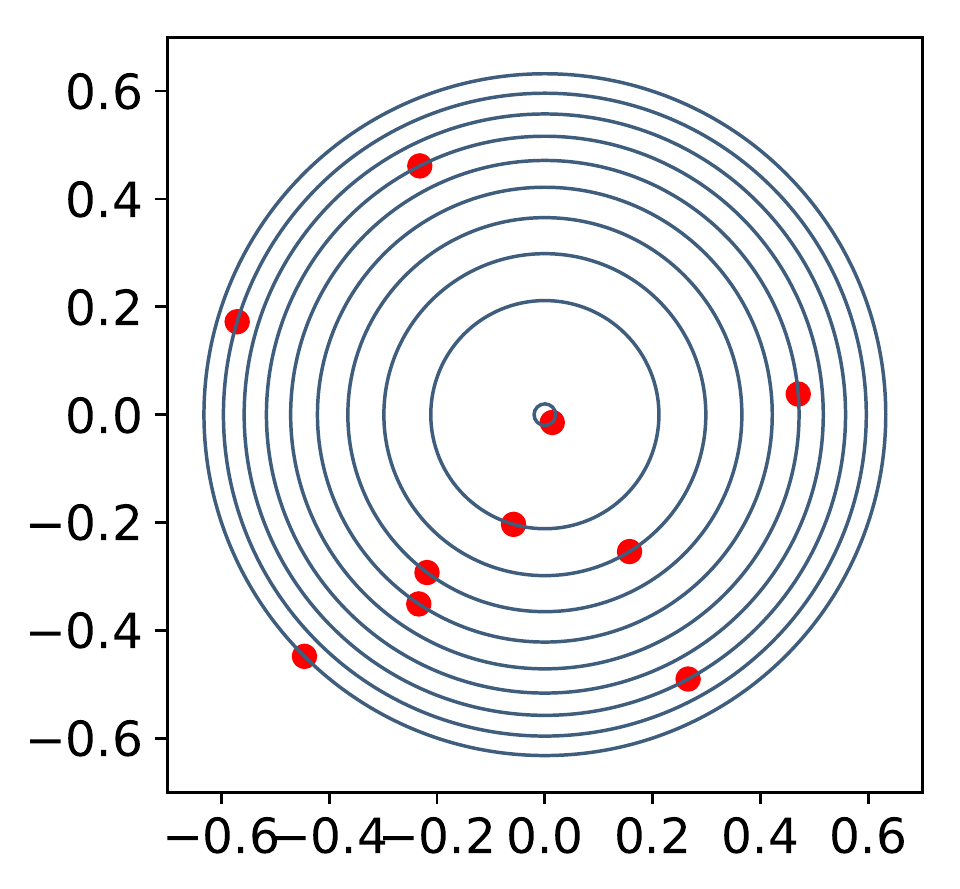}}
\subfloat[]{\includegraphics[width= 0.25\textwidth]{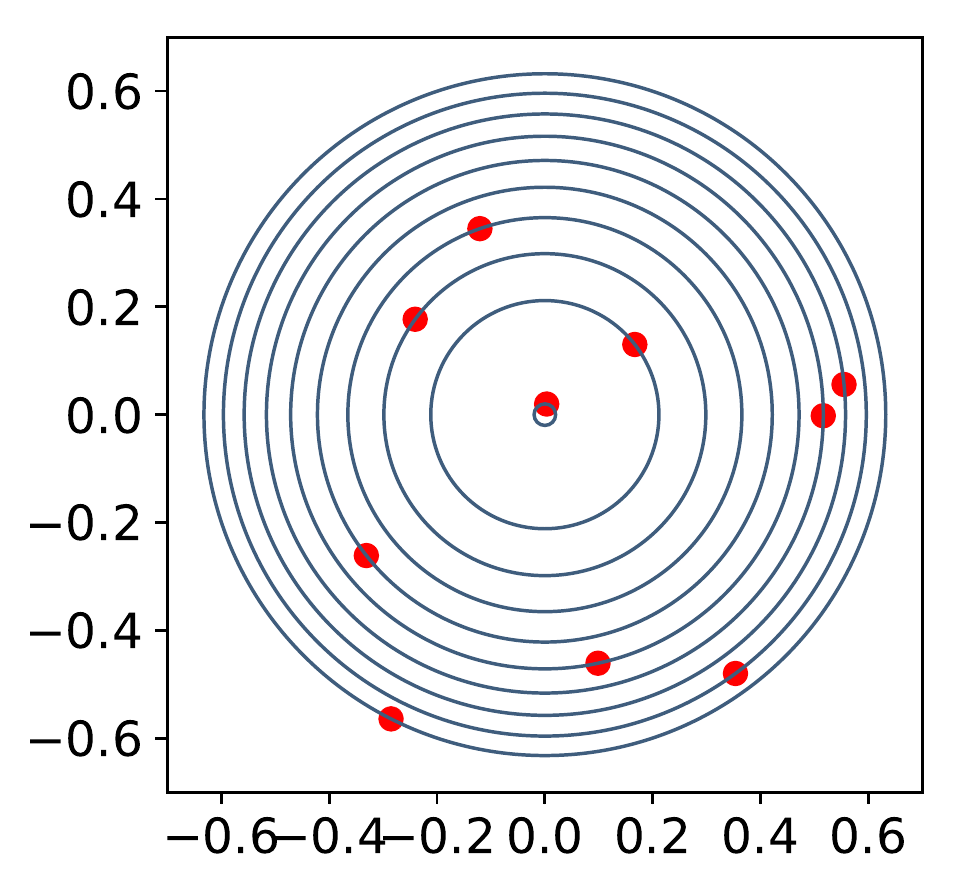}}
\subfloat[]{\includegraphics[width= 0.25\textwidth]{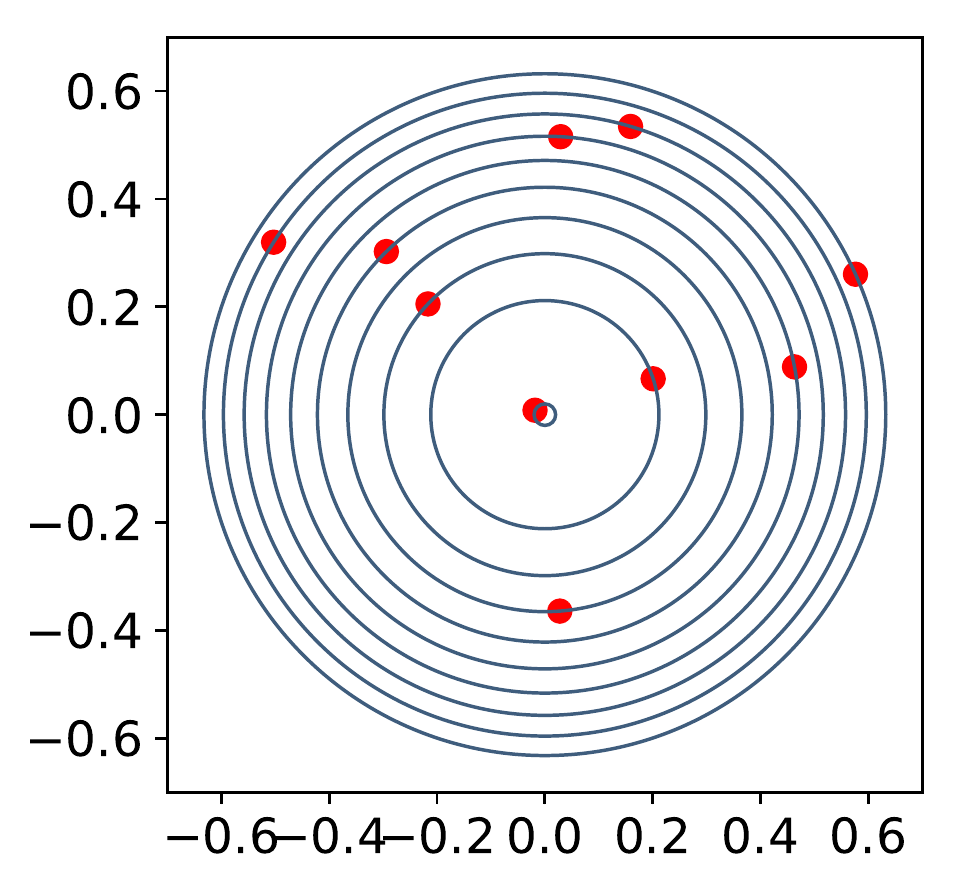}}
\caption{The output of our algorithm for $k=2, \: d = 10$, an input $\bm{\sigma} = (1,1)$, and $\ell$ as in Figure~\ref{f:highly_structured_matrix:a}. Each red dot has coordinates a column of $\bm{F}$. The blue circles have for radii the prescribed $(\sqrt{\ell_i})$. \label{f:gt_generator_outputs}}
\end{figure}



\begin{example}
For $k=2$, $d = 4$, consider the full-rank matrix
\begin{equation}
    \bm{F} = \begin{bmatrix}
1 & 0 & -1 & 0 \\
0 & 1 & 0 & -1
\end{bmatrix},
\end{equation}
Then
\begin{equation}
    \Lambda^{\text{out}}(\bm{F}) = \begin{bmatrix}
1 & 1 & 2 & 2 \\
0 & 1 & 1 & 2
\end{bmatrix}.
\end{equation}

\end{example}

\begin{proposition}
The outer eigensteps satisfy the following constraints:\\
\begin{equation}\label{eq:GT_polytope_equations}
\begin{cases}
     \forall i \in [k], \:\: \lambda_{0,i} = 0 \\
     \forall i \in [k], \:\: \lambda_{d,i} = \sigma_{i}^2 \\
     \forall r \in [d], \:\: (\lambda_{r,:}) \sqsubseteq  (\lambda_{r+1,:}) \\
     \forall r \in [d], \:\: \sum\limits_{i \in [d]} \lambda_{r,i} = \sum\limits_{i \in [r]} \ell_{i}
\end{cases}.
\end{equation}
\end{proposition}

\begin{figure}[!ht]
    \centering
    \begin{tikzpicture}[scale = 0.3]

\draw  (-40,7)  node  {$\ell_{1}=\lambda_{1,1}$} ;
\draw  (-36,7)  node  {$\newprec$} ;
\draw  (-32,7)  node  {$\lambda_{2,1}$} ;
\draw  (-28,7)  node  {$\newprec$} ;
\draw  (-24,7)  node  {$\lambda_{3,1}$} ;
\draw  (-20,7)  node  {$\dots$} ;
\draw  (-16,7)  node  {$\lambda_{d-1,1}$} ;
\draw  (-12,7)  node  {$\newprec$} ;
\draw  (-8,7)  node  {$\lambda_{d,1} = \sigma_{1}$} ;

\draw  (-38.5,5.5)  node  {$\scriptscriptstyle +$} ;
\draw  (-36,5.5)  node  {$\newtiltedprec$} ;
\draw  (-32,5.5)  node  {$\scriptscriptstyle +$} ;
\draw  (-28,5.5)  node  {$\newtiltedprec$} ;
\draw  (-24,5.5)  node  {$\scriptscriptstyle +$} ;
\draw  (-20,5.5)  node  {$\dots$} ;
\draw  (-16,5.5)  node  {$\scriptscriptstyle +$} ;
\draw  (-12,5.5)  node  {$\newtiltedprec$} ;
\draw  (-8.7,5.5)  node  {$\scriptscriptstyle +$} ;

\draw  (-40,4)  node  {$0=\lambda_{1,2}$} ;
\draw  (-36,4)  node  {$\newprec$} ;
\draw  (-32,4)  node  {$\lambda_{2,2}$} ;
\draw  (-28,4)  node  {$\newprec$} ;
\draw  (-24,4)  node  {$\lambda_{3,2}$} ;
\draw  (-20,4)  node  {$\dots$} ;
\draw  (-16,4)  node  {$\lambda_{d-1,2}$} ;
\draw  (-12,4)  node  {$\newprec$} ;
\draw  (-8,4)  node  {$\lambda_{d,2} = \sigma_{2}$} ;

\draw  (-38.5,2.5)  node  {$\scriptscriptstyle +$} ;
\draw  (-36,2.5)  node  {$\newtiltedprec$} ;
\draw  (-32,2.5)  node  {$\scriptscriptstyle +$} ;
\draw  (-28,2.5)  node  {$\newtiltedprec$} ;
\draw  (-24,2.5)  node  {$\scriptscriptstyle +$} ;
\draw  (-20,2.5)  node  {$\dots$} ;
\draw  (-16,2.5)  node  {$\scriptscriptstyle +$} ;
\draw  (-12,2.5)  node  {$\newtiltedprec$} ;
\draw  (-8.7,2.5)  node  {$\scriptscriptstyle +$} ;

\draw  (-40,1)  node  {$0=\lambda_{1,3}$} ;
\draw  (-36,1)  node  {$\newprec$} ;
\draw  (-32,1)  node  {$\lambda_{2,3}$} ;
\draw  (-28,1)  node  {$\newprec$} ;
\draw  (-24,1)  node  {$\lambda_{3,3}$} ;
\draw  (-20,1)  node  {$\dots$} ;
\draw  (-16,1)  node  {$\lambda_{d-1,3}$} ;
\draw  (-12,1)  node  {$\newprec$} ;
\draw  (-8,1)  node  {$\lambda_{d,3} = \sigma_{3}$} ;

\draw  (-40,0)  node  {$\vdots$} ;
\draw  (-32,0)  node  {$\vdots$} ;
\draw  (-24,0)  node  {$\vdots$} ;
\draw  (-20,0)  node  {$\vdots$} ;
\draw  (-16,0)  node  {$\vdots$} ;
\draw  (-8,0)  node  {$\vdots$} ;

\draw  (-40,-1.5)  node  {$0=\lambda_{1,k}$} ;
\draw  (-36,-1.5)  node  {$\newprec$} ;
\draw  (-32,-1.5)  node  {$\lambda_{2,k}$} ;
\draw  (-28,-1.5)  node  {$\newprec$} ;
\draw  (-24,-1.5)  node  {$\lambda_{3,k}$} ;
\draw  (-20,-1.5)  node  {$\dots$} ;
\draw  (-16,-1.5)  node  {$\lambda_{d-1,k}$} ;
\draw  (-12,-1.5)  node  {$\newprec$} ;
\draw  (-8,-1.5)  node  {$\lambda_{d,k} = \sigma_{k}$} ;

\draw[line width=0.15 mm] (-42,-2.5) -- (-3,-2.5);

\draw  (-38,-3.5)  node  {$\ell_{1}$} ;
\draw  (-36,-3.5)  node  {} ;
\draw  (-32,-4)  node  {$\sum\limits_{i \leq 2}\ell_{i}$} ;
\draw  (-28,-3.5)  node  {} ;
\draw  (-24,-4)  node  {$\sum\limits_{i \leq 3}\ell_{i}$} ;
\draw  (-20,-3.5)  node  {} ;
\draw  (-16,-4)  node  {$\sum\limits_{i \leq d-1}\ell_{i}$} ;
\draw  (-12,-3.5)  node  {} ;
\draw  (-8.7,-4)  node  {$\sum\limits_{i \leq d}\ell_{i}$} ;
\end{tikzpicture}
    \caption{The interlacing relationships \eqref{eq:GT_polytope_equations} satisfied by the outer eigensteps of a frame. Thick triangles are used in place of $\leq$ for improved readability.}
    \label{f:Interlacing_eigenspaces}
\end{figure}

In other words, the outer eigensteps are constrained to live in a polytope.
We define the restricted Gelfand-Tsetlin polytope $\bm{GT}_{(k,d)}(\bm{\sigma},\bm{\ell})$ to be the subset of $\mathbb{R}^{k \times d}$ defined by the equations \eqref{eq:GT_polytope_equations}. A more graphical summary of the interlacing and sum constraints is given in Figure~\ref{f:Interlacing_eigenspaces}. The restricted GT polytope\footnote{Note the difference with the Gelfand-Tsetlin polytope in the random matrix literature \citep{Bar01}, where only the spectrum is constrained, not the diagonal.} allows a parametrization of $\mathcal{M}_{(\bm{\ell},\bm{\sigma})}$ by the following reconstruction result.



\begin{theorem}[Theorem 3, \citealp{FiMiPo11}]\label{thm:GT_parameterization}
Every matrix $\bm{F}\in \mathcal{M}_{(\bm{\ell},\bm{\sigma})}$ can be constructed as follows:
\begin{itemize}
    \item pick a valid sequence of outer eigensteps noted $\Lambda^{\text{out}} \in \bm{GT}_{(k,d)}(\bm{\sigma},\bm{\ell})$,
    \item pick $\bm{f}_{1} \in \mathbb{R}^{k}$ such that
    \begin{equation}\label{eq:f_1_equation}
    \| \bm{f}_{1}\|^{2} = \ell_{1}
    ,\end{equation}
    \item for $r \in [d]$, consider the polynomial $p_{r}(x) = \prod\limits_{i \in [d]}(x - \lambda_{r,i})$, and for each $r \in [d-1]$, choose $\bm{f}_{r+1} \in \mathbb{R}^{k} $ such that
    \begin{equation}\label{eq:projections_of_f_on_eigenspaces}
        \forall \lambda \in \{\lambda_{r,i}\}_{i \in [d]} , \: \| \bm{P}_{r,\lambda}\bm{f}_{r+1}\|^{2} = -\lim\limits_{x \to \lambda}(x-\lambda)\frac{p_{r+1}(\lambda)}{p_{r}(\lambda)},
    \end{equation}
    where $\bm{P}_{r, \lambda}$ denotes the orthogonal projection onto the eigenspace $\Kerspace(\lambda \mathbb{I}_{k} - \bm{F}_{r}\bm{F}_{r}^{T})$.
\end{itemize}
Conversely, any matrix $\bm{F}$ constructed by this process is in $\mathcal{M}_{(\bm{\ell},\bm{\sigma})}$.

\end{theorem}
\cite{FiMiPo11} propose an algorithm to construct a vector $\bm{f}_{r}$ satisfying Equation \eqref{eq:projections_of_f_on_eigenspaces}. Finally, an algorithm for the construction of a valid sequence of eigensteps $\Lambda^{\text{out}} \in \bm{GT}_{(k,d)}(\bm{\sigma},\bm{\ell})$ was proposed in \citep{FMPS11}. This yields the following constructive result.

\begin{theorem}[Theorem 4.1, \citealp{FMPS11}]\label{thm:parametrization_of_polytope}
Every matrix $\bm{F}\in\mathcal{M}({\bm{\sigma},\bm{\ell}})$ can be constructed as follows:
\begin{itemize}
    \item Set $\forall i \in [k], \: \tilde\lambda_{d,i} = \sigma_{i}^2$,
    \item For $r \in \{d-1, \dots, 1 \}$, construct $\{\tilde\lambda_{r,:}\}$ as follows. For each $i \in \{k, \dots , 1\}$, pick $$\tilde\lambda_{r-1,i} \in [B_{i,r}(\bm{\ell},\bm{\sigma}),A_{i,r}(\bm{\ell},\bm{\sigma})],$$ where
    \begin{equation}
    \begin{split}
      A_{i,r}(\bm{\ell},\bm{\sigma}) = \max \left\{\tilde\lambda_{r+1,i+1}, \sum\limits_{t = i}^{k}\tilde\lambda_{r+1,t} - \sum\limits_{t = i+1}^{k}\tilde\lambda_{r,t} - \ell_{r+1} \right\}\\
      B_{i,r}(\bm{\ell},\bm{\sigma}) = \min \left\{\tilde\lambda_{r+1,i}, \min\limits_{z = 1,\dots,i} \left\{\sum\limits_{t=z}^{r}\ell_{t} - \sum\limits_{t=z+1}^{i}\tilde\lambda_{r+1,t} - \sum\limits_{t=i+1}^{k}\tilde\lambda_{r,t}\right\} \right\}.
    \end{split}
    \end{equation}
  \end{itemize}
Furthermore, any sequence constructed by this algorithm is a valid sequence of inner eigensteps.
\end{theorem}

Based on these results we propose an algorithm for the generation of orthogonal random matrices with a given profile of leverage scores.

\subsection{Our algorithm}
We consider a randomization of the algorithm given in Theorem~\ref{thm:parametrization_of_polytope}.
First, we generate a random sequence of valid inner eigensteps $\Lambda^{\text{in}}$ using Algorithm~\ref{f:Algo_Random_Eigensteps}. Then we proceed to the reconstruction a frame that admits $\Lambda^{\text{in}}$ as a sequence of eigensteps using the Algorithm proposed in \citep{FiMiPo11}.

Note that Equations \eqref{eq:f_1_equation} and \eqref{eq:projections_of_f_on_eigenspaces} admit several solutions. For example, for $r \in [d]$, and if $\bm{f}_{r+1}$ satisfies \eqref{eq:projections_of_f_on_eigenspaces}, $-\bm{f}_{r+1}$ satisfies this equation too. \cite{FiMiPo11} actually prove that the set of solutions of these equations is invariant under a specific action of the orthogonal group $\mathbb{O}(\rho(r,k))$ where $\rho(r,k) \in \mathbb{N}$ nontrivially depends on the eigensteps. In the reconstruction step of our algorithm, we apply a random Haar-distributed orthogonal matrix as soon as such an invariance is provable. Namely, we a random orthogonal matrix sampled from the Haar measure on $\mathbb{O}(d)$ to the vector $\bm{f}_{1}$ and, then, we apply an independent random orthogonal matrix sampled from the Haar measure on $\mathbb{O}(\rho(r,k))$ to each reconstructed vector $\bm{f}_{r+1}$.

Figure~\ref{f:gt_generator_outputs} displays a few samples from our algorithm, which display diversity and no apparent structure, as required for a generator of toy datasets. The question of fully characterizing the distribution of the output of our algorithm is an open question.

\begin{figure}
\centerline{
\scalebox{1}{
\begin{algorithm}{$\Algo{RandomEigensteps}\big(\bm{\ell},\bm{\sigma})$}\label{algo_givens}
\Aitem $\Lambda^{\text{out}} \longleftarrow \mathbb{O} \in \mathbb{R}^{k \times d}$
\Aitem $\forall i \in [k], \: \tilde\lambda_{d,i} \longleftarrow \sigma_{i} $
\Aitem \For $r \in \{d-1,\dots,1 \}$
\Aitem \mt \For $i \in \{k,\dots,1 \}$
\Aitem \mtt Pick $\tilde\lambda_{r-1,i} \sim \mathcal{U}([B_{i,r}(\bm{\ell},\bm{\sigma}),A_{i,r}(\bm{\ell},\bm{\sigma})])$\\
\Return $\Lambda^{\text{out}}$
\end{algorithm}
}
}
\caption{The pseudocode of the generator of random valid eigensteps taking as input $(\bm{\ell},\bm{\sigma})$.}
\label{f:Algo_Random_Eigensteps}
\end{figure}

\end{document}